\documentclass[11pt, a4paper]{neutral}

\usepackage{natbib}

\usepackage{commands}
\newcommand{\indep}{\perp \!\!\! \perp}

\usepackage{mathtools}
\usepackage{bbm}
\usepackage{multirow}
\usepackage{subcaption}
\usepackage{float}
\usepackage{wrapfig}
\usepackage{titletoc}
\usepackage[title]{appendix}
\usepackage{verbatim}

\definecolor{babypink}{rgb}{0.98, 0.8, 0.91}
\definecolor{lavenderblue}{rgb}{0.8, 0.8, 1.0}
\definecolor{warmcream}{rgb}{0.851, 0.910, 0.835} 
\definecolor{coolblue}{rgb}{0.863, 0.875, 0.996}  

\definecolor{lightblue}{HTML}{0071bc}
\usepackage[breaklinks=true, colorlinks, citecolor=lightblue,
            linkcolor=lightblue, urlcolor=lightblue, bookmarks=false]{hyperref}
\usepackage{url}

\long\def\dk#1{}
\long\def\kwz#1{}
\long\def\lai#1{}
\newcommand{\new}[1]{#1}

\theoremstyle{plain}
\newtheorem{theorem}{Theorem}[section]
\newtheorem{proposition}[theorem]{Proposition}
\newtheorem{lemma}[theorem]{Lemma}
\newtheorem{corollary}[theorem]{Corollary}
\theoremstyle{definition}

\newtheorem{assumption}[theorem]{Assumption}
\theoremstyle{remark}

\title{Tabular Foundation Models Can Do Survival Analysis}
\author[1]{Da In Kim\textsuperscript{*}}
\author[1]{Wei Siang Lai\textsuperscript{*}}
\author[1]{Kelly W. Zhang}
\affil[1]{Imperial College London}
\correspondingauthor{\textsuperscript{*}Equal contribution.\ Correspondence to \texttt{kelly.zhang@imperial.ac.uk}.}
\codelink{https://github.com/dain20000222/TabularFM4Survival}

\begin{abstract}

While tabular foundation models have achieved remarkable success in classification and regression, adapting them to model time-to-event outcomes for survival analysis is non-trivial due to right-censoring, where data observations may end before the event of interest occurs. \new{In this work, we systematically study a variety of approaches to utilize off-the-shelf tabular foundation models originally designed for classification for discrete-time survival modeling. We evaluate existing classification formulations, including hazard modeling and multi-class classification approaches, and introduce a new formulation that directly models cumulative failure probabilities at discretized time points, treating censored observations as examples with missing labels at certain time points. We prove that under standard censoring assumptions, minimizing the weighted cross entropy loss recovers the true survival probabilities as the training set size increases. Across $60$ real-world datasets ($54$ static and $6$ dynamic), we find that TFMs paired with classification reductions often outperform or are competitive with classical and deep learning survival baselines on average across multiple survival metrics.
Our results show that off-the-shelf TFMs provide a flexible approach to survival modeling, allowing time discretization to be adapted at inference without training. Our comparison provides practical guidance on when to use these models and which survival formulation to choose.}
\end{abstract}

\begin{document}
\maketitle

\section{Introduction}

Survival analysis or time-to-event analysis models \textit{when} an event of interest will occur, rather than just \textit{whether} it will happen or not \citep{wang2019machine}. Applications of survival analysis span many domains, including healthcare (e.g., time to disease recurrence or mortality), social sciences (e.g., duration of unemployment), engineering (e.g., time to equipment failure), and finance/business (e.g., time to default or customer churn). 
A defining feature of survival analysis that differentiates it from standard regression is \textit{censoring}, where the true event time is only partially observed \citep{klein2006survival}. Under right-censoring, the exact time an event of interest occurs may be unknown because the observation period may end before the event happens. For example, consider a $5$-year clinical drug trial where mortality is the main outcome; the mortality event may be right-censored if the individual drops out early or remains alive through the $5$-year observation period. This incomplete event time information distinguishes survival analysis from standard regression tasks, where the target is always observed.

Despite its broad applicability, survival analysis often remains challenging due to (a) right-censoring, (b) small dataset sizes, e.g., several hundred samples or fewer \citep{drysdale2022survset}, and (c) covariates that may be irregularly sampled over time \citep{wang2019machine}. While foundation models have been developed for survival analysis in specialized medical tasks \citep{steinberg2024motor,vu2025tabular}, for general-purpose survival analysis, it remains standard to use specialized models fit from scratch on a single dataset. 

Recently, tabular foundation models (TFMs) \citep{hollmann2022tabpfn, grinsztajn2025tabpfn, zhang2025mitra, gardner2024large} have achieved remarkable success in classification and regression tasks, often outperforming methods like XGBoost, especially in small-data regimes. 
TFMs are usually transformer-based models, pretrained on a diverse collection of synthetic or real-world datasets, that are able to perform in-context learning at inference time without parameter updates. While this paradigm offers a promising direction for survival analysis, adapting TFMs to this setting is non-trivial. Standard regression architectures cannot easily account for censoring, a challenge that historically necessitated the development of specialized variants for survival analysis, e.g., Random Survival Forest \citep{ishwaran2008random}. Additionally, while TFMs are primarily evaluated on classification accuracy or mean squared error, survival modeling requires well-calibrated probabilities that accurately characterize risk dynamics over time.

In this work, we cast survival analysis as a series of binary classification problems that involve modeling the cumulative failure probability. This allows survival models to be optimized using a \new{weighted} binary cross-entropy loss and furthermore enables TFM classifiers to perform survival analysis out-of-the-box via in-context learning, leveraging pretrained representations to achieve high performance without requiring additional model training or hyperparameter tuning. Our approach discretizes event times, a common practice in the survival analysis literature \citep{efron1988logistic,craig2025review,allison1982discrete,singer1993s,lee2019dynamic,yu2011learning,maystre2022temporally}, and is applicable to both static and dynamic survival modeling. Our contributions are as follows:


\new{
\bo{(1) A systematic study of classification formulations for survival analysis with TFMs.}
We compare approaches for adapting off-the-shelf TFM classifiers to survival analysis, including discrete-time hazard modeling, multi-class classification over time bins, and cumulative failure modeling. Across $54$ static and $6$ dynamic real-world datasets, we evaluate these formulations alongside classical and deep learning survival methods. 
In the static setting, we find that the best formulation depends on the underlying classifier model (MITRA \citep{zhang2025mitra} works best with CF, TabPFN \citep{grinsztajn2025tabpfn} with DH), that multi-class formulation is consistently weakest, and that TFM gains in discriminative performance are largest on small datasets.
These comparisons provide practical guidance on when to use TFMs and how to choose a survival formulation.

\bo{(2) A consistent cumulative-failure formulation for survival modeling.}
We present a cumulative failure modeling approach that formulates time-to-event modeling as a binary classification task, where censoring leads to missing labels. Building on the label structure used in prior work \citep{kvamme2023brier,yu2011learning}, we derive a weighted binary cross-entropy loss and show that, under conditionally independent censoring, minimizing this loss yields consistent survival probability estimates. We find this classification formulation combined with the MITRA TFM \citep{zhang2025mitra} is the best-performing model on average in our static experiments.}

\section{Problem statement}
\label{sec:problemStatement}

\new{In this work, we primarily focus on \textit{static} survival analysis. While static survival analysis models event times given baseline features collected at the start of observation, dynamic models update predictions as more information becomes available \citep{chen2024introduction,van2011dynamic}. We consider an extension to the \textit{dynamic} survival analysis setting in Section \ref{sec:dynamic-extension}.}

\bo{Static survival modeling.}
Each dataset consists of a collection of $n$ i.i.d. data tuples $\{ (X_{i,0}, T_i, C_i, \delta_i ) \}_{i=1}^n$. Here, $X_{i,0} \in \real^d$ is a feature vector consisting of baseline covariates given at the start of the observation period. $T_i > 0$ represents the (latent) event time and $C_i > 0$ is the censoring time. We use $\delta_i \triangleq \mathbbm{1}(T_i \leq C_i)$ as an indicator of whether the event is observed or censored. Note, $\delta_i = 1$ if the event is observed and $\delta_i = 0$ if censored.

Our primary objective is to estimate the survival probability $S(t \mid X_{i,0})$, which represents the probability that the event occurs after time $t$:
\begin{align}
    \label{eqn:survival-static}
    S(t \mid X_{i,0}) \triangleq \PP \left( T_i > t \mid X_{i,0} \right).
\end{align}
Note that $S(t \mid X_{i,0})$ is a monotonically decreasing function of $t$ with $S(0 \mid X_{i,0}) = 1$ and $\lim_{t \to \infty} S(t \mid X_{i,0}) = 0$. 
Next, we introduce a conditionally independent censoring assumption. Note, without this assumption, the survival probabilities are provably not identifiable from observed data \citep{tsiatis1975nonidentifiability}. 
\begin{assumption}[Conditionally Independent Censoring]
    \label{assump:censoring}
    The event time $T_i$ and censoring time $C_i$ are conditionally independent given $X_{i,0}$, i.e., $T_i \indep C_i \mid X_{i,0}$.
\end{assumption}

\bo{Overview of common static survival modeling approaches.}
The Cox Proportional Hazards (CoxPH) model \citep{cox1972regression} and Random Survival Forests \citep{ishwaran2008random} are the most widely utilized approaches in applied static survival analysis \citep{wang2019machine}. Deep learning-based static survival models largely fall into two categories: (1) \textit{Deep extensions of (semi-)parametric survival models} \citep{wiegrebe2024deep,kvamme2019time,nagpal2021deep}, e.g., DeepSurv \citep{katzman2018deepsurv}, which replaces the linear risk function in CoxPH with a neural network, and (2) \textit{Discrete-time models} that partition the time horizon into intervals before modeling event probabilities \citep{lee2018deephit,gensheimer2019scalable,chi2021deep}; see Section \ref{sec:classification-related} for classification approaches.

\section{Background: Classification formulations for survival analysis}
\label{sec:classification-related}

In this section, we review a variety of existing formulations of survival analysis as a classification problem. 

\bo{Temporal discretization.} To transform the continuous survival task into a tractable classification problem, it is standard to partition the time horizon into $K$ discrete intervals $\{ \tau_k \}_{k=1}^K$. Given the boundaries $0 = t_0 < t_1 < t_2 < \dots < t_{K-1}$, we define the intervals as $\tau_k = (t_{k-1}, t_k]$ for $k < K$, with a final open-ended interval $\tau_K = (t_{K-1}, \infty)$. For convenience, let $t_K \triangleq \infty$. 
This time discretization approach is well-established in the survival analysis literature \citep{efron1988logistic,craig2025review,allison1982discrete,singer1993s,lee2019dynamic,yu2011learning,maystre2022temporally}. A common approach to form these bins is to use the quantiles of the Kaplan--Meier estimate of the marginal survival function \citep{kvamme2021continuous}.

\new{\subsection{Discrete-hazard modeling}
One of the first papers to use classification in survival modeling was developed by \cite{efron1988logistic}, who used a logistic regression model to approximate the discrete hazard rate $\lambda (t_k \mid X_{i,0}) \triangleq \PP \big( T_i \in \tau_k \mid X_{i,0}, T_i > t_{k-1} \big)$, i.e., the probability that an event occurs in the $k$th interval given that it has not occurred by the start of the interval. The logistic regression model is fit by maximizing the likelihood of the observed data, which factors into a function of the hazard rate; note, the likelihood of the observed data includes both censored and uncensored examples. \citet{singer1993s} extended this to log-odds models. This general paradigm of using binary classifiers to model the hazard rate has been termed \textit{Survival Stacking} \citep{craig2025review,zhong2019survival}. The general loss function used to optimize a discrete-hazard model can be written as follows: 
\begin{align*}
    \hat\ell_{\TN{DH}}(\hat\lambda) = \frac{1}{n} \sum_{i=1}^n \sum_{k=1}^{K-1} \mathbbm{1}\big( \min(T_i, C_i) > t_{k-1} \big) \cdot \TN{BCE}\big( \hat\lambda(t_k \mid X_{i,0}), \mathbbm{1}\big(T_i \in \tau_k\big) \big), 
\end{align*}
where 
$\TN{BCE}\big( p, Z \big) \triangleq - \left\{ Z \log p + (1-Z) \log(1-p) \right\}$.
The survival probability is recovered as the product of a function of the hazard rate: $S(t_k \mid X_{i,0}) = \prod_{j=1}^{k} \big\{ 1-\lambda (t_j \mid X_{i,0}) \big\}$. 
}


\new{\subsection{Multi-task logistic regression and extensions}
\citet{yu2011learning} introduced an alternative binary classification formulation called Multi-Task Logistic Regression (MTLR). They define binary labels $Z_{i,k} \triangleq \mathbbm{1}(T_i \geq t_k)$, which is one if the event has not occurred by time $t_k$ and zero otherwise; note that under censoring, the sequence $Z_{i,1}, \dots Z_{i,K}$ is truncated. We use $K_i \in \{ 1, \dots, K\}$ to denote the final observed label, i.e., $K_i \triangleq \begin{cases}
    K & \TN{if~} \delta_i = 1 \\
    k & \TN{if~} \delta_i = 0 \TN{~and~} C_i \in \tau_{k+1}
\end{cases}.$
They fit a multi-task logistic regression model $\hat p(x)$ that approximates $\PP( Z_{i,1}, \dots Z_{i,K_i} \mid X_{i,0} = x)$. Specifically, they fit their model by maximizing the probability of the observed component of the sequence $Z_{i,1}, \dots Z_{i,K_i}$ given $X_{i,0}$ using a conditional-random-field regularized objective, which is optimized via expectation maximization. Subsequent work has extended the \citet{yu2011learning} framework to neural survival models trained from scratch \citep{fotso2018deep} and to models that learn a multi-task logistic regression model on top of representations from a tabular foundation model \citep{pham2026survival}.
}

\new{\subsection{BCESurv}
\label{sec:bce-surv}
\citet{kvamme2023brier} introduce an alternative classification approach that also uses a multi-task binary classifier where $\hat p(x, t_k)$  approximates $S(t_k \mid x)$. Their approach can be used with general classification models and specifically they minimize the following loss:
\begin{align*}
    \hat\ell_{\TN{BCE}}(\hat p) \triangleq \frac{1}{n} \sum_{i=1}^n \sum_{k=1}^K \Big(1 - \mathbbm{1}( C_i < T_i) \mathbbm{1}( T_i \leq t_k) \Big) \cdot \TN{BCE}\big( \hat{p}(X_{i,0}, t_k), Z_{i,k} \big),
\end{align*}
where recall $Z_{i,k} \triangleq \mathbbm{1}(T_i \geq t_k)$ and $\TN{BCE}\big( p, Z \big) \triangleq - \left\{ Z \log p + (1-Z) \log(1-p) \right\}$. Similar to the MTLR approach, the BCESurv approach minimizes the binary cross entropy loss on only the observed (non-censored) component of the trajectory $Z_{i,1}, \dots Z_{i,K}$. Above, the weighting term, $1 - \mathbbm{1}( C_i < T_i) \mathbbm{1}( T_i \leq t_k)$, is always one if the trajectory is not censored and only equal to one for censored trajectories before the censoring time.
Note though that \citet{kvamme2023brier} show that the BCESurv approach will lead to biased survival estimates; specifically, they show that the minimizer $\hat{p}(X_{i,0}, t_k)$ of this loss will approach $S(t_k \mid X_{i,0}) \mathbbm{1}( C_i > t_k )$ \citep[Appendix B]{kvamme2023brier}.
}

\new{\subsection{Multi-class classification approaches}
}
\label{sec:multiclass-related}

\new{Beyond binary classification models, a variety of approaches have been proposed to learn survival probabilities by training a neural network to predict which of $K$ time intervals an event occurs in, including DeepHit, Dynamic DeepHit, and TabSurv \citep{lee2018deephit,lee2019dynamic,kirpichenko2026tabsurv}. 
Specifically, a multi-class classification model outputs a probability mass function $\big(\hat p_1(x), \dots, \hat p_K(x)\big)$ where $\hat p_k(x)$ estimates $\PP( T_i \in \tau_k \mid X_{i,0} = x )$, the probability that the event will occur in the $k$th time interval. Note that by construction $T_i \in \tau_k$ for some interval $\tau_k$ where $k \in [1 \colon K]$, since $T_i < \infty$. For example, the DeepHit-based models handle censored examples by maximizing the likelihood that the event occurred in all intervals following the censoring time. Specifically, for the event interval $\kappa_i$, where $\kappa_i = k$ if $\min(T_i, C_i) \in \tau_k$, DeepHit's primary loss objective\footnote{They also include an additional ranking-based loss.} is to minimize the following
\begin{align*}
    \hat\ell_{\TN{DeepHit}}(\hat p) \triangleq - \frac{1}{n} \sum_{i=1}^n \Bigg\{ \delta_i \log \hat p_{\kappa_i}(X_{i,0}) + (1 - \delta_i) \log \sum_{j > \kappa_i} \hat p_j(X_{i,0}) \Bigg\}.
\end{align*}
Survival probability estimates are formed as follows: $\hat S(t_k \mid x) = 1 - \sum_{j=1}^{k} \hat p_j(x)$.
}

\subsection{Concurrent work on in-context learning models for survival analysis}
Two related and independent efforts were released in May 2026: Survival In-Context (SIC) \citep{seletkov2026survival} and SurvivalPFN \citep{qi2026survivalpfn}. Both develop \emph{survival-specific} tabular foundation models by pretraining transformers from scratch on large-scale synthetic survival tasks. SIC uses the DeepHit training loss \citep{lee2018deephit}, while SurvivalPFN uses a cross-entropy objective and provides a censoring event indicator separately as an input to the model. 
These models are complementary to our work: they incorporate survival structure directly into the pretraining procedure, whereas we ask whether existing general-purpose TFMs can be used for survival analysis without survival-specific pretraining, fine-tuning, or architectural changes.
\new{SIC and SurvivalPFN fix the time grid during pretraining, since the number of time bins is part of the output layer ($K=10$ for SIC and $K=1024$ for SurvivalPFN).} 
\new{Specifically, for SurvivalPFN \citep{qi2026survivalpfn}: the classifier outputs a probability mass function $\hat p(X_{i,0}, \delta_i) = \big( \hat p_1(X_{i,0}, \delta_i), \dots, \hat p_K(X_{i,0}, \delta_i) \big)$. The model is trained using a cross-entropy classification loss
\begin{align}
    \hat\ell_{\TN{SPFN}}(\hat p) \triangleq \frac{1}{n} \sum_{i=1}^n \TN{Cross}\mathrm{-}\TN{Entropy} \big( \hat p(X_{i,0}, \delta_i), \kappa_i \big)
    \label{eqn:survivalpfn}
\end{align}
where $\kappa_i \in \{1, \dots, K\}$ is the index of the interval that contains $\min(T_i, C_i)$. At inference time, the survival probability is estimated using $\hat S(t_k \mid x) = 1 - \sum_{j \le k} \hat p_j(x, \delta = 1)$.}

\new{If one wants to use SIC or SurvivalPFN with a different time-discretization grid, it requires retraining the model or interpolating the predicted survival curve. A flexible time grid can be useful in practice because the appropriate time resolution for the grid may depend on the dataset size: with many bins, a small dataset leaves many bins without events, while a large dataset can resolve the survival curve more finely. Additionally, applications also often target fixed horizons, such as $1$-, $3$- and $5$-year survival in oncology. 
In contrast, for binary classification frameworks---discrete hazard, MTLR, BCESurv, and cumulative failure (which we present below)---the time points enter only through the constructed classification data, which allows the number of time bins $K$ to be changed by reprocessing the survival data. This means that off-the-shelf TFMs can be used for survival analysis without retraining if one wants to change the discretization granularity $K$. Note that we were not able to directly empirically compare to SIC and SurvivalPFN as they were not publicly available at the time of submission.}

\section{Cumulative failure formulation for time-to-event modeling}
\label{sec:method}

We introduce an alternative classification framework for survival modeling under right-censoring that frames the problem as a series of binary classification tasks. This will allow us to leverage off-the-shelf tabular foundation models and enable survival modeling via in-context learning without significant survival-specific training or hyperparameter tuning. 
Similar to BCESurv \citep{kvamme2023brier} and MTLR \citep{yu2011learning}, our approach involves transforming each time-to-event data tuple into a sequence of binary labels indicating whether the event has occurred by the end of each respective time interval. 

\subsection{\new{Constructing the supervised learning problem}}
\label{sec:staticMethod} 
To transform survival modeling into a supervised learning problem, we define binary labels $\{ Y_{i,k} \}_{k=1}^{K-1}$ where 
\begin{align}
    \label{eqn:Ydef}
    Y_{i,k} \triangleq \mathbbm{1}( T_i \leq t_k ).
\end{align}
Due to right-censoring, $Y_{i,k}$ is missing when the event is censored at or before $t_k$ ($C_i \leq t_k$ and $\delta_i = 0$); it is observed as $0$ if $t_k < \min(T_i, C_i)$, and as $1$ if the event is observed by $t_k$ ($\delta_i = 1$, $T_i \leq t_k$).\footnote{Note that comparing $Y_{i,k} \triangleq \mathbbm{1}(T_i \leq t_k)$ and $Z_{i,k} \triangleq \mathbbm{1}(T_i \geq t_k)$ from Section \ref{sec:classification-related}, we do not think there is generally meaningful difference in which set of label definition to choose.} 
As depicted in Figure~\ref{fig:static_method}, the original time-to-event data tuple $(X_{i,0}, T_i, C_i, \delta_i)$ is expanded into $K-1$ data tuples for binary classification: $\big\{ \big( (X_{i,0}, t_k), Y_{i,k}, \new{W_{i,k}} \big) \big\}_{k=1}^{K-1}$; \new{we discuss the exact form of the weights $W_{i,k}$ we use later in display \eqref{eqn:weights}}.

We train a classifier $\hat{p}(X_{i,0}, t_k)$ to approximate the failure probability $\PP( T_i \leq t_k \mid X_{i,0}) = \PP(Y_{i,k} = 1 \mid X_{i,0})$. Specifically, we fit $\hat{p}$ by minimizing the following \new{weighted} binary cross-entropy loss: 
\begin{align}
    \hat{\ell}_{\TN{static}}(\hat{p}) \triangleq \frac{1}{n} \sum_{i=1}^n \sum_{k=1}^{K-1} \new{W_{i,k}} \cdot \TN{BCE}\big( \hat{p}(X_{i,0}, t_k), Y_{i,k} \big)
    \label{eqn:empirical-loss}
\end{align}
where $\TN{BCE}\big( p, Y \big) \triangleq - \left\{ Y \log p + (1-Y) \log(1-p) \right\}.$ 
\new{Note our choice of weights $W_{i,k}$ will ensure that we only train on $Y_{i,k}$'s that are observed, i.e., $W_{i,k} = 0$ when the event is censored at or before $t_k$. We discuss the choice of weights in Section \ref{sec:weights}.}  

The cross-entropy loss from \eqref{eqn:empirical-loss} can be minimized using standard classification models, such as XGBoost or multi-layer perceptrons. Furthermore, this framework allows us to ``fit'' $\hat{p}$ via in-context learning using tabular foundation models for classification \citep{hollmann2022tabpfn,zhang2025mitra}, where $\hat{p}$ is computed by conditioning on the set of tuples $\big( (X_{i,0}, t_k), Y_{i,k} \big)$ from the training set. 

\bo{Inference.} Given a binary classifier $\hat{p}$, we form our estimate of the survival probability $S(t_k \mid X_{i,0})$ as follows:
\begin{align}
    \label{eqn:static_inference}
    \hat{S}(t_k \mid X_{i,0}) \triangleq 
    1- \hat{p}(X_{i,0}, t_k).
\end{align}
\new{To ensure the survival probabilities are always monotonically decreasing in $t_k$, one can use isotonic regression or simply clip $1-\hat{p}(X_{i,0}, t_k)$, i.e., set $1-\hat{p}(X_{i,0}, t_k) \triangleq \min _{j \leq k} 1-\hat{p}(X_{i,0}, t_{j})$. This ensures that $\hat S(t_1 \mid X_{i,0}) \ge \hat S(t_2 \mid X_{i,0}) \ge \cdots \ge \hat S(t_{K-1} \mid X_{i,0})$.}

\begin{figure}[h!]
    \vspace{-5mm}
  \begin{center}   
  \centerline{\includegraphics[width=0.95\linewidth]{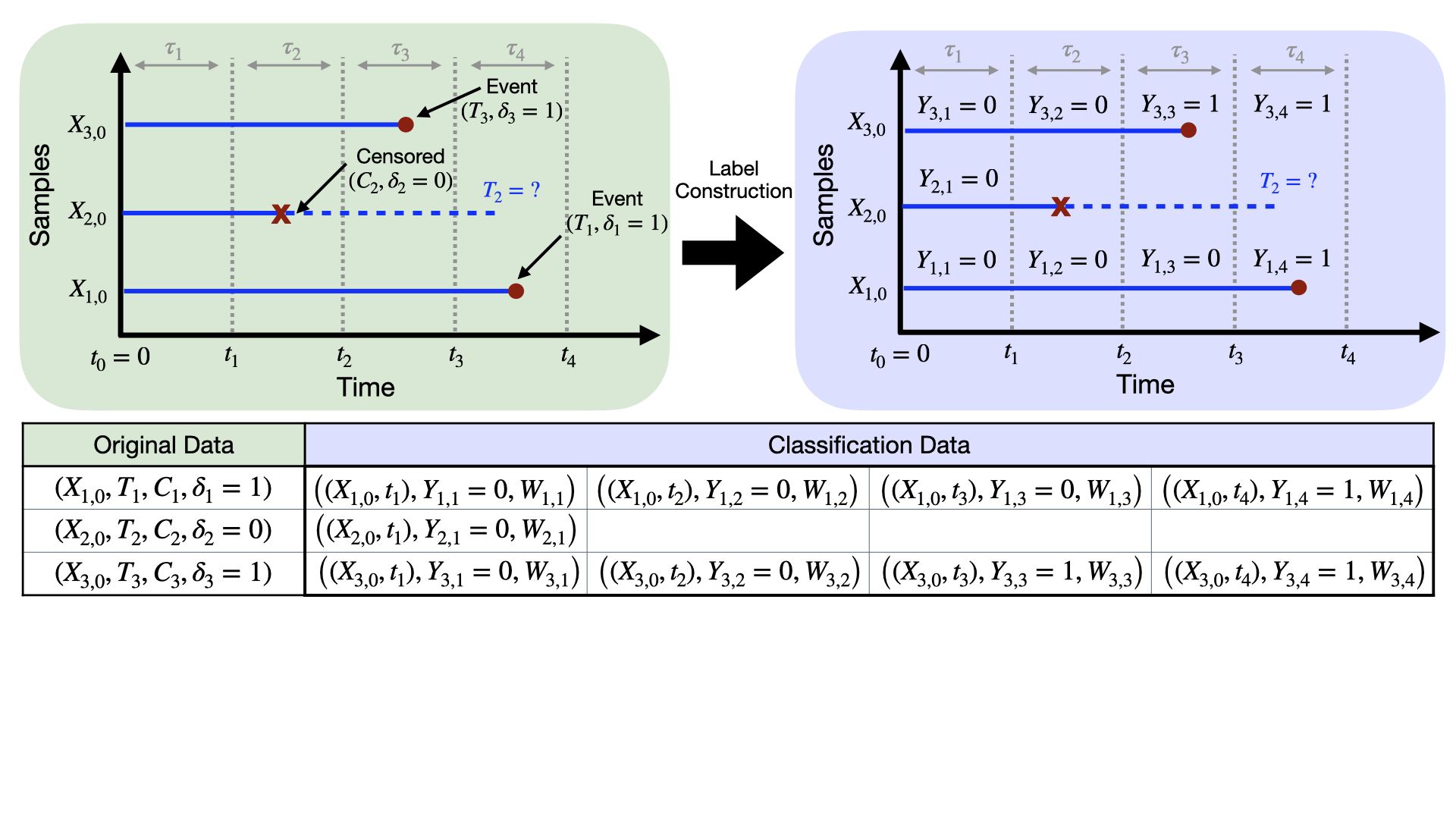}}
    \vspace{-23mm}
    \caption{\bo{Transformation of survival data into a binary classification dataset (static setting).} The \colorbox{warmcream}{left panel} illustrates the original time-to-event data $(X_{i,0}, T_i, C_i, \delta_i)$. For the second example, the event ($T_2=?$) is censored at $C_2$. The \colorbox{coolblue}{right panel} displays the discretization of time into $5$ intervals $\{\tau_k \}_{k=1}^5$, where binary labels $Y_{i,k} = \mathbbm{1} (T_i \leq t_k)$ are defined at each timestep $t_k$ for $k=1, 2, 3, 4$. 
    \new{The table (bottom) shows the resulting classification dataset: each example is transformed into at most $4$ tuples $\big((X_{i,0}, t_k), Y_{i,k}, W_{i,k}\big)$, where tuples after the censoring time have weight $W_{i,k} = 0$ and are dropped.}
    \vspace{-8mm}
    }
    \label{fig:static_method}
  \end{center}
\end{figure}

\subsection{\new{Choice of weights $W_{i,k}$}}
\label{sec:weights}
\new{If the censoring times $C_i$ were always observed, then we could simply use the weights $W_{i,k} = \mathbbm{1}(t_k < C_i)$; this choice would ensure that we only train on observed $Y_{i,k}$'s. Moreover, under the conditionally independent censoring assumption (Assumption \ref{assump:censoring}), the minimizer of the loss \eqref{eqn:empirical-loss} with weights $W_{i,k} = \mathbbm{1}(t_k < C_i)$ would recover the true underlying survival probabilities as the number of data points $n \to \infty$ (see Proposition~\ref{prop:minimizer-static-censoring}).}

\new{In practice though, the censoring time $C_i$ is typically not observed when the event occurs ($T_i < C_i$). In this case, choosing weights $W_{i,k} = \mathbbm{1}(t_k < C_i)$ is not feasible. However, we are able to construct a suitable set of weights that adjusts for the fact that $C_i$ is not always observed:}
\begin{align}
    \new{
    W_{i,k} \triangleq \begin{cases}
        1 & \TN{if}~ t_k < \min(C_i, T_i) \\
        \frac{ G(t_k \mid X_{i,0}) }{ G(T_i \mid X_{i,0}) } & \TN{if}~ t_k \geq \min(C_i, T_i) \TN{~and~} \delta_i = 1 \\
        0 & \TN{otherwise}
    \end{cases}.
    }
    \label{eqn:weights}
\end{align}
\new{Above, $G(t \mid x) := \PP( C_i > t \mid X_{i,0} = x)$ is the probability of not being censored before time $t$. In practice, we can estimate $G(t \mid x)$ using data, e.g., using a Kaplan-Meier curve. Note that since $G(t_k \mid X_{i,0})$ is a monotonically decreasing function in $t$, the weights $W_{i,k} \leq 1$ with probability $1$. The weights $W_{i,k}$ (i) give weight $0$ to the labels that are missing due to censoring, and (ii) correct for the fact that the censoring event $C_i$ is unobserved and we only train on examples for which $C_i > T_i$.
We formally show in Section~\ref{sec:theory} that minimizing the loss \eqref{eqn:empirical-loss} with weights from \eqref{eqn:weights} recovers the true underlying survival probabilities as $n \to \infty$.

\bo{Comparison to BCESurv.}
Recall from Section \ref{sec:bce-surv} that we discussed BCESurv \citep{kvamme2023brier}, a classification formulation of survival analysis that used a binary cross entropy loss with an alternative weighting. (They used labels $Z_{i,k} \triangleq \mathbbm{1}(T_i \geq t_k)$ instead of $Y_{i,k} \triangleq \mathbbm{1}(T_i \leq t_k)$, but we do not believe this to be very significant.) Within the context of loss \eqref{eqn:empirical-loss}, the BCESurv approach uses the following weights:
\begin{align}
    \label{eqn:BCESurv-weight}
    W_{i,k}^{\TN{BCE}} \triangleq \begin{cases}
    1 & \TN{if~} t_k < \min(C_i, T_i) \\
    1 & \TN{if~} t_k \geq \min(C_i, T_i) \TN{~and~} \delta_i = 1 \\
    0 & \TN{otherwise}
\end{cases}.
\end{align}
As we noted earlier in Section \ref{sec:bce-surv}, \citet[Appendix B]{kvamme2023brier} prove that using these weights will lead to inconsistent survival probabilities. In the next section, we show formally that using the weights from \eqref{eqn:weights} leads to consistent survival estimates. 
}

\subsection{Theory: Consistent survival models via loss minimization}
\label{sec:theory}

We now show formally that minimizing the losses from \eqref{eqn:empirical-loss} and \eqref{eqn:empirical-loss-dynamic} for the static and dynamic settings, respectively, yields consistent survival models as the number of training samples increases ($n \to \infty$). This means that despite the presence of censoring, perfect binary classifiers necessarily recover the true underlying survival probabilities. The result follows from the fact that binary cross-entropy is minimized by the true conditional event probabilities, together with conditionally independent censoring (Assumptions \ref{assump:censoring} and \ref{assump:censoring-dynamic}), and establishes that the proposed classification objectives remain statistically aligned with the underlying survival targets despite censoring.

For infinite training data ($n \to \infty$), the empirical loss from \eqref{eqn:empirical-loss} equals the population loss
\begin{align}
    \label{eqn:lossStatic}
    \ell_{\TN{static}}(p) \triangleq \E \bigg[ \sum_{k=1}^{K-1} \new{W_{i,k}} \cdot \TN{BCE}\big( p(X_{i,0}, t_k), Y_{i,k} \big) \bigg]. 
\end{align}
In Theorem \ref{thm:minimizer-static} below, we show that a binary classifier $p$ minimizes the loss $\ell(p)$ if and only if it accurately models the ground-truth survival probabilities. 

\begin{theorem}[Consistent static survival model via loss minimization]
    \label{thm:minimizer-static}
    Let $p$ be any binary prediction model for $Y_{i,k}$ given $(X_{i,0}, t_k)$. Under Assumption \ref{assump:censoring} (Conditionally Independent Censoring), the population loss $\ell_{\TN{static}}(p)$ with weights $W_{i,k}$ from \eqref{eqn:weights} is minimized if and only if, for each $k \in \{ 1, \dots, K-1 \}$, 
\begin{align} 
    p(x, t_k) 
    = \PP(Y_{i,k} = 1 \mid X_{i,0} = x) 
    = 1-\PP(T_i > t_k \mid X_{i,0} = x)
    \label{eqn:accurate-survival}
\end{align}
for all $x$ in the support of $X_{i,0}$ satisfying the positivity condition $\PP(C_i > t_k \mid X_{i,0} = x) > 0$. 
\end{theorem}

\new{
\begin{proof}
The equality $\PP(Y_{i,k} = 1 \mid X_{i,0}) = 1 - \PP(T_i > t_k \mid X_{i,0})$ above holds by the definition $Y_{i,k} \triangleq \mathbbm{1} (T_i \leq t_k)$. 

By Proposition \ref{prop:minimizer-static-censoring} (Appendix \ref{app:proof-minimizer-static}), under Assumption \ref{assump:censoring}, we have that a model $p$ satisfies \eqref{eqn:accurate-survival} for each $k \in \{ 1, \dots, K-1 \}$ if and only if it minimizes the following loss, which assumes that the censoring outcome $C_i$ is always observed (see the discussion in Section \ref{sec:weights}):
\begin{align}
    \tilde{\ell}_{\TN{static}}(p) \triangleq 
    \E \bigg[ \sum_{k=1}^{K-1} \mathbbm{1}(C_i > t_k) \cdot \TN{BCE}\big( p(X_{i,0}, t_k), Y_{i,k} \big) \bigg]. 
    \label{eqn:loss-censor-known}
\end{align}
The core idea of the proof is showing that the population loss under known censoring in \eqref{eqn:loss-censor-known} equals the population loss under weighting from \eqref{eqn:lossStatic}.
Note that since $Y_{i,k} \triangleq \mathbbm{1}(T_i \leq t_k)$,
\begin{align*}
    \tilde{\ell}_{\TN{static}}(p) &= \sum_{k=1}^{K-1} \E \Big[ \mathbbm{1}(C_i > t_k) \cdot \mathbbm{1}(T_i \leq t_k) \cdot \TN{BCE}\big( p(X_{i,0}, t_k), 1 \big) + \mathbbm{1}(C_i > t_k) \cdot \mathbbm{1}(T_i > t_k) \cdot \TN{BCE}\big( p(X_{i,0}, t_k), 0 \big) \Big] \\ 
    &= \sum_{k=1}^{K-1} \E \Big[ \mathbbm{1}(C_i > t_k) \cdot \mathbbm{1}(T_i \leq t_k) \cdot \TN{BCE}\big( p(X_{i,0}, t_k), 1 \big) + \mathbbm{1}(\min(C_i, T_i) > t_k) \cdot \TN{BCE}\big( p(X_{i,0}, t_k), 0 \big) \Big] 
\end{align*}
By the law of iterated expectations and Assumption \ref{assump:censoring}, $\E \big[ \mathbbm{1}(C_i > t_k) \mid X_{i,0}, T_i \big] 
= \E \big[ \mathbbm{1}(C_i > t_k) \mid X_{i,0} \big]
= \PP( C_i > t_k \mid X_{i,0}) = G(t_k \mid X_{i,0})$, thus,
\begin{align*}
    &= \sum_{k=1}^{K-1} \E \Big[ G(t_k \mid X_{i,0}) \cdot \mathbbm{1}(T_i \leq t_k) \cdot \TN{BCE}\big( p(X_{i,0}, t_k), 1 \big) + \mathbbm{1}(\min(C_i, T_i) > t_k) \cdot \TN{BCE}\big( p(X_{i,0}, t_k), 0 \big) \Big] 
\end{align*}
Next, we use that 
$G(t_k \mid X_{i,0}) \cdot \mathbbm{1}(T_i \leq t_k)
= \frac{ G(t_k \mid X_{i,0}) }{ G(T_i \mid X_{i,0}) } \cdot \mathbbm{1}(T_i \leq t_k) \cdot \PP( C_i > T_i \mid X_{i,0}, T_i) \\
= \frac{ G(t_k \mid X_{i,0}) }{ G(T_i \mid X_{i,0}) } \cdot \mathbbm{1}(T_i \leq t_k) \cdot \E \left[ \delta_i \mid X_{i,0}, T_i \right] 
= \E \left[ \frac{ G(t_k \mid X_{i,0}) }{ G(T_i \mid X_{i,0}) } \mathbbm{1}(T_i \leq t_k) \delta_i \,\Big|\, X_{i,0}, T_i \right] \\
= \E \left[ \frac{ G(t_k \mid X_{i,0}) }{ G(T_i \mid X_{i,0}) } \mathbbm{1}(\min(C_i,T_i) \leq t_k) \delta_i \,\Big|\, X_{i,0}, T_i \right]$ (recall that $\delta_i \triangleq \mathbbm{1}(C_i > T_i)$ and that when $\delta_i = 1$ that $T_i = \min(C_i, T_i)$). By the law of iterated expectations, 
\begin{align*}
    &= \sum_{k=1}^{K-1} \E \Bigg[ \frac{ G(t_k \mid X_{i,0}) }{ G(T_i \mid X_{i,0}) } \mathbbm{1}(\min(C_i, T_i) \leq t_k) \cdot \delta_i \cdot \TN{BCE}\big( p(X_{i,0}, t_k), 1 \big) + \mathbbm{1}(\min(C_i, T_i) > t_k) \cdot \TN{BCE}\big( p(X_{i,0}, t_k), 0 \big) \Bigg] \\ 
    &= \sum_{k=1}^{K-1} \E \Bigg[ \bigg\{ \frac{ G(t_k \mid X_{i,0}) }{ G(T_i \mid X_{i,0}) } \mathbbm{1}(\min(C_i, T_i) \leq t_k) \cdot \delta_i + \mathbbm{1}(\min(C_i, T_i) > t_k) \bigg\} \cdot \TN{BCE}\big( p(X_{i,0}, t_k), Y_{i,k} \big) \Bigg]. 
\end{align*}
Finally, the above equals $\sum_{k=1}^{K-1} \E \Big[ W_{i,k} \cdot \TN{BCE}\big( p(X_{i,0}, t_k), Y_{i,k} \big) \Big]$ using the definition of $W_{i,k}$ from \eqref{eqn:weights}. 
\end{proof}
}

The significance of Theorem \ref{thm:minimizer-static} is that under censoring Assumption \ref{assump:censoring}, we can in principle recover the true survival probabilities $\PP(T_i > t_k \mid X_{i,0})$ via loss minimization, given sufficient data \new{using a simple weighted loss}. \textit{This shows that survival analysis under temporal discretization can be reduced to a supervised, binary classification problem in a theoretically grounded manner; our empirical contribution is to demonstrate that this reduction unlocks practical use of modern tabular foundation models for survival analysis.} 

\section{Experimental results}
\label{sec:experiements}

\subsection{Datasets}
We use datasets from \emph{SurvSet}, a collection of real-world, right-censored time-to-event datasets \citep{drysdale2022survset}. These datasets span multiple domains, for example, clinical studies in biomedical research, labor datasets related to unemployment duration, as well as engineering reliability datasets concerning time to system or component failure.
\new{Out of the total $69$ static datasets in \emph{SurvSet}, we utilize $54$ of them. We excluded datasets (i) in which the dimension of $X_{i,0}$ (after one-hot encoding) was at least the number of training observations $n$ (11 datasets), (ii) had more than $10,000$ examples (2 datasets), or (iii) had a Kaplan-Meier estimate of the censoring probability $\PP( C_i > t_k ) = 0$ for some $t_k$ (2 datasets).} 
The final criterion ensures a positive probability of not being censored; without it, the survival distribution is non-identifiable in regions lacking support, and the inverse weighting used in standard evaluation metrics leads to unstable, biased estimates \citep{uno2011c}.
\new{Across the $54$ datasets, the median sample size is $482$ (IQR: $211$--$1918$) and the number of covariates in $X_{i,0}$ (before one-hot encoding) ranges from $3$ to $160$ (median: $8$, mean: $14.9$). The censoring rate ranges from $0\%$ to $94.4\%$ (median: $57.3\%$).} See Appendix \ref{app:datasets} for more details on the datasets.

\subsection{Survival models}
\label{sec:models} 
We consider classification models based on two tabular foundation models MITRA \citep{zhang2025mitra} and TabPFN-2.5 \citep{grinsztajn2025tabpfn}, as well as an XGBoost model.  
We evaluate our cumulative failure modeling approach (Section \ref{sec:method}), and a host of classification formulations introduced in Section \ref{sec:classification-related}. \new{For the multi-class (MC) formulation, we follow SurvivalPFN \citep{qi2026survivalpfn}, which uses the loss from \eqref{eqn:survivalpfn}. 
Since MITRA and TabPFN-2.5 support at most $10$ classes, MITRA-MC and TabPFN-MC are evaluated only for $K \in \{5, 10\}$.} We compare to classical survival models CoxPH \citep{cox1972regression} and Random Survival Forest (RSF) \citep{ishwaran2008random}, as well as neural survival methods DeepSurv \citep{katzman2018deepsurv}, a neural-network extension of the CoxPH model, and DeepHit \citep{lee2018deephit}. We also compare to the other discrete-time classification models of Section~\ref{sec:classification-related}: MTLR \citep{yu2011learning}, Neural MTLR \citep{fotso2018deep}, and BCESurv \citep{kvamme2023brier}. 

\bo{Temporal discretization.} \new{We set the grid points $t_k$ using equally spaced quantiles of the Kaplan--Meier estimate of the survival function on the training and validation sets \citep{kvamme2021continuous}. We consider four different discretization granularities $K \in \{5, 10, 15, 20\}$.} 

\new{\bo{Calibration of TFM models.} Note that for the cumulative-failure (CF), discrete hazard (DH), and multi-class (MC) classification frameworks with tabular foundation models, the models are fit in-context without additional fine-tuning. Note though, that we do a simple post-hoc calibration of the predictions using a monotonic transformation. Specifically we use one of three calibration approaches: (i) Platt scaling \citep{platt1999probabilistic}, $\sigma(a \cdot \mathrm{logit}(\hat p) + b)$; (ii) Platt scaling with one intercept $b_k$ per grid point $t_k$; or (iii) isotonic regression \citep{zadrozny2002transforming}. For each dataset and value of $K$, we choose the best performing calibration approach based on validation set performance.

For the CF models, to ensure monotonicity of the survival predictions, we clip the probabilities as described in Section \ref{sec:staticMethod}. For the TFM-based DH and MC models, they are calibrated using their respective losses (Section~\ref{sec:classification-related}). For the CF models, we calibrate using the weighted loss function from \eqref{eqn:lossStatic}. Note that to compute these weights, we use the marginal Kaplan-Meier curve estimate of the probability of censoring ~\ref{assump:censoring}, $G(t) = P( C_i > t )$. This means that in our experimental setup, minimizing our loss only leads to consistent survival estimates under a stronger assumption than Assumption~\ref{assump:censoring}, namely that $C_i$ is independent of $(T_i, X_{i,0})$; see Appendix \ref{app:consistent-strong} for a formal statement and proof.

} 

\subsection{Evaluation metrics}
\label{subsec:metrics}

We evaluate the survival models using standard metrics for discrimination and calibration. Specifically, we assess the ability to rank-order survival times using the (time-dependent) \textit{C-Index} and the \textit{Integrated AUC}. To quantify calibration and overall predictive accuracy, we use the \textit{Integrated Brier Score}. We now discuss these three metrics; note that to account for censoring, we use Inverse-Probability-of-Censoring Weighted (IPCW) estimators for all three metrics. \new{Note, we fit IPCW weights that model the marginal censoring probabilities $G(t) = \PP( C_i > t )$ using the Kaplan--Meier estimator. These metrics are therefore only consistent under a stronger assumption than Assumption~\ref{assump:censoring}, namely that $C_i$ is independent of $(T_i, X_{i,0})$ (Appendix~\ref{app:metrics}).} 

\bo{Concordance index (C-Index).}
The C-Index quantifies a survival model's discriminative ability by calculating the probability that, for a randomly selected pair of observations, the model assigns a higher risk to the observation that experiences the event first \citep{harrell1982evaluating,uno2011c}. A higher C-Index indicates greater discriminative ability. In the static setting, this is given by
\begin{align}
    \label{eqn:staticcindex}
    C_{\TN{Index}} \triangleq \mathbb{P} \big(\hat{r}(X_{i,0}) > \hat{r}(X_{j,0}) \mid T_i < T_j \big),
\end{align}
where $\hat{r}$ is a risk scoring model (higher values imply earlier predicted events). 
For our classification-based survival models, we construct the risk score as follows:
\begin{align}
    \label{eqn:staticRisk}
    \hat{r}(x) = \frac{1}{K-1} \sum_{k=1}^{K-1} \left\{1 - \hat{S}(t_k \mid x) \right\}.
\end{align}
Above, the risk score $\hat{r}(x)$ can be interpreted as the average probability that an event has occurred across all time steps $t_1, \dots, t_{K-1}$. This aggregate measure ensures that observations with consistently lower survival curves are assigned higher risk, even if their survival at a single specific time point is similar. For our baseline models, we discuss how we construct risk scores in Appendix \ref{app:baseline}.

\bo{Integrated AUC.}
The time-dependent AUC \citep{hung2010estimation} assesses a model's ability to differentiate between examples that survive up to some time $t$ versus those that do not:
\begin{align}
\label{eqn:static_auc}
    \mathrm{AUC}(t) \triangleq \mathbb{P} \big( \hat{r}(X_{i,0}) > \hat{r}(X_{j,0}) \mid T_i \leq t, T_j > t \big),
\end{align}
where $\hat{r}$ is a risk scoring model. See \eqref{eqn:staticRisk} for the risk score we use for our classification-based survival models. 
The Integrated AUC takes a weighted average of $\mathrm{AUC}(t)$ over time. Specifically, in the static setting we use $\overline{\mathrm{AUC}} \triangleq \frac{\sum_{k=1}^{K-1} \mathrm{AUC}(t_k) w_{t_k}}{\sum_{k=1}^{K-1} w_{t_k}}$,
where $w_{t_k} \triangleq S(t_{k-1}) - S(t_k)$ for $S(t_k) \triangleq \PP( T_i > t_k )$ is the marginal survival probability \citep{antolini2005time,lambert2016summary}. This weights the AUC by the proportion of events occurring in each interval.

\bo{Integrated Brier Score (IBS).} 
The Brier Score is a mean-squared error of the estimated survival probability at time $t$ and the true event indicator at time $t$ \citep{graf1999assessment,kvamme2023brier}:
\begin{align}
\label{eqn:static_brier}
    \TN{BS}(t) \triangleq \mathbb{E}\left[ \big\{ \mathbbm{1}( T_i > t ) - \hat{S}(t \mid X_{i,0}) \big\}^2 \right].
\end{align}
Low Brier Scores imply the model's survival probabilities are accurate and well-calibrated. The Integrated Brier Score averages the $\TN{BS}(t)$ over time: $\overline{\TN{BS}} \triangleq \frac{1}{t_{K-1}-t_1} \int_{t_1}^{t_{K-1}} \mathrm{BS}(t) \, d t$. 

\subsection{Experimental results}
\label{sec:static-results}
\new{Our static setting experiments are organized around the following five questions:
\begin{enumerate}[leftmargin=*] 
    \item How do TFM-based approaches perform on average compared to survival baselines?
    \item Are there some datasets that benefit strongly from using TFM-based approaches?
    \item Which formulation (CF, DH or MC) works best with each classifier?
    \item What is the impact empirically of incorporating weights $W_{i,k}$ from \eqref{eqn:weights} on performance?
    \item Does lower loss \eqref{eqn:lossStatic} lead to greater performance on survival metrics?
\end{enumerate}}

\begin{table}[b!]
\vspace{-3mm}
\centering
\caption{\bo{Static Setting.} We show results on average across $54$ \textit{SurvSet} datasets and across discretization granularities $K \in \{5,10,15,20\}$ (we form discretization bins by taking the quantiles of a Kaplan-Meier survival curve fit on the training and validation sets). Best and second best (ties share the colour) in \colorbox{blue!25}{dark} and \colorbox{blue!10}{light} blue. The error bars denote one standard error. }
\vspace{-2mm}
\label{tab:performance_comparison_km}
\resizebox{\textwidth}{!}{
\begin{tabular}{lccccccc}
\toprule
\multirow{2}{*}{\textbf{Model}} & \multicolumn{2}{c}{\textbf{C-Index}} & \multicolumn{2}{c}{\textbf{Integrated AUC}} & \multicolumn{2}{c}{\textbf{Integrated Brier Score}} & \multirow{2}{*}{\textbf{Avg. Rank $\downarrow$}} \\
\cmidrule(r){2-3} \cmidrule(r){4-5} \cmidrule(r){6-7}
 & Value $\uparrow$ & Rank $\downarrow$ & Value $\uparrow$ & Rank $\downarrow$ & Value $\downarrow$ & Rank $\downarrow$ & \\
\midrule
\multicolumn{8}{l}{\textit{Cumulative-Failure Models}} \\
\textbf{MITRA-CF} & \cellcolor{blue!25}\textbf{0.699}$_{\pm 0.016}$ & \cellcolor{blue!25}\textbf{6.3}$_{\pm 0.4}$ & \cellcolor{blue!25}\textbf{0.752}$_{\pm 0.016}$ & \cellcolor{blue!25}\textbf{5.6}$_{\pm 0.4}$ & \cellcolor{blue!25}\textbf{0.152}$_{\pm 0.007}$ & \cellcolor{blue!25}\textbf{4.7}$_{\pm 0.3}$ & \cellcolor{blue!25}\textbf{5.5}$_{\pm 0.3}$ \\
\textbf{TabPFN-CF} & 0.665$_{\pm 0.018}$ & 8.9$_{\pm 0.6}$ & 0.727$_{\pm 0.017}$ & 9.2$_{\pm 0.5}$ & 0.158$_{\pm 0.007}$ & 6.7$_{\pm 0.4}$ & 8.2$_{\pm 0.5}$ \\
\textbf{XGBoost-CF} & 0.688$_{\pm 0.015}$ & 8.0$_{\pm 0.5}$ & 0.737$_{\pm 0.016}$ & 7.5$_{\pm 0.5}$ & \cellcolor{blue!10}0.156$_{\pm 0.007}$ & \cellcolor{blue!10}5.4$_{\pm 0.5}$ & 7.0$_{\pm 0.4}$ \\
\midrule
\multicolumn{8}{l}{\textit{Discrete-time Hazard Models}} \\
\textbf{MITRA-DH} & 0.684$_{\pm 0.018}$ & 6.9$_{\pm 0.5}$ & \cellcolor{blue!10}0.748$_{\pm 0.017}$ & 6.4$_{\pm 0.5}$ & 0.158$_{\pm 0.007}$ & 6.7$_{\pm 0.4}$ & 6.7$_{\pm 0.4}$ \\
\textbf{TabPFN-DH} & 0.675$_{\pm 0.017}$ & 6.9$_{\pm 0.5}$ & 0.746$_{\pm 0.016}$ & \cellcolor{blue!10}5.7$_{\pm 0.4}$ & \cellcolor{blue!10}0.156$_{\pm 0.007}$ & \cellcolor{blue!10}5.4$_{\pm 0.4}$ & \cellcolor{blue!10}6.0$_{\pm 0.4}$ \\
\textbf{XGBoost-DH} & 0.691$_{\pm 0.016}$ & 6.8$_{\pm 0.5}$ & 0.737$_{\pm 0.017}$ & 7.3$_{\pm 0.5}$ & 0.160$_{\pm 0.008}$ & 6.5$_{\pm 0.5}$ & 6.9$_{\pm 0.4}$ \\
\midrule
\multicolumn{8}{l}{\textit{Multi-class Models}} \\
\textbf{MITRA-MC}$^\dagger$ & 0.628$_{\pm 0.018}$ & 11.4$_{\pm 0.6}$ & 0.700$_{\pm 0.021}$ & 10.8$_{\pm 0.6}$ & 0.294$_{\pm 0.014}$ & 14.9$_{\pm 0.3}$ & 12.4$_{\pm 0.4}$ \\
\textbf{TabPFN-MC}$^\dagger$ & 0.637$_{\pm 0.016}$ & 10.8$_{\pm 0.7}$ & 0.705$_{\pm 0.021}$ & 9.4$_{\pm 0.7}$ & 0.293$_{\pm 0.014}$ & 14.6$_{\pm 0.4}$ & 11.6$_{\pm 0.5}$ \\
\textbf{XGBoost-MC} & 0.625$_{\pm 0.015}$ & 11.7$_{\pm 0.5}$ & 0.689$_{\pm 0.016}$ & 10.8$_{\pm 0.5}$ & 0.290$_{\pm 0.013}$ & 14.4$_{\pm 0.2}$ & 12.4$_{\pm 0.3}$ \\
\midrule
\multicolumn{8}{l}{\textit{Other Discrete-time Classification Models}} \\
\textbf{MTLR} & 0.670$_{\pm 0.014}$ & 8.1$_{\pm 0.5}$ & 0.712$_{\pm 0.016}$ & 8.6$_{\pm 0.5}$ & 0.173$_{\pm 0.007}$ & 8.3$_{\pm 0.5}$ & 8.3$_{\pm 0.4}$ \\
\textbf{Neural MTLR} & 0.671$_{\pm 0.013}$ & 8.6$_{\pm 0.5}$ & 0.707$_{\pm 0.016}$ & 9.2$_{\pm 0.5}$ & 0.178$_{\pm 0.007}$ & 9.3$_{\pm 0.4}$ & 9.0$_{\pm 0.4}$ \\
\textbf{MLP BCESurv} & 0.648$_{\pm 0.015}$ & 9.7$_{\pm 0.5}$ & 0.689$_{\pm 0.019}$ & 9.6$_{\pm 0.5}$ & 0.195$_{\pm 0.007}$ & 11.9$_{\pm 0.3}$ & 10.4$_{\pm 0.4}$ \\
\textbf{XGBoost BCESurv} & 0.690$_{\pm 0.015}$ & 7.9$_{\pm 0.5}$ & 0.736$_{\pm 0.016}$ & 7.4$_{\pm 0.5}$ & 0.160$_{\pm 0.008}$ & 6.9$_{\pm 0.5}$ & 7.4$_{\pm 0.4}$ \\
\midrule
\multicolumn{8}{l}{\textit{Classical \& Neural Survival Baselines}} \\
\textbf{CoxPH} & \cellcolor{blue!10}0.695$_{\pm 0.016}$ & \cellcolor{blue!10}6.6$_{\pm 0.6}$ & 0.740$_{\pm 0.016}$ & 7.4$_{\pm 0.6}$ & 0.158$_{\pm 0.007}$ & 6.4$_{\pm 0.4}$ & 6.7$_{\pm 0.5}$ \\
\textbf{RSF} & 0.692$_{\pm 0.016}$ & 7.7$_{\pm 0.5}$ & 0.732$_{\pm 0.016}$ & 8.7$_{\pm 0.5}$ & 0.162$_{\pm 0.007}$ & 7.5$_{\pm 0.5}$ & 8.0$_{\pm 0.4}$ \\
\textbf{DeepHit} & 0.622$_{\pm 0.013}$ & 12.3$_{\pm 0.4}$ & 0.646$_{\pm 0.016}$ & 12.9$_{\pm 0.4}$ & 0.200$_{\pm 0.008}$ & 12.2$_{\pm 0.4}$ & 12.4$_{\pm 0.3}$ \\
\textbf{DeepSurv} & 0.665$_{\pm 0.016}$ & 9.2$_{\pm 0.4}$ & 0.703$_{\pm 0.017}$ & 10.1$_{\pm 0.4}$ & 0.172$_{\pm 0.007}$ & 9.5$_{\pm 0.4}$ & 9.5$_{\pm 0.3}$ \\
\bottomrule
\end{tabular}
}
\\[1mm] {\footnotesize $^\dagger$MITRA-MC and TabPFN-MC results are for $K \in \{5, 10\}$. Ranks are among the models available at each $K$.}
\end{table}

\subsubsection{How do TFM-based approaches perform on average compared to survival baselines? (Table \ref{tab:performance_comparison_km})}
\new{In Table \ref{tab:performance_comparison_km}, we show the performance of a variety of TFM and classification-based survival models compared to a variety of classical and neural survival baselines. 
Our results are averaged over the $54$ static datasets and $K \in \{5, 10, 15, 20\}$. Appendix~\ref{app:K_sensitivity} reports the results for each $K$ (Table~\ref{tab:static_by_k}). Overall, we find that MITRA-CF has the best performance overall on all metrics (value and rank of C-Index, Integrated AUC, and Integrated Brier Score), with an overall average rank of $5.5$. The second-best model overall by rank is TabPFN-DH ($6.0$). Overall, we find that the TFM-based approaches---specifically MITRA-CF, MITRA-DH, and TabPFN-DH---are among the most competitive models (by average rank overall). Among the non-TFM approaches, the most competitive approaches are XGBoost-CF, XGBoost-DH and CoxPH. Comparing the three classification formulations (CF, DH, MC), MC is the weakest for all three classification models (e.g., a C-Index of $0.628$ for MITRA-MC vs.\ $0.699$ for MITRA-CF); Figure~\ref{fig:formulations} compares the formulations for each classification model.}

Notably, the deep survival models (MLP BCESurv, DeepHit, DeepSurv, Neural MTLR) are among the worst-performing models. This poor performance is likely because these are trained from scratch and are evaluated in relatively small sample regimes (e.g., DeepHit was originally developed on survival datasets with $\approx 2{,}000$--$70{,}000$ examples, while the sample sizes we considered had a median of $482$ (IQR: $211$--$1918$)).

\begin{figure}[b!]
    \vspace{-2mm}
    \centering
    \includegraphics[width=0.85\linewidth]{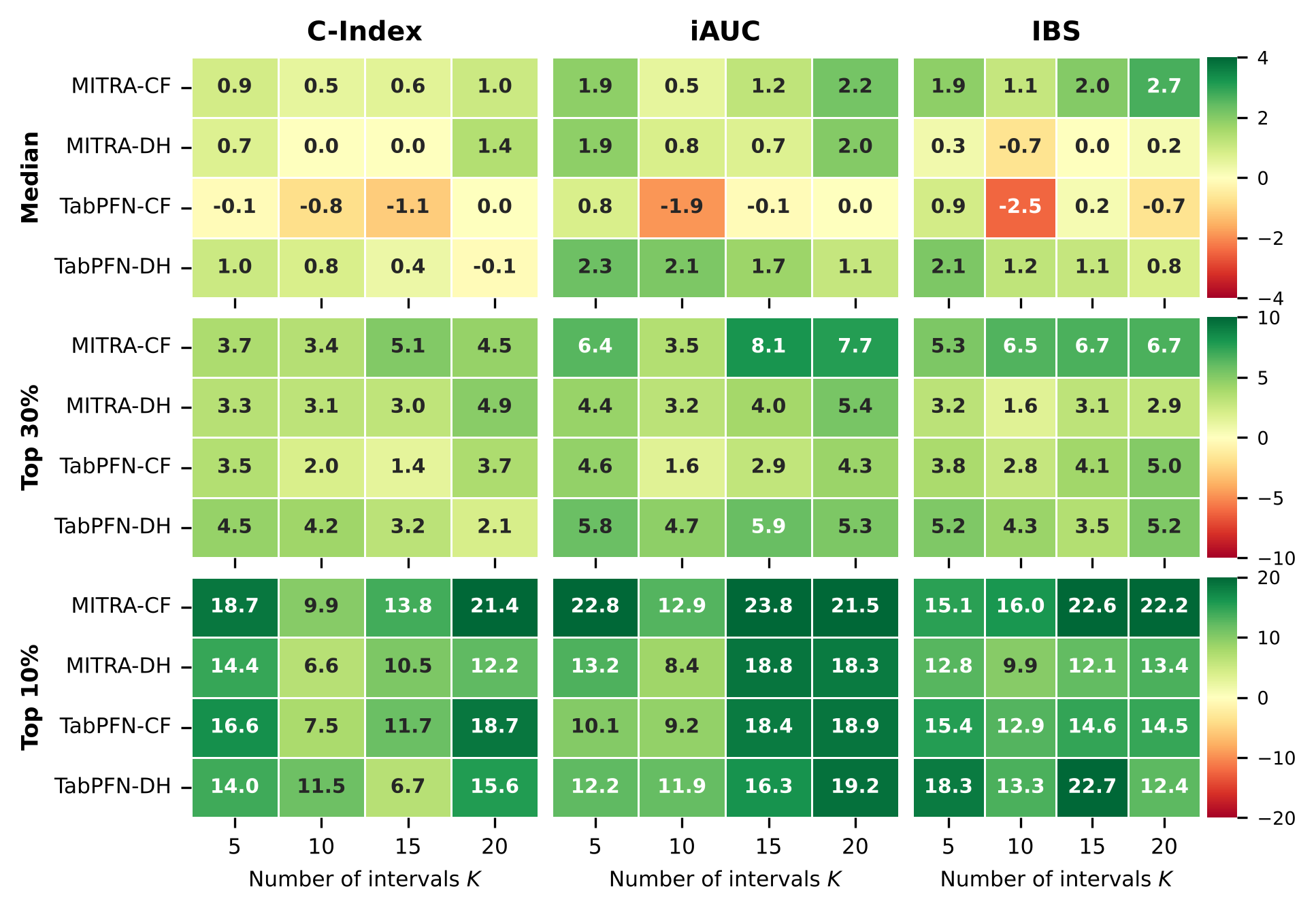} 
    \caption{\bo{Relative gain $G_d$ of TFM-based models over classical baselines across $54$ datasets.} The median gain (top) is modest; however, the top $30\%$ of datasets (middle) show $3$--$8\%$ improvement, and the top $10\%$ (bottom) show $10$--$24\%$ improvement for MITRA-CF. For each dataset, $K$, and metric, the best classical baseline is chosen by validation performance on that metric among \{CoxPH, RSF, DeepHit, DeepSurv\}.} 
    \label{fig:performance_gains}
\end{figure}

\subsubsection{Are there some datasets that benefit strongly from using TFM-based approaches? (Figures \ref{fig:performance_gains}, \ref{fig:gains_by_size})}

While the aggregate gains in Table~\ref{tab:performance_comparison_km} are modest on average, the per-dataset distribution is heterogeneous. For each of the $54$ static datasets $d$, we compute the relative gain of MITRA-CF over the best non-TFM baseline: $G_d \triangleq (m^{\TN{TFM}}_d - m^{\TN{best}}_d) / m^{\TN{best}}_d$, where $m$ is the metric of interest (C-Index, Integrated AUC, or IBS, with signs adjusted so that larger $G_d$ corresponds to greater improvement for the TFM), and $m^{\TN{best}}_d$ is that of the model in \{CoxPH, RSF, DeepHit, DeepSurv\} that performs best on the validation set. \textit{We find that for MITRA-CF the gains in the top $30\%$ of datasets show a $3$--$8\%$ improvement and the top $10\%$ have $10$--$24\%$ improvement over the best non-TFM baseline.} \new{Appendix~\ref{app:tfm_vs_xgboost} additionally compares each TFM directly with XGBoost under the same formulation: MITRA-CF has a better C-Index than XGBoost-CF on $69\%$ of the datasets, whereas TabPFN-CF does on $41\%$ (Table~\ref{tab:tfm_vs_xgboost}).}


\begin{figure}[h!]
    \centering
    \vspace{-1mm}
    \includegraphics[width=0.7\linewidth]{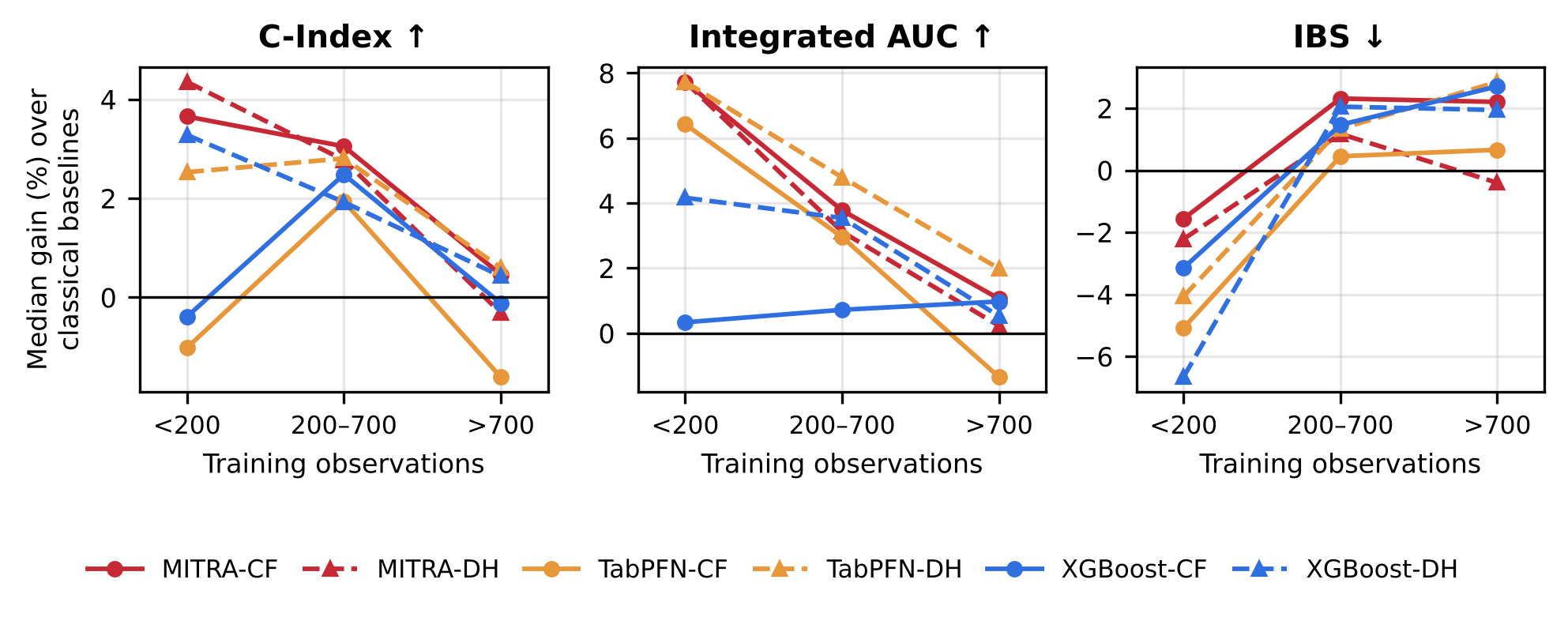} 
    \vspace{-3mm}
    \caption{\new{\bo{Gain by dataset size.} Median relative gain (\%) over the classical baselines, by the number of training observations; the $54$ datasets are split into three groups (18 datasets per group). For each dataset, the gain is averaged over $K \in \{5, 10, 15, 20\}$. For each dataset, $K$, and metric, the baseline is the best of CoxPH, RSF, DeepHit, and DeepSurv on the validation set.}}
    \label{fig:gains_by_size}
\end{figure}
\bo{Relative gains as a function of dataset size.} We examined the gains over baselines as a function of dataset size for the 54 datasets we examine (Figure~\ref{fig:gains_by_size}). We split these datasets into thirds: those with $<200$, $200-700$, and $>700$ training observations.
For C-Index and Integrated AUC, the median gains over the classical baselines are the largest for the top models for the smallest third of datasets; in contrast, for the largest third of datasets ($>700$) there are close to zero gains over the baselines. However, on the IBS, the pattern is reversed: on the smallest third, every classification model has a worse IBS than the best classical baseline, and the median datasets improve upon the baseline only for $>200$ training observations.



\begin{figure}[h!]
    \centering
    \includegraphics[width=0.75\linewidth]{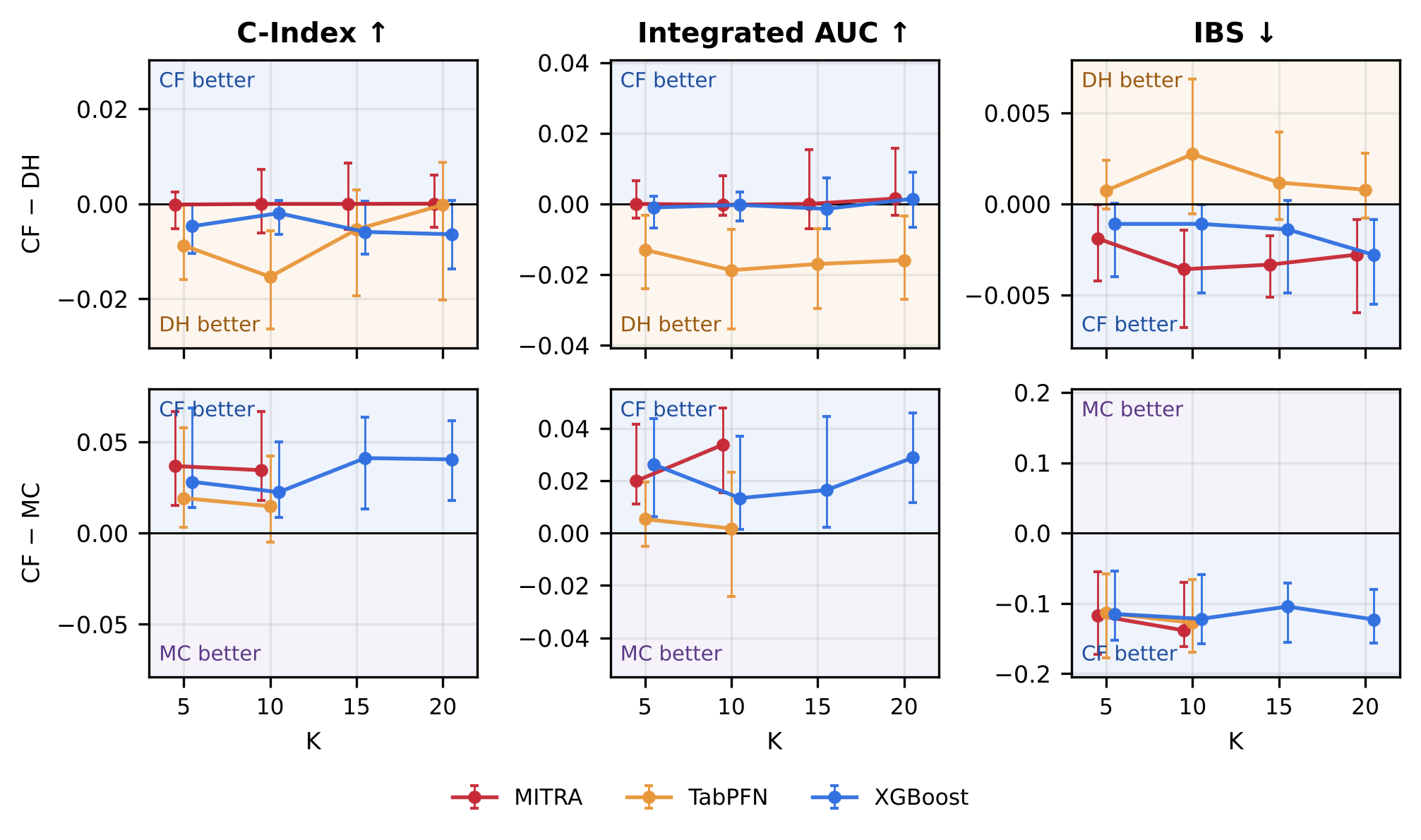} 
    \vspace{-3mm}
    \caption{\new{\bo{CF vs.\ DH (top) and CF vs.\ MC (bottom) for each classification model}. For each of the $54$ datasets, we compute the performance of a classification model (MITRA, TabPFN, XGBoost) using two classification formulations (CF, DH). We take the difference between the performance metrics and report the median difference across the $54$ datasets and a $95\%$ bootstrap interval. Note, MITRA-MC and TabPFN-MC are evaluated only for $K \leq 10$, since both TFMs support at most $10$ classes.}}
    \label{fig:formulations}
    \vspace{-3mm}
\end{figure}

\subsubsection{Which classification formulation works best with each model? (Figure~\ref{fig:formulations})}
\new{For each classification model (MITRA, TabPFN, XGBoost), we compare their performance using the three different frameworks (CF, DH, MC). Figure~\ref{fig:formulations} compares the CF formulation with the DH and MC formulations of the same classifier at each $K$. For MITRA, the C-Index and iAUC performance are tied for CF and DH, but the CF formulation has lower IBS (especially as $K$ grows); specifically, CF is better than DH on $35$ of $54$ datasets at $K=5$ and on $39$ at $K=20$.
For TabPFN, the DH formulation is better than CF, most clearly on the Integrated AUC. MC is the weakest formulation with every classifier. Note that these results are after post-hoc calibration (Section~\ref{sec:models}); without it, the IBS of the DH models worsens as $K$ grows, most clearly for TabPFN; see Appendix~\ref{app:K_sensitivity} (Figure~\ref{fig:dh_k_calibration}).}

\subsubsection{What is the impact empirically of incorporating weights $W_{i,k}$ from \eqref{eqn:weights} on the performance of CF models? (Table~\ref{tab:performance_comparison_km} and Figure~\ref{fig:weight_ablation})}
\label{sec:weights-ablation} 

\par
\begingroup
\begin{wrapfigure}{r}{0.45\textwidth}
    \centering
    \vspace{-8mm}
    \includegraphics[width=0.85\linewidth]{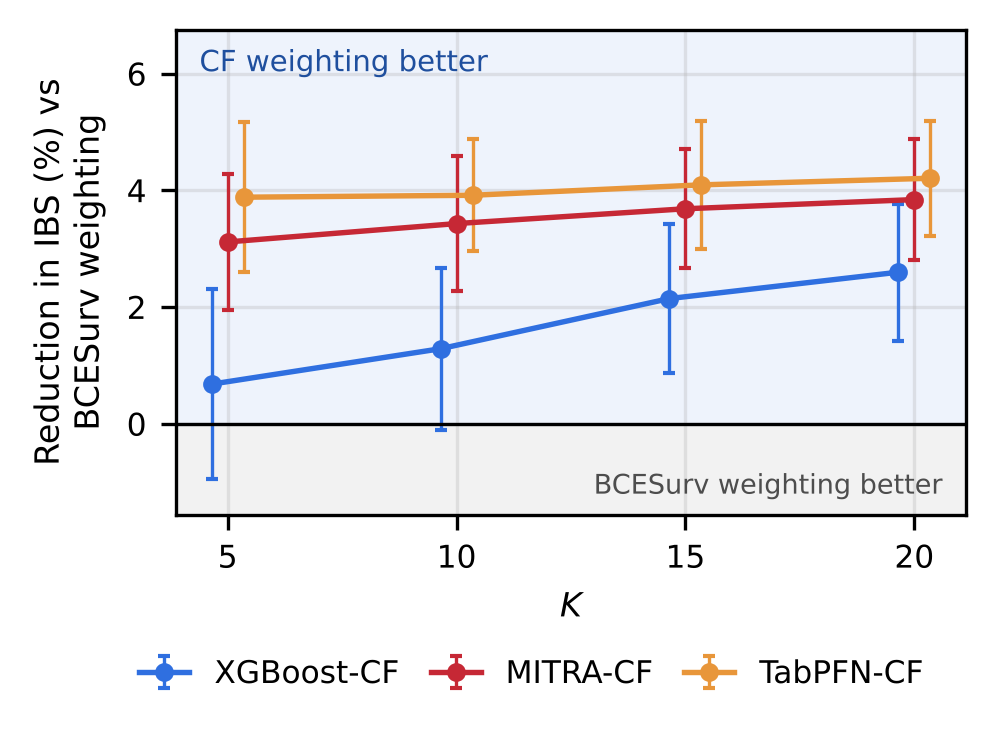}
    \vspace{-4mm}
    \caption{\new{\bo{Effect of weighting on IBS.} Reduction (\%) in IBS from using CF vs BCESurv weighting; we report the mean across $54$ datasets $\pm$ one standard error.}} 
    \label{fig:weight_ablation}
\end{wrapfigure}
We now examine the empirical impact of our choice of weighting $W_{i,k}$ from \eqref{eqn:weights} that we use in our CF loss; recall that the weights are used to optimize the XGBoost training and the post-hoc calibration of the TFM models (see Section \ref{sec:models}). Specifically, we compare the loss \eqref{eqn:lossStatic} with weights $W_{i,k}$ from \eqref{eqn:weights} to the same loss with the weights $W_{i,k}^{\TN{BCE}}$ from \eqref{eqn:BCESurv-weight}, i.e., the BCESurv loss \citep{kvamme2023brier}. 

We find that the choice of weighting does not change the performance of XGBoost-CF much in terms of C-Index ($0.688$ vs.\ $0.690$) and Integrated AUC ($0.737$ vs.\ $0.736$). However, we find that for all three models (XGBoost, MITRA, TabPFN), using our weights from \eqref{eqn:weights} improved IBS for all three CF models at every $K$ (Figure~\ref{fig:weight_ablation}).

\par
\endgroup

\subsubsection{Does lower loss \eqref{eqn:lossStatic} lead to better performance on classical metrics? (Figure \ref{fig:correlation})}
\begin{wrapfigure}[22]{r}{0.45\textwidth}
    \centering
    \vspace{-5mm}
    \includegraphics[width=\linewidth]{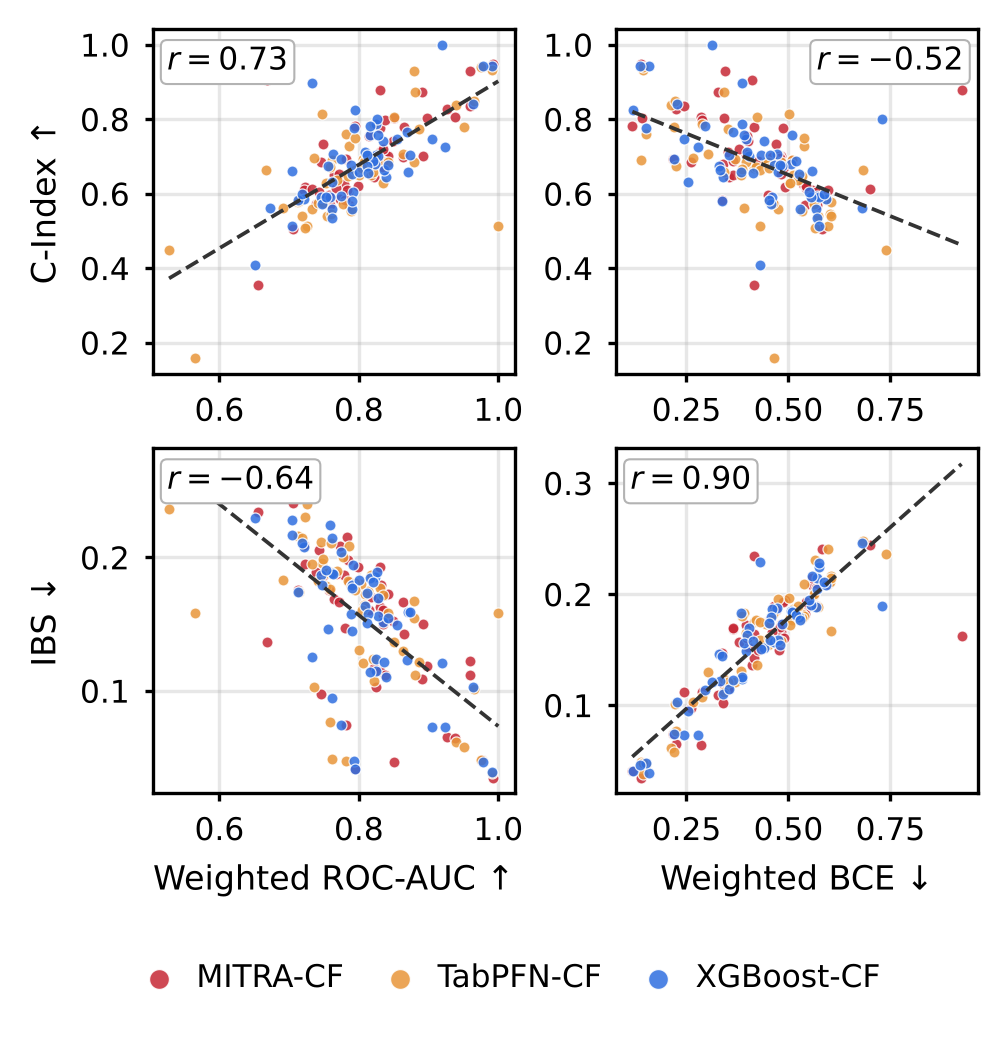}
    \vspace{-5mm}
    \caption{\bo{Correlation between classification and survival metrics.} Each point is a CF model on one of the $54$ datasets (test set, averaged over $K \in \{5, 10, 15, 20\}$).}
    \label{fig:correlation}
\end{wrapfigure}
A natural question is whether improvements in the binary classification objective from Section \ref{sec:method} translate into improvements in standard survival metrics. To investigate this, we examine the relationship between \textit{classification metrics}---weighted binary cross-entropy loss and ROC-AUC---and \textit{survival metrics}---C-Index and IBS---across the $54$  datasets. We find only a moderate correlation between the weighted binary cross-entropy loss and C-Index ($r=-0.52$), reflecting the fact that our binary objective does not explicitly optimize for ranking. In contrast, weighted binary cross-entropy loss \eqref{eqn:lossStatic} is strongly correlated with Integrated Brier Score ($r = 0.9$), supporting the theoretical result that minimizing the proposed binary cross-entropy loss targets accurate survival probability estimates (Theorem \ref{thm:minimizer-static}). 

Practically, these results suggest that to minimize IBS, model selection can be performed directly in the classification space: among candidate classifiers, including TFMs and non-TFM models, choose the one with the lowest validation weighted binary cross-entropy loss. This provides a simple, survival-aware selection criterion and empirically supports the reduction of survival probability modeling to binary classification formalized in Theorem \ref{thm:minimizer-static}.

\section{Extension to dynamic survival analysis}
\label{sec:dynamic-extension}

In this section, we introduce the dynamic survival analysis setting and present empirical results.

\subsection{Dynamic problem setting}
\label{sec:dynamic_problem_statement}
In the dynamic setting, the baseline covariates $X_{i,0}$ are replaced by a sequence of observations revealed over time that could be observed at irregular intervals. We denote this sequence of covariates as $\MC{X}_i \triangleq \{ X_{i,t} : t \in \MC{T}_i\}$, where $\MC{T}_i \subset [0, \infty)$ represents the observation times. The history of covariates up to time $t$ is defined as
\begin{align}
    \label{eqn:history}
    \HH_{i,t} \triangleq \big\{ (s, X_{i,s}) : s \in \MC{T}_i, s \leq t \big\}.
\end{align}
The dataset has $n$ i.i.d. tuples $\{ (\MC{X}_i, T_i, C_i, \delta_i ) \}_{i=1}^n$, where the definitions of $T_i$, $C_i$, and $\delta_i$ follow Section~\ref{sec:problemStatement}.

The modeling task is to estimate the conditional survival probability for a future horizon $\Delta > 0$, given the information available at time $t$:
\begin{align}
    \label{eqn:survival-dynamic}
    S(t+\Delta \mid T_i > t, \HH_{i,t} ) \triangleq \PP \left( T_i > t +\Delta \mid T_i > t, \HH_{i,t} \right).
\end{align}
We use a dynamic analogue of the conditionally independent censoring condition from the static setting (Assumption \ref{assump:censoring}), which ensures the dynamic survival probabilities \eqref{eqn:survival-dynamic} are identifiable from observed data.
\begin{assumption}[Dynamic Conditionally Independent Censoring]
    \label{assump:censoring-dynamic}
    For any $t \geq 0$, the event time $T_i$ and censoring time $C_i$ are conditionally independent given the history and survival up to time $t$, i.e., $T_i \indep C_i \mid (\HH_{i,t}, T_i > t)$.
\end{assumption}

\bo{Overview of common dynamic survival modeling approaches.}
The predominant dynamic survival modeling approaches fall into three categories: (1) \textit{Landmarking}, which fits a static survival model at multiple timepoints \citep{van2007dynamic}. (2) \textit{Deep sequence or Joint modeling}, which jointly models longitudinal observations and event times \citep{tsiatis2004joint,ibrahim2010basic}; this includes classical statistical approaches that use linear mixed-effects models and CoxPH, as well as deep-learning variants that use recurrent neural networks with survival objectives \citep{giunchiglia2018rnn,ren2019deep}. Neural ODEs have also been used to explicitly handle irregularly sampled data \citep{sodenTang,moon2022survlatent}. (3) \textit{Discrete-time models}, that extend the static discrete-time models to longitudinal data \citep{lee2019dynamic,kvamme2021continuous,thorsen2022discrete,maystre2022temporally}.

\subsection{Dynamic cumulative failure modeling}
\label{sec:dynamicMethod}
We adopt the same temporal discretization and binary label definition, $Y_{i,k} \triangleq \mathbbm{1}(T_i \leq t_k)$, as in the static setting. In the dynamic setting, however, the model must account for time-varying information. At a given timestep $t_k$, the model is provided with the covariate history and the knowledge that the observation has survived up to $t_k$ (i.e., $Y_{i,1} = \dots = Y_{i,k} =0$) $\HH_{i,t_k} \triangleq \big\{ (s, X_{i,s}) : s \in \MC{T}_i, s \leq t_k \big\} \cup Y_{i,1:k}$. 
Consequently, each original time-to-event data tuple $(\MC{X}_i, T_i, C_i, \delta_i)$ is transformed into training examples: $\big( (\HH_{i,t_{k}}, 
t_{k+\Delta}), Y_{i,k+\Delta}, \new{W_{i,k,k+\Delta}} \big)$, for each $k \in \{1, \dots, K-2\}$
and $\Delta \in \{ 1, \dots, K-k-1\}$, provided $t_k < \min(T_i, C_i)$. We use $Y_{i,1:k} \triangleq (Y_{i,1}, \dots, Y_{i,k})$. 

\bo{Model training.}
We train a binary classifier $\hat{p}$ to minimize the dynamic empirical binary cross-entropy loss:
\begin{align}
\hspace{-1.5mm}\hat{\ell}_{\text{dynamic}}(\hat{p}) & \triangleq \frac{1}{n} \sum_{i=1}^n \sum_{k=1}^{K-2} \sum_{\Delta=1}^{K-k-1} \mathbbm{1}(t_{k} < T_i) \cdot \new{W_{i,k,k+\Delta}} \cdot \text{BCE} \big( \hat{p}(\mathcal{H}_{i,t_k}, 
t_{k+\Delta}), Y_{i,k+\Delta} \big). \label{eqn:empirical-loss-dynamic}
\end{align}
We include the indicator $\mathbbm{1}(t_{k} < T_i)$ to ensure that we only train on examples for which the event is at risk, i.e., if the event occurred before time $t_k$, we are not using $\HH_{i,t_k}$ to do further prediction.
\new{Similar to the static setting, the weights $W_{i,k,k+\Delta}$ allow for the use of all information prior to censoring and adjust for the fact that the censoring time is unobserved. Specifically,}
\begin{align}
    \new{W_{i,k,k+\Delta} \triangleq \begin{cases}
        1 & \TN{if}~ t_{k+\Delta} < \min(C_i, T_i) \\
        \frac{ G(t_{k+\Delta} \mid \HH_{i,t_k}) }{ G(T_i \mid \HH_{i,t_k}) } & \TN{if}~ t_{k+\Delta} \geq \min(C_i, T_i) \TN{~and~} \delta_i = 1 \\
        0 & \TN{otherwise}
    \end{cases}.}
    \label{eqn:weights-dynamic}
\end{align}
\new{Above, $G(t \mid \HH_{i,t_k}) \triangleq \PP( C_i > t \mid \HH_{i,t_k})$ is the probability of not being censored before time $t$.} 
Practically, for training these classifiers across different history lengths ($\HH_{i,t_k}$ for different values of $k$), we featurize the inputs into a fixed dimension $d$ vector $\phi_k(\HH_{i,t_k}, 
t_{k+\Delta}) \in \real^d$ for vector-valued functions $\phi_k$.
For our classification models, we use vectors $\phi_k(\HH_{i,t_k}, 
t_{k+\Delta})$ as input and $Y_{i,k+\Delta}$ as the label. See Appendix \ref{app:classification_model_dynamic} for details on how we choose $\phi_k$.

\bo{Inference.}
Given a classifier $\hat{p}$ we estimate the dynamic survival probability $S(t_{k+\Delta} \mid T_i > t_k, \HH_{i,t_k})$ using:
\begin{align}
\label{eqn:dynamic_inference}
    \hat{S}(t_{k+\Delta} \mid T_i > t_k, \HH_{i,t_k}) \triangleq 
    1- \hat{p}(\HH_{i,t_k}, 
    t_{k+\Delta}).
\end{align}
\new{Similar to the static setting, to ensure the survival probabilities are monotonically decreasing in $t_{k+\Delta}$ we can use isotonic regression or clipping (Appendix \ref{app:classification_model_dynamic}).}

\bo{Dynamic setting theory.}
Analogously to the static setting, in the dynamic setting, we can show that for infinite training data ($n \to \infty$), minimizing the loss from \eqref{eqn:empirical-loss-dynamic} will necessarily lead the classifier to learn the true underlying survival probabilities. This result is very similar to Theorem \ref{thm:minimizer-static}; see Appendix \ref{app:proof-minimizer-dynamic}.

\subsection{Dynamic experiments}
\label{sec:dynamic-results}

\subsubsection{Experimental setup}
\bo{Datasets.}
We use $6$ dynamic datasets from \emph{SurvSet} \citep{drysdale2022survset}: \texttt{oldmort} has $4{,}603$ observations, \texttt{heart} $103$, and the others $446$ to $697$; the number of covariates ranges from $4$ to $7$, and the censoring rate from $27.2\%$ to $89.7\%$.

\bo{Temporal discretization.}
\new{For each of the $6$ datasets, we choose the discretization granularity $K$ using the following procedure. We fit a marginal survival curve using the Kaplan--Meier estimator, and for a candidate $K$ we take the grid points $t_1 < \dots < t_{K-1}$ at equally spaced quantiles of that fitted curve. For each of the 6 datasets, we choose the discretization granularity K using the following procedure: We take K equally spaced quantiles of that fitted curve and decrease $K$ until at least 5 events in the training set in each interval (with a maximum of $K=10$). This leads to values of $K$ between $5$ and $10$.} 

\new{\bo{Evaluation metrics.} We use dynamic analogues of C-Index, Integrated AUC, and Integrated Brier Score to evaluate our models. Additional details are in Appendix \ref{app:metrics}.}


\bo{Dynamic survival models.} We compare to landmarked versions of the CoxPH and RSF models \citep{van2007dynamic}. We also compare to two popular dynamic survival models: Dynamic DeepHit \citep{lee2019dynamic}, and Joint Modeling \citep{ibrahim2010basic}, which simultaneously models the longitudinal trajectory of time-varying covariates with a linear mixed-effects model and the time-to-event process with a CoxPH model. 
\new{We also compare to Temporally-Consistent Survival Analysis (TCSA) \citep{maystre2022temporally} a methodology designed for environments with Markovian dynamics (Appendix~\ref{app:tcsa}). As in the static setting, we use the CF and DH formulations with MITRA, TabPFN-2.5, and XGBoost. All of these models, including the XGBoost models, are calibrated post hoc as in Section~\ref{sec:models}; the calibration only affects the IBS metric.}  

\new{\bo{Representation of the history.} The classification-based dynamic models (MITRA, TabPFN, and XGBoost) take the vector $\phi_k(\HH_{i,t_k}, t_{k+\Delta})$ as input. Let $s_i(t_k)$ be the time of the most recent observation prior to $t_k$, $X^{\mathrm{LOCF}}_{i,t_k} \triangleq X_{i,s_i(t_k)}$ the covariates at that observation, and $\bar X_{i,t_k}$ the running means of the covariates computed up to $t_k$. The base representation is
\begin{align*}
    \phi_k\!\left(\HH_{i,t_k}, t_{k+\Delta}\right) = \Big[ X^{\mathrm{LOCF}}_{i,t_k},\; \bar X_{i,t_k},\; t_k,\; t_{k+\Delta} \Big].
\end{align*}
We also consider representations that additionally include the time elapsed since the most recent observation, $t_k - s_i(t_k)$, and the index of the prediction horizon, $k + \Delta$; the representation is selected on the validation set for each model and dataset (Appendix~\ref{app:classification_model_dynamic}).}

\begin{table}[h!]
\centering
\caption{\textbf{Average performance in the dynamic setting.} Means over the $6$ longitudinal datasets, each averaged over its prediction times. Ranks are computed using the $11$ models listed. Best and second best among the rows shown in \colorbox{blue!25}{dark} and \colorbox{blue!10}{light} blue.}
\label{tab:performance_comparison_dynamic}
\resizebox{\textwidth}{!}{
\begin{tabular}{lccccccc}
\toprule
\multirow{2}{*}{\textbf{Model}} & \multicolumn{2}{c}{\textbf{C-Index}} & \multicolumn{2}{c}{\textbf{Integrated AUC}} & \multicolumn{2}{c}{\textbf{Integrated Brier Score}} & \multirow{2}{*}{\textbf{Avg. Rank $\downarrow$}} \\
\cmidrule(r){2-3} \cmidrule(r){4-5} \cmidrule(r){6-7}
 & Value $\uparrow$ & Rank $\downarrow$ & Value $\uparrow$ & Rank $\downarrow$ & Value $\downarrow$ & Rank $\downarrow$ & \\
\midrule
\multicolumn{8}{l}{\textit{Classification-based Models (CF, ours)}} \\
\textbf{MITRA-CF} & 0.674$_{\pm 0.049}$ & 6.5$_{\pm 1.5}$ & 0.655$_{\pm 0.031}$ & 5.8$_{\pm 1.3}$ & 0.162$_{\pm 0.032}$ & 6.7$_{\pm 1.6}$ & 6.3$_{\pm 1.3}$ \\
\textbf{TabPFN-CF} & 0.636$_{\pm 0.064}$ & 5.3$_{\pm 1.4}$ & 0.642$_{\pm 0.051}$ & 5.4$_{\pm 1.5}$ & 0.164$_{\pm 0.029}$ & 8.0$_{\pm 0.8}$ & 6.3$_{\pm 0.9}$ \\
\textbf{XGBoost-CF} & 0.588$_{\pm 0.028}$ & 8.0$_{\pm 1.1}$ & 0.588$_{\pm 0.041}$ & 7.5$_{\pm 1.6}$ & 0.158$_{\pm 0.029}$ & 6.8$_{\pm 1.3}$ & 7.4$_{\pm 1.2}$ \\
\midrule
\multicolumn{8}{l}{\textit{Discrete-time Hazard Variants (DH)}} \\
\textbf{MITRA-DH} & 0.659$_{\pm 0.036}$ & 5.8$_{\pm 1.1}$ & 0.641$_{\pm 0.046}$ & \cellcolor{blue!10}5.2$_{\pm 1.4}$ & 0.153$_{\pm 0.027}$ & \cellcolor{blue!10}4.0$_{\pm 0.6}$ & 5.0$_{\pm 0.9}$ \\
\textbf{TabPFN-DH} & 0.644$_{\pm 0.044}$ & 5.5$_{\pm 1.2}$ & 0.640$_{\pm 0.054}$ & \cellcolor{blue!10}5.2$_{\pm 1.5}$ & 0.156$_{\pm 0.027}$ & 5.2$_{\pm 0.7}$ & 5.3$_{\pm 1.1}$ \\
\textbf{XGBoost-DH} & 0.609$_{\pm 0.027}$ & 7.5$_{\pm 1.2}$ & 0.625$_{\pm 0.028}$ & 6.3$_{\pm 1.0}$ & \cellcolor{blue!10}0.150$_{\pm 0.027}$ & 4.3$_{\pm 0.9}$ & 6.1$_{\pm 1.0}$ \\
\midrule
\multicolumn{8}{l}{\textit{Dynamic Survival Baselines}} \\
\textbf{Landmark CoxPH} & 0.652$_{\pm 0.023}$ & 6.8$_{\pm 1.0}$ & 0.639$_{\pm 0.042}$ & 6.5$_{\pm 1.5}$ & 0.193$_{\pm 0.055}$ & 6.5$_{\pm 1.4}$ & 6.6$_{\pm 1.2}$ \\
\textbf{Landmark RSF} & 0.652$_{\pm 0.031}$ & 5.3$_{\pm 1.6}$ & 0.651$_{\pm 0.032}$ & 5.7$_{\pm 0.7}$ & \cellcolor{blue!25}\textbf{0.149}$_{\pm 0.027}$ & \cellcolor{blue!25}\textbf{3.5}$_{\pm 0.7}$ & \cellcolor{blue!10}4.8$_{\pm 0.6}$ \\
\textbf{Dynamic DeepHit} & 0.639$_{\pm 0.045}$ & 7.0$_{\pm 1.6}$ & 0.613$_{\pm 0.044}$ & 7.5$_{\pm 1.5}$ & 0.381$_{\pm 0.083}$ & 10.2$_{\pm 0.7}$ & 8.2$_{\pm 1.0}$ \\
\textbf{Joint Model} & \cellcolor{blue!25}\textbf{0.712}$_{\pm 0.056}$ & \cellcolor{blue!10}4.2$_{\pm 1.5}$ & \cellcolor{blue!25}\textbf{0.681}$_{\pm 0.035}$ & \cellcolor{blue!25}\textbf{4.8}$_{\pm 1.1}$ & 0.170$_{\pm 0.034}$ & 5.2$_{\pm 1.5}$ & \cellcolor{blue!25}\textbf{4.7}$_{\pm 1.2}$ \\
\textbf{TCSA} & \cellcolor{blue!10}0.702$_{\pm 0.036}$ & \cellcolor{blue!25}\textbf{4.0}$_{\pm 1.0}$ & \cellcolor{blue!10}0.656$_{\pm 0.036}$ & 6.2$_{\pm 1.5}$ & 0.164$_{\pm 0.029}$ & 5.7$_{\pm 1.6}$ & 5.3$_{\pm 1.1}$ \\
\bottomrule
\end{tabular}
}
\end{table}

\subsection{Results}


\bo{Overall performance (Table~\ref{tab:performance_comparison_dynamic}).} \new{We find that in the dynamic setting, the classification-based models tend to underperform existing dynamic survival baselines: The Joint Model and Landmark RSF have the best average ranks, and the Joint Model and TCSA the highest C-Index. The closest classification models are MITRA-DH and TabPFN-DH. Note though, that with only $6$ datasets, the standard errors are larger than most of the differences between the models. Additionally, note that unlike in the static setting, the best CF does not consistently outperform the best DH model: MITRA-CF has a higher C-Index than MITRA-DH on only $2$ of the $6$ datasets, and for both TFMs the DH model has the lower mean IBS. Note though that (i) the TFM-based CF and DH models outperform Dynamic DeepHit, a neural survival model trained from scratch, in terms of rank on all three metrics, and (ii) the TFM-based CF and DH models outperform the XGBoost formulations on both C-Index and iAUC. A significant limitation of the classification-based methods in our implementation is that the history representation $\phi_k$ that we use does not explicitly model the dynamics of covariate changes over time, which the Joint Model and TCSA do.}

\section{Discussion}
In this work, we demonstrate the utility of a classification-based formulation of survival analysis, specifically as it allows us to leverage pretrained TFMs. \new{In the static setting, MITRA-CF, MITRA-DH, and TabPFN-DH are competitive with or better than the survival baselines. We also found that MITRA performs particularly well with the CF formulation, whereas TabPFN performs well with the DH formulation.} 
Beyond its strong modeling performance in the static setting, the TFM-based classification approaches are also computationally efficient. \new{On a single NVIDIA A100 GPU (80GB), MITRA-CF and MITRA-DH produce the validation and test predictions for all $54$ static datasets at $K = 10$ in about $3$ minutes; this is increased to about $15$ minutes if we include post-hoc calibration.} 



\bo{Limitations.}
Casting continuous event times into $K$ discrete bins introduces an unavoidable loss of temporal resolution, particularly when event times are concentrated or highly non-uniformly distributed. Increasing $K$ mitigates this resolution loss but raises computational cost. 
\new{Additionally, our cumulative failure modeling is theoretically justified only under conditionally independent censoring (Assumptions~\ref{assump:censoring} and \ref{assump:censoring-dynamic}). Informative censoring, where whether an event is censored and when the event occurs are dependent, is common in many real-world problems. A natural direction for future work is to investigate approaches to handle dependent censoring.} 

\bo{Future work.}
Directions for future work include (1) exploring fine-tuning TFMs explicitly for survival analysis, e.g., to explicitly optimize for ranking or C-Index; (2) investigating ways to leverage TFMs for continuous-time survival modeling rather than relying on discretization; (3) extending our approach to handle informative censoring; (4) leveraging TFMs for treatment effect analysis in time-to-event outcomes, a common use case for survival analysis \citep{hosmer2008applied}; \new{and (5) examining alternative ways to improve dynamic survival performance, e.g., by forming richer representations of the history that can model changes in the covariates over time (which Joint Model and TCSA do).}

\section*{Acknowledgements}
K.W.Z is supported by the UK Research and Innovation Engineering and Physical Sciences Research Council grant number UKRI4510. We thank Sonali Parbhoo for feedback on early versions of this work. 

\bibliographystyle{plainnat}
\bibliography{references}

\newpage
\appendix

\section*{Table of Contents for Appendices}
{
\setlength{\parskip}{0pt}
\setlength{\itemsep}{0pt}
\setlength{\parsep}{0pt}
\startcontents[appendices]
\printcontents[appendices]{}{1}{\setcounter{tocdepth}{2}}
}

\clearpage
\section{Theoretical Results}
\label{app:theory}

\subsection{Consistent static survival model}
\label{app:proof-minimizer-static}


\begin{proposition}[Consistent static survival under observed censoring]
    Let $p$ be any binary prediction model for $Y_{i,k}$ given $(X_{i,0}, t_k)$. Under Assumption \ref{assump:censoring} (Conditionally Independent Censoring), the population loss $\tilde{\ell}_{\TN{static}}(p) \triangleq \E \big[ \sum_{k=1}^{K-1} \mathbbm{1}(C_i > t_k) \cdot \TN{BCE}\big( p(X_{i,0}, t_k), Y_{i,k} \big) \big]$ is minimized if and only if, for each $k \in \{ 1, \dots, K-1 \}$,  
\begin{align*} 
    p(x, t_k) 
    = \PP(Y_{i,k} = 1 \mid X_{i,0} = x) 
    = \PP(T_i \leq t_k \mid X_{i,0} = x)
\end{align*}
for all $x$ in the support of $X_{i,0}$ satisfying the positivity condition $\PP(C_i > t_k \mid X_{i,0} = x) > 0$. Note, by convention, we define $0 \log 0 \triangleq 0$ and $x \log 0 = - \infty$ for $x > 0$. 
    \label{prop:minimizer-static-censoring}
\end{proposition}

\begin{proof}
By the definition of $\tilde{\ell}_{\TN{static}}(p)$,
\begin{align}
    &\tilde{\ell}_{\TN{static}}(p) = \E \bigg[ \sum_{k=1}^{K-1} \mathbbm{1}(C_i > t_k) \cdot \TN{BCE}\big( p(X_{i,0}, t_k), Y_{i,k } \big) \bigg]  
    \underbrace{=}_{(a)} \sum_{k=1}^{K-1} \E \left[ \mathbbm{1}(C_i > t_k) \cdot \TN{BCE}\big( p(X_{i,0}, t_k), Y_{i,k } \big) \right] \nonumber \\ 
    &\underbrace{=}_{(b)} \sum_{k=1}^{K-1} \E \bigg[ \E \left[ \mathbbm{1}(C_i > t_k) \cdot \TN{BCE}\big( p(X_{i,0}, t_k), Y_{i,k } \big) \mid X_{i,0} \right] \bigg] \nonumber \\ 
    &\underbrace{=}_{(c)} \sum_{k=1}^{K-1} \E \bigg[ \PP(C_i > t_k \mid X_{i,0}) \cdot \E \left[ \TN{BCE}\big( p(X_{i,0}, t_k), Y_{i,k } \big) \mid X_{i,0} \right] \bigg] \nonumber \\ 
    &\underbrace{=}_{(d)} - \sum_{k=1}^{K-1} \E \bigg[ \PP(C_i > t_k \mid X_{i,0}) \cdot \E \left[ Y_{i,k }\log p(X_{i,0}, t_k) + (1-Y_{i,k}) \log \big(1- p(X_{i,0}, t_k) \big) \mid X_{i,0} \right] \bigg] \nonumber \\
    &\underbrace{=}_{(e)} - \sum_{k=1}^{K-1} \E \bigg[ \PP(C_i > t_k \mid X_{i,0}) \bigg\{ \PP(Y_{i,k} = 1 \mid X_{i,0}) \log p(X_{i,0}, t_k) + \big\{ 1-\PP(Y_{i,k} = 1 \mid X_{i,0}) \big\} \log \big(1- p(X_{i,0}, t_k) \big) \bigg\} \bigg] \label{eqnapp:ellstatic}
\end{align}
Above, equality (a) holds by linearity of expectations. 
Equality (b) holds by the law of iterated expectations.
Equality (c) holds because $C_i \indep Y_{i,k} \mid X_{i,0}$ since $Y_{i,k} \triangleq \mathbbm{1}(T_i \leq t_k)$ and $C_i \indep T_i \mid X_{i,0}$ by Assumption \ref{assump:censoring}.
Equality (d) holds by the definition of BCE from Section \ref{sec:staticMethod}. 
Equality (e) holds because $Y_{i,k}$ is binary, $\E[ Y_{i,k} \mid X_{i,0}] = \PP( Y_{i,k} = 1 \mid X_{i,0})$.

We will now use the result from \eqref{eqnapp:ellstatic} to prove the if and only if statement of the theorem. Note that by Lemma \ref{lemma:minimizer}, 
\begin{align*}
    \PP(Y_{i,k} = 1 \mid X_{i,0} = x) = \argmin_{b \in [0,1]} ~ - \PP(Y_{i,k} = 1 \mid X_{i,0} = x) \log b - \big\{ 1-\PP(Y_{i,k} = 1 \mid X_{i,0} = x) \big\} \log \big(1- b \big)
\end{align*}
and the minimizer is unique. Thus, $p$ is a minimizer of the loss $\tilde{\ell}_{\TN{static}}(p)$ if and only if $p(x, t_k) = \PP(Y_{i,k} = 1 \mid X_{i,0} = x)$ for all $x$ in the support of $X_{i,0}$ for which $\PP(C_i > t_k \mid X_{i,0} = x) > 0$. 
\end{proof}

\begin{lemma}
    Fix $q \in [0,1]$ and let $b \in [0,1]$. The function $f(b) = -q \log b - (1 - q) \log(1-b)$ is uniquely minimized at $b = q$.
    \label{lemma:minimizer}
\end{lemma}

\begin{proof}
In the edge case that $q = 0$, $f(b)$ is minimized at $b = 0$. Also, when $q = 1$, $f(b)$ is minimized at $b = 1$.
Now suppose $q \in (0, 1)$. Note that $f'(b) = -\frac{q}{b} + \frac{1-q}{1-b}$. When $f'(b) = 0$, we get that $b = q$. This must be the global minimizer of $f(b)$ because the second derivative of $f(b)$ is strictly positive: $f''(b) = \frac{q}{b^2} + \frac{1-q}{(1-b)^2} > 0$.
\end{proof}

\subsection{Consistent static survival model via loss minimization with marginal censoring probabilities}
\label{app:consistent-strong}

\begin{assumption}[Independent Censoring]
    \label{assump:censoring-strong}
    The censoring time $C_i$ is conditionally independent of $(T_i, X_{i,0})$.
\end{assumption}

\begin{corollary}[Consistent static survival under observed censoring]
    Let $p$ be any binary prediction model for $Y_{i,k}$ given $(X_{i,0}, t_k)$. Under Assumption \ref{assump:censoring-strong} (Independent Censoring), the population loss $\tilde{\ell}_{\TN{static}}(p) \triangleq \E \big[ \sum_{k=1}^{K-1} \mathbbm{1}(C_i > t_k) \cdot \TN{BCE}\big( p(X_{i,0}, t_k), Y_{i,k} \big) \big]$ is minimized if and only if, for each $k \in \{ 1, \dots, K-1 \}$, 
\begin{align*} 
    p(x, t_k) 
    = \PP(Y_{i,k} = 1 \mid X_{i,0} = x) 
    = \PP(T_i \leq t_k \mid X_{i,0} = x)
\end{align*}
as long as $\PP(C_i > t_k) > 0$. Note, by convention, we define $0 \log 0 \triangleq 0$ and $x \log 0 = - \infty$ for $x > 0$.
    \label{corr:minimizer-static-censoring-strong}
\end{corollary}

\begin{proof}
This result is a direct corollary of Proposition \ref{prop:minimizer-static-censoring}. The only step in which Assumption \ref{assump:censoring} was used in the proof of Proposition \ref{prop:minimizer-static-censoring} is in step (c). Under the stronger Assumption \ref{assump:censoring}, we have that
\begin{align}
    &\tilde{\ell}_{\TN{static}}(p) = \E \bigg[ \sum_{k=1}^{K-1} \mathbbm{1}(C_i > t_k) \cdot \TN{BCE}\big( p(X_{i,0}, t_k), Y_{i,k } \big) \bigg] \nonumber\\
    &= \sum_{k=1}^{K-1} \E \bigg[ \PP(C_i > t_k) \cdot \E \left[ \TN{BCE}\big( p(X_{i,0}, t_k), Y_{i,k } \big) \mid X_{i,0} \right] \bigg] \nonumber \\ 
    &= - \sum_{k=1}^{K-1} \E \bigg[ \PP(C_i > t_k) \bigg\{ \PP(Y_{i,k} = 1 \mid X_{i,0}) \log p(X_{i,0}, t_k) + \big\{ 1-\PP(Y_{i,k} = 1 \mid X_{i,0}) \big\} \log \big(1- p(X_{i,0}, t_k) \big) \bigg\} \bigg] \nonumber
\end{align}
The remainder of the proof uses Lemma \ref{lemma:minimizer}, just as is done in the proof of Proposition \ref{prop:minimizer-static-censoring} and that $\PP(C_i > t_k) > 0$ by assumption.
\end{proof}

\begin{corollary}[Consistent static survival model via loss minimization]
    Let $p$ be any binary prediction model for $Y_{i,k}$ given $(X_{i,0}, t_k)$. Under Assumption \ref{assump:censoring-strong}, the population loss $\ell_{\TN{static}}(p)$ with weights
    \begin{align}
        W_{i,k} \triangleq \begin{cases}
        1 & \TN{if}~ t_k < \min(C_i, T_i) \\
        \frac{ G(t_k) }{ G(T_i) } & \TN{if}~ t_k \geq \min(C_i, T_i) \TN{~and~} \delta_i = 1 \\
        0 & \TN{otherwise}
    \end{cases}
    \label{eqn:weights-marginal}
    \end{align}
    is minimized if and only if, for each $k \in \{ 1, \dots, K-1 \}$, 
\begin{align*} 
    p(x, t_k) 
    = \PP(Y_{i,k} = 1 \mid X_{i,0} = x) 
    = 1-\PP(T_i > t_k \mid X_{i,0} = x)
\end{align*}
as long as $\PP(C_i > t_k) > 0$.
\end{corollary}

\begin{proof}
This result is a direct corollary of Theorem \ref{thm:minimizer-static}.

By Proposition \ref{corr:minimizer-static-censoring-strong}, under Assumption \ref{assump:censoring-strong}, we have that a model $p$ satisfies \eqref{eqn:accurate-survival} for each $k \in \{ 1, \dots, K-1 \}$ if and only if it minimizes the following loss, which assumes that the censoring outcome $C_i$ is always observed (see the discussion in Section \ref{sec:weights}):
\begin{align*}
    \tilde{\ell}_{\TN{static}}(p) \triangleq 
    \E \bigg[ \sum_{k=1}^{K-1} \mathbbm{1}(C_i > t_k) \cdot \TN{BCE}\big( p(X_{i,0}, t_k), Y_{i,k} \big) \bigg].
\end{align*}
Since $Y_{i,k} \triangleq \mathbbm{1}(T_i \leq t_k)$,
\begin{align*}
    \tilde{\ell}_{\TN{static}}(p) &= \sum_{k=1}^{K-1} \E \Big[ \mathbbm{1}(C_i > t_k) \cdot \mathbbm{1}(T_i \leq t_k) \cdot \TN{BCE}\big( p(X_{i,0}, t_k), 1 \big) + \mathbbm{1}(C_i > t_k) \cdot \mathbbm{1}(T_i > t_k) \cdot \TN{BCE}\big( p(X_{i,0}, t_k), 0 \big) \Big] \\ 
    &= \sum_{k=1}^{K-1} \E \Big[ \mathbbm{1}(C_i > t_k) \cdot \mathbbm{1}(T_i \leq t_k) \cdot \TN{BCE}\big( p(X_{i,0}, t_k), 1 \big) + \mathbbm{1}(\min(C_i, T_i) > t_k) \cdot \TN{BCE}\big( p(X_{i,0}, t_k), 0 \big) \Big] 
\end{align*}
By the law of iterated expectations and Assumption \ref{assump:censoring-strong}, $\E \big[ \mathbbm{1}(C_i > t_k) \mid X_{i,0}, T_i \big] 
= \E \big[ \mathbbm{1}(C_i > t_k)\big]
= \PP( C_i > t_k) = G(t_k)$, thus,
\begin{align*}
    &= \sum_{k=1}^{K-1} \E \Big[ G(t_k) \cdot \mathbbm{1}(T_i \leq t_k) \cdot \TN{BCE}\big( p(X_{i,0}, t_k), 1 \big) + \mathbbm{1}(\min(C_i, T_i) > t_k) \cdot \TN{BCE}\big( p(X_{i,0}, t_k), 0 \big) \Big]
\end{align*}
Next, we use that 
$G(t_k) \cdot \mathbbm{1}(T_i \leq t_k)
= \frac{ G(t_k) }{ G(T_i) } \cdot \mathbbm{1}(T_i \leq t_k) \cdot \PP( C_i > T_i \mid T_i) = \frac{ G(t_k ) }{ G(T_i) } \cdot \mathbbm{1}(T_i \leq t_k) \cdot \E \left[ \delta_i \mid T_i \right] \\
= \E \left[ \frac{ G(t_k) }{ G(T_i) } \mathbbm{1}(T_i \leq t_k) \delta_i \,\Big|\, T_i \right] 
= \E \left[ \frac{ G(t_k ) }{ G(T_i) } \mathbbm{1}(\min(C_i,T_i) \leq t_k) \delta_i \,\Big|\, T_i \right]$ (recall that $\delta_i \triangleq \mathbbm{1}(C_i > T_i)$ and that when $\delta_i = 1$ that $T_i = \min(C_i, T_i)$). By the law of iterated expectations, 
\begin{align*}
    &= \sum_{k=1}^{K-1} \E \Bigg[ \frac{ G(t_k) }{ G(T_i) } \mathbbm{1}(\min(C_i, T_i) \leq t_k) \cdot \delta_i \cdot \TN{BCE}\big( p(X_{i,0}, t_k), 1 \big) + \mathbbm{1}(\min(C_i, T_i) > t_k) \cdot \TN{BCE}\big( p(X_{i,0}, t_k), 0 \big) \Bigg] \\
    &= \sum_{k=1}^{K-1} \E \Bigg[ \bigg\{ \frac{ G(t_k) }{ G(T_i) } \mathbbm{1}(\min(C_i, T_i) \leq t_k) \cdot \delta_i + \mathbbm{1}(\min(C_i, T_i) > t_k) \bigg\} \cdot \TN{BCE}\big( p(X_{i,0}, t_k), Y_{i,k} \big) \Bigg].
\end{align*}
Finally, the above equals $\sum_{k=1}^{K-1} \E \Big[ W_{i,k} \cdot \TN{BCE}\big( p(X_{i,0}, t_k), Y_{i,k} \big) \Big]$ using the definition of $W_{i,k}$ from \eqref{eqn:weights-marginal}.
\end{proof}

\subsection{Consistent dynamic survival model}
\label{app:proof-minimizer-dynamic}

Analogously to the static setting, in the dynamic setting, we can show that as the amount of training samples $n \to \infty$, minimizing the loss from \eqref{eqn:empirical-loss-dynamic} will necessarily lead the classifier to learn the true underlying survival probabilities. Below, we define $\ell_{\text{dynamic}}(p)$, which is the population version of the loss from \eqref{eqn:empirical-loss-dynamic}. 
\begin{align}
    \ell_{\text{dynamic}}(p) &\triangleq \E \bigg[ \sum_{k=1}^{K-2} \sum_{\Delta=1}^{K-k-1} \mathbbm{1}(t_{k} < T_i) \cdot W_{i,k,k+\Delta} 
    \cdot \text{BCE} \big( p(\mathcal{H}_{i,t_k}, 
    t_{k+\Delta}), Y_{i,k+\Delta} \big) \bigg]. \label{eqn:lossDynamic}
\end{align}
Theorem \ref{thm:minimizer-dynamic} formalizes how minimizing the loss $\ell_{\text{dynamic}}(p)$ is equivalent to learning the underlying survival probabilities, despite censoring, given that the dynamic censoring Assumption \ref{assump:censoring-dynamic} holds. Theorem \ref{thm:minimizer-dynamic} uses Proposition \ref{prop:minimizer-dynamic} as a key step.

\begin{proposition}[Consistent dynamic survival under observed censoring]
    \label{prop:minimizer-dynamic}
    Let $p$ be any binary prediction model for $Y_{i,k+\Delta}$ given $(\HH_{i,t_{k}}, 
    t_{k+\Delta})$. Under Assumption \ref{assump:censoring-dynamic} (Dynamic Conditionally Independent Censoring), the population loss
    \begin{align*}
        \tilde{\ell}_{\TN{dynamic}}(p) \triangleq \E \bigg[ \sum_{k=1}^{K-2} \sum_{\Delta=1}^{K-k-1} \mathbbm{1}(t_{k} < T_i) \cdot \mathbbm{1}(t_{k+\Delta} < C_i) 
    \cdot \TN{BCE} \big( p(\mathcal{H}_{i,t_k}, 
    t_{k+\Delta}), Y_{i,k+\Delta} \big) \bigg]
    \end{align*}
    is minimized if and only if, for each $k \in \{ 1, \dots, K-2 \}$ and $\Delta \in \{ 1, \dots, K-k-1 \}$,
    \begin{align*} 
        p(h, 
        t_{k+\Delta}) 
        = \PP(Y_{i,k+\Delta} = 1 \mid \HH_{i,t_{k}}=h)  
        = 1 - \PP(T_i > t_{k+\Delta} \mid \HH_{i,t_{k}}=h) 
    \end{align*}
    for all $h$ in the support of $\HH_{i,t_{k}}$ that (i) satisfy the positivity condition $\PP(C_i > t_{k+\Delta} \mid \HH_{i,t_{k}}=h) > 0$ and (ii) have not had the event occur by $t_k$, i.e., $y_1 = \dots = y_k = 0$. Note, by convention, we define $0 \log 0 \triangleq 0$ and $x \log 0 = - \infty$ for $x > 0$. 
\end{proposition}

\begin{proof}
By the definition of $\tilde \ell_{\TN{dynamic}}(p)$,
\begin{align}
    \tilde \ell_{\TN{dynamic}}(p) &= \E \bigg[ \sum_{k=1}^{K-2} \sum_{\Delta=1}^{K-k-1} \mathbbm{1}(T_i > t_{k}) \cdot \mathbbm{1}(C_i > t_{k+\Delta}) \cdot \TN{BCE}\big( p(\mathcal{H}_{i,t_k}, t_{k+\Delta}), Y_{i,k+\Delta} \big) \bigg] \nonumber \\ 
    &\underbrace{=}_{(a)} \sum_{k=1}^{K-2} \sum_{\Delta=1}^{K-k-1} \E \left[ \mathbbm{1}(T_i > t_{k}) \cdot \mathbbm{1}(C_i > t_{k+\Delta}) \cdot \TN{BCE}\big( p(\mathcal{H}_{i,t_k}, t_{k+\Delta}), Y_{i,k+\Delta} \big) \right] \nonumber \\ 
    &\underbrace{=}_{(b)} \sum_{k=1}^{K-2} \sum_{\Delta=1}^{K-k-1} \E \bigg[ \mathbbm{1}(T_i > t_{k}) \cdot \E \left[ \mathbbm{1}(C_i > t_{k+\Delta}) \cdot \TN{BCE}\big( p(\mathcal{H}_{i,t_k}, t_{k+\Delta}), Y_{i,k+\Delta} \big) \mid \mathcal{H}_{i,t_k} \right] \bigg] \nonumber \\ 
    &\underbrace{=}_{(c)} \sum_{k=1}^{K-2} \sum_{\Delta=1}^{K-k-1} \E \bigg[ \mathbbm{1}(T_i > t_{k}) \cdot \PP \left( C_i > t_{k+\Delta} \mid \mathcal{H}_{i,t_k} \right) \cdot \E \left[ \TN{BCE}\big( p(\mathcal{H}_{i,t_k}, t_{k+\Delta}), Y_{i,k+\Delta} \big) \mid \mathcal{H}_{i,t_k} \right] \bigg] \nonumber \\ 
    &\underbrace{=}_{(d)} - \sum_{k=1}^{K-2} \sum_{\Delta=1}^{K-k-1} \E \bigg[ \mathbbm{1}(T_i > t_{k}) \cdot \PP \left( C_i > t_{k+\Delta} \mid \mathcal{H}_{i,t_k} \right) \nonumber \\ 
    & \qquad \qquad \qquad \qquad 
    \times \E \big[ Y_{i,k+\Delta} \log p(\mathcal{H}_{i,t_k}, t_{k+\Delta}) + (1-Y_{i,k+\Delta}) \log \big(1- p(\mathcal{H}_{i,t_k}, t_{k+\Delta}) \big) \mid \mathcal{H}_{i,t_k} \big] \bigg] \nonumber \\
    &\underbrace{=}_{(e)} - \sum_{k=1}^{K-2} \sum_{\Delta=1}^{K-k-1} \E \bigg[ \mathbbm{1}(T_i > t_{k}) \cdot \bigg\{ \PP \left( C_i > t_{k+\Delta} \mid \mathcal{H}_{i,t_k} \right) \cdot \PP(Y_{i,k+\Delta} = 1 \mid \mathcal{H}_{i,t_k}) \log p(\mathcal{H}_{i,t_k}, t_{k+\Delta}) \nonumber \\ 
    &\qquad \qquad \qquad \qquad \qquad + \big\{ 1 - \PP(Y_{i,k+\Delta} = 1 \mid \mathcal{H}_{i,t_k}) \big\} \log \big(1- p(\mathcal{H}_{i,t_k},  t_{k+\Delta}) \big) \mid \mathcal{H}_{i,t_k}\big] \bigg\} \bigg] \label{eqnapp:elldynamic}
\end{align}
Above, equality (a) holds by linearity of expectations. 
Equality (b) holds by the law of iterated expectations and since $\mathbbm{1}(T_i > t_{k})$ is known given $\HH_{i,t_k}$.
Equality (c) holds because $C_i \indep Y_{i,k+\Delta} \mid \mathcal{H}_{i,t_k}$ since $Y_{i,k+\Delta} \triangleq \mathbbm{1}(T_i \leq t_{k+\Delta})$ and $T_i \indep C_i \mid (\HH_{i,t}, T_i > t)$ by Assumption \ref{assump:censoring-dynamic}. 
Equality (d) holds because $Y_{i,k+\Delta}$ is binary, $\E[ Y_{i,k+\Delta} \mid \mathcal{H}_{i,t_k}] = \PP( Y_{i,k+\Delta} = 1 \mid \mathcal{H}_{i,t_k})$.

We will now use the result from \eqref{eqnapp:elldynamic} to prove the if and only if statement of the theorem. Note that by Lemma \ref{lemma:minimizer},
\begin{align*}
    \PP(Y_{i,k+\Delta} = 1 \mid \mathcal{H}_{i,t_k}=h) &= \argmin_{b \in [0,1]} \big\{ -\PP(Y_{i,k+\Delta} = 1 \mid \mathcal{H}_{i,t_k}=h) \log b \\ 
    &\qquad \qquad \qquad - \big\{ 1-\PP(Y_{i,k+\Delta} = 1 \mid \mathcal{H}_{i,t_k}=h) \big\} \log \big(1- b \big) \big\}, 
\end{align*}
and the minimizer is unique. Thus, $p$ is a minimizer of the loss $\tilde{\ell}_{\TN{dynamic}}(p)$ if and only if $p(\mathcal{H}_{i,t_k}=h, t_{k+\Delta}) = \PP(Y_{i,k+\Delta} = 1 \mid \mathcal{H}_{i,t_k}=h)$ if and only if for all $h, y$ in the support of $\mathcal{H}_{i,t_k}$ for which $\PP(C_i > t_{k+\Delta} \mid \mathcal{H}_{i,t_k}=h) > 0$ and for which the event has not occurred by $t_k$, i.e., $y_1 = \dots = y_k = 0$. 
\end{proof}

\begin{theorem}[Consistent dynamic survival model via loss minimization]
    \label{thm:minimizer-dynamic}
    Let $p$ be any binary prediction model for $Y_{i,k+\Delta}$ given $(\HH_{i,t_{k}}, t_{k+\Delta})$. Under Assumption \ref{assump:censoring-dynamic} (Dynamic Conditionally Independent Censoring), the population loss $\ell_{\TN{dynamic}}(p)$ from \eqref{eqn:lossDynamic} is minimized if and only if, for each $k \in \{ 1, \dots, K-2 \}$ and $\Delta \in \{ 1, \dots, K-k-1 \}$,
    \begin{align} 
        p(h, t_{k+\Delta}) 
        = \PP(Y_{i,k+\Delta} = 1 \mid \HH_{i,t_{k}}=h) 
        = 1 - \PP(T_i > t_{k+\Delta} \mid \HH_{i,t_{k}}=h)
        \label{eqn:accurate-survival-dynamic}
    \end{align}
    for all $h$ in the support of $\HH_{i,t_{k}}$ that (i) satisfy the positivity condition $\PP(C_i > t_{k+\Delta} \mid \HH_{i,t_{k}}=h) > 0$ and (ii) have not had the event occur by $t_k$, i.e., $y_1 = \dots = y_k = 0$. Note, by convention, we define $0 \log 0 \triangleq 0$ and $x \log 0 = - \infty$ for $x > 0$. 
\end{theorem}

\begin{proof}
By Proposition \ref{prop:minimizer-dynamic}, under Assumption \ref{assump:censoring-dynamic}, we have that a model $p$ satisfies \eqref{eqn:accurate-survival-dynamic} for each $k \in \{ 1, \dots, K-2 \}$ and $\Delta \in \{ 1, \dots, K-k-1 \}$ if and only if it minimizes the following loss, which assumes that the censoring outcome $C_i$ is always observed: 
\begin{align}
    \tilde{\ell}_{\TN{dynamic}}(p) \triangleq 
    \E \bigg[ \sum_{k=1}^{K-2} \sum_{\Delta=1}^{K-k-1} \mathbbm{1}(t_{k} < T_i) \cdot \mathbbm{1}(t_{k+\Delta} < C_i) 
    \cdot \text{BCE} \big( p(\mathcal{H}_{i,t_k}, t_{k+\Delta}), Y_{i,k+\Delta} \big) \bigg].
    \label{eqn:loss-censor-known-dynamic}
\end{align}
The core idea of the proof is showing that the population loss under known censoring in \eqref{eqn:loss-censor-known-dynamic} equals the population loss under weighting from \eqref{eqn:lossDynamic}.
Note that since $Y_{i,k} \triangleq \mathbbm{1}(T_i \leq t_k)$,
\begin{multline*}
    \tilde{\ell}_{\TN{dynamic}}(p) = \sum_{k=1}^{K-2} \sum_{\Delta=1}^{K-k-1} \E \bigg[ \mathbbm{1}(t_{k} < T_i) \cdot \mathbbm{1}(t_{k+\Delta} < C_i) 
    \cdot \Big\{ \mathbbm{1}(T_i \leq t_{k+\Delta}) \cdot \text{BCE} \big( p(\mathcal{H}_{i,t_k}, t_{k+\Delta}), 1 \big) \\ 
    \qquad \qquad \qquad \qquad \qquad \qquad \qquad \qquad \qquad \qquad \qquad \qquad + \mathbbm{1}(T_i > t_{k+\Delta}) \cdot \text{BCE} \big( p(\mathcal{H}_{i,t_k}, t_{k+\Delta}), 0 \big) \Big\} \bigg] 
\end{multline*}
\vspace{-7mm}
\begin{multline*}
    = \sum_{k=1}^{K-2} \sum_{\Delta=1}^{K-k-1} \E \bigg[ \mathbbm{1}(t_{k} < T_i) \cdot \mathbbm{1}(t_{k+\Delta} < C_i) 
    \cdot  \mathbbm{1}(T_i \leq t_{k+\Delta}) \cdot \text{BCE} \big( p(\mathcal{H}_{i,t_k}, t_{k+\Delta}), 1 \big) \\ 
    \qquad \qquad \qquad \qquad \qquad \qquad \qquad \qquad + \mathbbm{1}(t_{k} < T_i) \cdot \mathbbm{1}( \min(C_i, T_i) > t_{k+\Delta})  \cdot \text{BCE} \big( p(\mathcal{H}_{i,t_k}, t_{k+\Delta}), 0 \big) \bigg] 
\end{multline*}
By the law of iterated expectations and Assumption \ref{assump:censoring-dynamic}, $\E \big[ \mathbbm{1}(C_i > t_{k+\Delta}) \mid \HH_{i,t_k}, T_i \big] 
= \E \big[ \mathbbm{1}(C_i > t_{k+\Delta}) \mid \HH_{i,t_k} \big]
= \PP( C_i > t_{k+\Delta} \mid \HH_{i,t_k}) = G(t_{k+\Delta} \mid \HH_{i,t_k})$, thus,
\begin{multline*}
    = \sum_{k=1}^{K-2} \sum_{\Delta=1}^{K-k-1} \E \bigg[ \mathbbm{1}(t_{k} < T_i) \cdot G(t_{k+\Delta} \mid \HH_{i,t_k}) 
    \cdot  \mathbbm{1}(T_i \leq t_{k+\Delta}) \cdot \text{BCE} \big( p(\mathcal{H}_{i,t_k}, t_{k+\Delta}), 1 \big) \\ 
    \qquad \qquad \qquad \qquad \qquad \qquad \qquad \qquad + \mathbbm{1}(t_{k} < T_i) \cdot \mathbbm{1}( \min(C_i, T_i) > t_{k+\Delta})  \cdot \text{BCE} \big( p(\mathcal{H}_{i,t_k}, t_{k+\Delta}), 0 \big) \bigg] 
\end{multline*}
Next, we use that 
$G(t_{k+\Delta} \mid \HH_{i,t_k}) \cdot \mathbbm{1}(T_i \leq t_{k+\Delta})
= \frac{ G(t_{k+\Delta} \mid \HH_{i,t_k}) }{ G(T_i \mid \HH_{i,t_k}) } \cdot \mathbbm{1}(T_i \leq t_{k+\Delta}) \cdot \PP( C_i > T_i \mid \HH_{i,t_k}, T_i) \\
= \frac{ G(t_{k+\Delta} \mid \HH_{i,t_k}) }{ G(T_i \mid \HH_{i,t_k}) } \cdot \mathbbm{1}(T_i \leq t_{k+\Delta}) \cdot \E \left[ \delta_i \mid \HH_{i,t_k}, T_i \right] 
= \E \left[ \frac{ G(t_{k+\Delta} \mid \HH_{i,t_k}) }{ G(T_i \mid \HH_{i,t_k}) } \mathbbm{1}(T_i \leq t_{k+\Delta}) \delta_i \,\Big|\, \HH_{i,t_k}, T_i \right] \\ = \E \left[ \frac{ G(t_{k+\Delta} \mid \HH_{i,t_k}) }{ G(T_i \mid \HH_{i,t_k}) } \mathbbm{1}(\min(C_i,T_i) \leq t_{k+\Delta}) \delta_i \,\Big|\, \HH_{i,t_k}, T_i \right]$ (recall that $\delta_i \triangleq \mathbbm{1}(T_i \leq C_i)$, which equals $\mathbbm{1}(C_i > T_i)$ with probability one when $T_i$ and $C_i$ have no common atoms, and that when $\delta_i = 1$ that $T_i = \min(C_i, T_i)$). By the law of iterated expectations, 
\begin{multline*}
    = \sum_{k=1}^{K-2} \sum_{\Delta=1}^{K-k-1} \E \bigg[ \mathbbm{1}(t_{k} < T_i) \cdot \frac{ G(t_{k+\Delta} \mid \HH_{i,t_k}) }{ G(T_i \mid \HH_{i,t_k}) } 
    \cdot  \mathbbm{1}(\min(C_i, T_i) \leq t_{k+\Delta}) \cdot \delta_i \cdot \text{BCE} \big( p(\mathcal{H}_{i,t_k}, t_{k+\Delta}), 1 \big) \\ 
    \qquad \qquad \qquad \qquad \qquad \qquad \qquad \qquad + \mathbbm{1}(t_{k} < T_i) \cdot \mathbbm{1}( \min(C_i, T_i) > t_{k+\Delta})  \cdot \text{BCE} \big( p(\mathcal{H}_{i,t_k}, t_{k+\Delta}), 0 \big) \bigg] 
\end{multline*}
\vspace{-7mm}
\begin{multline*}
    = \sum_{k=1}^{K-2} \sum_{\Delta=1}^{K-k-1} \E \bigg[ \mathbbm{1}(t_{k} < T_i) \cdot \bigg\{ \frac{ G(t_{k+\Delta} \mid \HH_{i,t_k}) }{ G(T_i \mid \HH_{i,t_k}) } 
    \cdot  \mathbbm{1}(\min(C_i, T_i) \leq t_{k+\Delta}) \delta_i + \mathbbm{1}( \min(C_i, T_i) > t_{k+\Delta}) \bigg\} \\
    \text{BCE} \big( p(\mathcal{H}_{i,t_k}, t_{k+\Delta}), Y_{i,k+\Delta} \big) \bigg] 
\end{multline*}
Finally, the above equals $\sum_{k=1}^{K-2} \sum_{\Delta=1}^{K-k-1}  \E \big[ \mathbbm{1}(t_{k} < T_i) \cdot W_{i,k,k+\Delta} 
\cdot \TN{BCE} \big( p(\mathcal{H}_{i,t_k}, 
t_{k+\Delta}), Y_{i,k+\Delta} \big) \big]$ using the definition of weights $W_{i,k,k+\Delta}$ from \eqref{eqn:weights-dynamic}.
\end{proof}

\clearpage
\section{Metrics}
\label{app:metrics}

This appendix provides formal definitions, corresponding estimators and implementation details for the evaluation metrics used in the \emph{static setting} and \emph{dynamic setting}.
All survival metrics are implemented using the \texttt{sksurv} library \citep{polsterl2020scikit}. See our codebase for additional details. All three metrics correct for censoring with inverse-probability-of-censoring weights (IPCW), i.e., weights $1/\hat G$, where $\hat G$ estimates the probability of remaining uncensored. We use the marginal Kaplan--Meier estimator $\hat G(t)$ of $G(t) \triangleq \PP(C_i > t)$, the default of \texttt{sksurv}, fitted on the training data. In the dynamic setting, $\hat G$ is fitted at each landmark on the training observations still at risk.
Note, that these metrics are therefore only consistent under a stronger assumption than Assumptions~\ref{assump:censoring} and \ref{assump:censoring-dynamic}, namely that $C_i$ is independent of $T_i$ and of the covariates.

\subsection{Static Metrics}

\subsubsection{Concordance Index (C-Index)}
\label{app:static_cindex}
The concordance index (C-Index) is a measure of a model's ability to accurately rank individuals by their risk of experiencing an event \citep{harrell1982evaluating,uno2011c}. We say a pair of samples is concordant if the individual with a shorter time-to-event has a higher risk score assigned by a model. 

Recall from display~\eqref{eqn:staticcindex}, the C-Index measures the following in the static setting
\begin{equation}
\label{eq:cindex_static}
C_{\TN{Index}} \triangleq \mathbb{P}\!\left(\hat{r}(X_{i,0}) > \hat{r}(X_{j,0}) \mid T_i < T_j\right).
\end{equation}
where $\hat r$ is a risk scoring model. A common approach to correct for censoring is by using Inverse-Probability-of-Censoring-Weights (IPCW), which provides a consistent estimator \citep{uno2011c}. 
We first consider a truncated version of display~\eqref{eq:cindex_static}:
\begin{equation}
\label{eq:cindex_target}
C_{\TN{Index}} \triangleq
\mathbb{P} \left(\hat{r}(X_{i,0}) > \hat{r}(X_{j,0}) \,\middle|\, T_i < T_j, \; T_i \leq T_{\TN{max}} \right),
\end{equation}
where we set $T_{\TN{max}}$ to be the maximum observed timepoint in the training and validation sets (which are used to fit the final static models). 
Truncation is used to ensure that the predicted probability of not being censored remains positive, which is required when using IPCW weighting. Following \citet{uno2011c}, an IPCW estimator of display~\eqref{eq:cindex_target} is
\begin{equation*}
\label{eq:c_index_ipcw}
\hat{C}_{\TN{Index}} = 
\frac{ \sum_{i \not= j} \hat{G}(T_i)^{-2} \delta_i \mathbbm{1}\big(T_i < T_j, T_i \leq T_{\TN{max}} \big) \mathbbm{1}\big(\hat{r}(X_{i,0}) > \hat{r}(X_{j,0}) \big) }{ \sum_{i \not= j} \delta_i \hat{G}(T_i)^{-2} \mathbbm{1}\big(T_i < T_j, T_i \leq T_{\TN{max}} \big) },  
\end{equation*}
where $G(t)$ denotes the population censoring function at time $t$. Formally,
\begin{equation}
    \label{eqn:censored_survival}
    G(t) \triangleq \PP ( C_i > t),
\end{equation}
and, for the results in the main text, $\hat G(t)$ is the marginal Kaplan-Meier estimator \citep{kaplan1958nonparametric,polsterl2020scikit} (see the estimators listed at the beginning of this appendix).

\subsubsection{Integrated AUC}
\label{app:static_auc}
The AUC (area under the ROC curve) at time $t$ is a measure of how well a model can distinguish individuals who experience an event by time $t$ from those who do not. The weighted mean of these AUCs tells us about the overall discriminative ability of a survival model.

Recall from display~\eqref{eqn:static_auc}, the time-dependent AUC measures the following in the static setting
\begin{equation}
\label{eq:auc_target}
\mathrm{AUC}(t) \triangleq
\mathbb{P}\!\left(\hat{r}(X_{i,0}) > \hat{r}(X_{j,0}) \,\middle|\,
T_i \leq t, \; T_j > t\right).
\end{equation}

A consistent IPCW estimate \citep{hung2010estimation} of display~\eqref{eq:auc_target} is given by
\begin{equation}
\widehat{\mathrm{AUC}}(t) = 
\frac{ \sum_{i\not=j} \hat{G}(T_i)^{-1} \, \delta_i \,\mathbbm{1}\big(T_i \leq t, \; T_j > t \big) \,\mathbbm{1}\big(\hat{r}(X_{i,0}) > \hat{r}(X_{j,0}) \big)}{ \sum_{i\not=j} \hat{G}(T_i)^{-1} \, \delta_i \,\mathbbm{1} \big(T_i \leq t, \; T_j > t\big)},
\label{eqn:AUCest}
\end{equation}
where ${G}(t)$ is the censoring function at time $t$; see display~\eqref{eqn:censored_survival}. We evaluate display~\eqref{eqn:AUCest} at $t_k$ for $k \in \{1, \dots, K-1\}$, and compute a weighted average of the resulting AUCs, using the estimated survival function as weights
\begin{equation*}
\overline{\mathrm{AUC}}
=
\frac{\sum_{k'=1}^{K-1} \widehat{\mathrm{AUC}}(t_{k'})\, w_{t_{k'}}}
{\sum_{k'=1}^{K-1} w_{t_{k'}}},
\label{eqn:aveAUC}
\end{equation*}
where $w_{t_k} = \hat S(t_{k-1})-\hat S(t_{k})$, and $\hat{S}(t)$ is the Kaplan-Meier estimator of the marginal survival function: $S(t) \triangleq \PP (T_i > t)$ \citep{kaplan1958nonparametric}.

\subsubsection{Integrated Brier Score (IBS)}

The Integrated Brier Score quantifies a model's calibration and accuracy by averaging over time, the squared error between the predicted survival probabilities and the observed survival outcomes.
Recall from display~\eqref{eqn:static_brier}, the time-dependent Brier score at time $t$ is defined as the mean squared error between the binary outcome $\mathbbm{1}(T_i>t)$ and the predicted survival function
\begin{equation}
\label{eq:bs_target}
\mathrm{BS}(t)
\triangleq
\mathbb{E}\Big[\big(\mathbbm{1}(T_i>t)-\hat{S}(t\mid X_{i,0})\big)^2\Big].
\end{equation}
A standard IPCW estimator of display~\eqref{eq:bs_target} from \citet{kvamme2023brier} is
\begin{equation}
\label{eq:bs_def}
\widehat{\mathrm{BS}}(t) =
\frac{1}{n}\sum_{i=1}^n
\left[
\frac{\mathbbm{1}(T_i \le t,\,\delta_i=1)}{\hat{G}(T_i)}
\big(0-\hat{S}(t \mid X_{i,0})\big)^2
+
\frac{\mathbbm{1}(T_i > t)}{\hat{G}(t)}
\big(1-\hat{S}(t \mid X_{i,0})\big)^2
\right],
\end{equation}
where ${G}(t)$ is the censoring function at time $t$, defined in display~\eqref{eqn:censored_survival}. We estimate the Integrated Brier Score by evaluating display~\eqref{eq:bs_def} at time points 
$t_k$, $k \in \{1,\ldots,K-1\}$, and then integrating these values over time
\begin{equation*}
\label{eq:ibs_target}
    \overline{\TN{BS}} = \frac{1}{t_{K-1}-t_1} \int_{t_1}^{t_{K-1}} \widehat{\mathrm{BS}}(t) \, d t.
\end{equation*}
We use the trapezoidal rule to approximate the integral.

\subsection{Dynamic Metrics}
\label{app:dynnamic_metrics}

\paragraph{Risk scores.}
A risk score is a real-valued prediction of an individual's risk, where higher values correspond to a greater risk of experiencing the event. We denote the risk scoring model by $\hat r$. In the static setting, $\hat r (X_{i,0})$ denotes a risk score for individual $i$ given baseline covariates $X_{i,0}$. In the dynamic setting, $\hat r (\HH_{i,t_k})$ denotes the time-dependent risk score for individual $i$ at time $t_k$, computed using the history available up to time $t_k$. 

\subsubsection{Concordance Index (C-Index)}
In the dynamic setting, the C-Index at time $t$ is defined as
\begin{equation}
\label{eq:cindex_dynamic}
C_{\TN{Index}}(t) \triangleq \mathbb{P}\!\left(\hat{r}(\HH_{i,t}) > \hat{r}(\HH_{j,t}) \mid t <T_i < T_j \right),
\end{equation}
where the risk scores are computed using history available up to $t$. Analogous to the static setting (see Appendix \ref{app:static_cindex}), we use an IPCW estimator for \eqref{eq:cindex_dynamic}. In the dynamic setting, the metric is restricted to the observations still at risk at time $t$:
\begin{equation*}
\hat{C}_{\TN{Index}} (t) =
\frac{ \sum_{i \not= j} \hat{G}(T_i\mid t)^{-2} \delta_i \mathbbm{1}\big(t < T_i < T_j \big) \mathbbm{1}\big(\hat{r}(\HH_{i,t}) > \hat{r}(\HH_{j,t}) \big) }{ \sum_{i \not= j} \delta_i \hat{G}(T_i \mid t)^{-2} \mathbbm{1}\big(t < T_i < T_j \big) }, 
\end{equation*}
where $G(s \mid t)$, for $s > t$, denotes the conditional censoring function, or formally
\begin{equation}
\label{eqn:conditional_censoring}
G(s \mid t) \triangleq \PP(C_i > s \mid C_i > t),
\end{equation}
which is the probability of remaining uncensored up to time $s$ given being uncensored at time $t$. Similar to the static setting, $\hat G(s \mid t)$ is, for the results in the main text, the Kaplan-Meier estimator \citep{kaplan1958nonparametric,polsterl2020scikit} computed on the training observations at risk at $t$. 

\subsubsection{Integrated AUC}
In the dynamic setting, the time dependent AUC at time $t + \Delta$, given survival up to $t$ is given by
\begin{equation}
\mathrm{AUC}(t + \Delta \mid t) \triangleq \mathbb{P} \left( \hat{r}(\HH_{i,t}) > \hat{r}(\HH_{j,t}) \mid t < T_i \leq t + \Delta < T_j\right)
\label{eqn:dynamicAUC}  
\end{equation}

Analogous to the static setting (see Appendix~\ref{app:static_auc}), we form the following IPCW estimate of display~\eqref{eqn:dynamicAUC}
\begin{equation*}
\widehat{\mathrm{AUC}}(t+\Delta \mid t) = 
\frac{ \sum_{i\not=j} \hat{G}(T_i \mid t)^{-1} \, \delta_i \,\mathbbm{1}\big(t < T_i \leq t+\Delta < T_j \big) \,\mathbbm{1}\big(\hat{r}(\HH_{i,t}) > \hat{r}(\HH_{j,t}) \big)}{ \sum_{i\not=j} \hat{G}(T_i\mid t)^{-1} \, \delta_i \,\mathbbm{1} \big(t < T_i \leq t+ \Delta < T_j\big)},
\label{eqn:AUCest_dynamic}
\end{equation*}
where ${G}(t + \Delta \mid t)$ is the conditional censoring function; see display~\eqref{eqn:conditional_censoring}. 

At each conditioning time $t_k$, we select the future evaluation timepoints $t_{k+\Delta}$ from the discretized horizons. 
For example, if the time grid is $t_1,\ldots,t_4$ and we condition on survival up to $t_2$, we compute the time dependent AUC for $\mathrm{AUC}(t_3 \mid t_2)$ and $\mathrm{AUC}(t_4 \mid t_2)$, which assess the accuracy of the predicted survival probabilities at times $t_3$ and $t_4$ given information available up to $t_2$. 

Under our discretized setting, we estimate the Integrated AUC at $t_k$ by
\begin{equation*}
\overline{\mathrm{AUC}}(t_k)
=
\frac{\sum_{\Delta=1}^{K-k-1} \widehat{\mathrm{AUC}}(t_{k+\Delta} \mid t_k)\, w_{t_{k+\Delta}}}
{\sum_{\Delta=1}^{K-k-1} w_{t_{k+\Delta}}},
\label{eqn:aveAUC_dynamic}
\end{equation*}

where $w_{t_{k+\Delta}} = \hat S(t_{k +\Delta -1} \mid t_k)-\hat S(t_{k+\Delta}\mid t_k)$, and $\hat S(t_{k+\Delta} \mid t_k)$ is the Kaplan--Meier estimator, computed on the test observations at risk at $t_k$, $S(t_{k+\Delta} \mid t_k) \triangleq \PP(T_i > t_{k+\Delta} \mid T_i > t_k)$, as in the static setting.

\subsubsection{Integrated Brier Score (IBS)}
In the dynamic setting, the time dependent Brier Score at time $t +\Delta$, given survival up to $t$ can be defined as
\begin{equation}
\label{eq:bs_dynamic}
\TN{BS}(t+\Delta\mid t) \triangleq
\mathbb{E}
\Big[
\Big(
\mathbbm{1}\{T_i > t + \Delta\}
-
\hat{S}(t+\Delta \mid T_i > t, \;\HH_{i,t})
\Big)^2 \Big]
.   
\end{equation}
An analogous estimator for display~\eqref{eq:bs_dynamic} is 
\begin{multline*}
\label{eq:bs_dynamic_est}
\widehat{\TN{BS}}(t + \Delta \mid t)
=
\frac{1}{n}\sum_{i=1}^n
\Bigg[
\frac{\mathbbm{1}(t <T_i \le t +\Delta,\,\delta_i=1)}{\hat{G}(T_i \mid t)}
\big(0-\hat{S}(t +\Delta \mid T_i > t, \;\HH_{i,t})\big)^2 
\\
+\frac{\mathbbm{1}(T_i > t+ \Delta)}{\hat{G}(t+ \Delta \mid t)}
\big(1-\hat{S}(t + \Delta \mid T_i > t, \;\HH_{i,t})\big)^2
\Bigg],
\end{multline*}
where ${G}(t+ \Delta\mid t)$ is the conditional censoring function, as in display~\eqref{eqn:conditional_censoring}.

Similar to the dynamic Integrated AUC, for each conditioning time $t_k$, we select the future evaluation timepoints $t_{k+\Delta}$ from the discretized horizons. Under this discretized setting, the Integrated Brier Score at time $t_k$ can be estimated by 
\begin{equation*}
\label{eq:ibs_target_dynamic}
    \overline{\TN{BS}}(t_k) = \frac{1}{t_{K-1}-t_{k+1}} \int_{t_{k+1}}^{t_{K-1}} \widehat{\mathrm{BS}}(t \mid t_k) \, d t.
\end{equation*}
The integral is approximated with the trapezoidal rule over the evaluation times $t_{k+1}, \dots, t_{K-1}$ that lie within the follow-up of the test observations at risk at $t_k$, and divided by the length of this time range, as in \texttt{sksurv}. If only one evaluation time remains, $\overline{\TN{BS}}(t_k)$ is the Brier score at that time.

\clearpage
\section{Modeling Details}
\label{app:baseline}

\subsection{Static Models}
\label{app:static_baselines}

\subsubsection{TFM-based classification models}
\label{app:classificationModel} 
The context windows of tabular foundation models are generally limited ($8,192$ for MITRA and $50,000$ for TabPFN).
To manage this, after constructing the supervised learning dataset as described in Section~\ref{sec:staticMethod}
\begin{itemize}[leftmargin=*]
    \item For MITRA, we sample uniformly at random 8,192 examples $\big( (X_{i,0}, t_k), Y_{i,k} \big)$ from the training set.
    \item For TabPFN, for our static experiments, we feed all examples $\big( (X_{i,0}, t_k), Y_{i,k} \big)$ from the training set (with TabPFN’s \texttt{ignore\_pretraining\_limits} option). For the dynamic setting, we subsample 50,000 examples from the training set uniformly at random.
\end{itemize}

\subsubsection{Cox Proportional Hazards (CoxPH)}
\label{app:static_cox}

We use the \texttt{sksurv} package to implement the CoxPH model.\footnote{\url{https://scikit-survival.readthedocs.io/en/stable/api/generated/sksurv.linear_model.CoxPHSurvivalAnalysis.html}}
The CoxPH model relates the baseline covariates $X_{i,0}$ to the hazard of an event at time $t$ using $\lambda(t\mid X_{i,0})=\lambda_0(t)\exp((X_{i,0})^\top \beta)$,
where $\beta$ is the vector of regression coefficients, and $\lambda_0(t)$ is the baseline hazard function, a population-level function defined as the hazard for the baseline covariate $X=\mathbf{0}$. 
\begin{equation}
\label{eqn:baseline_hazard}
\lambda_0(t) \triangleq \lambda(t \mid X_{i,0} = \mathbf{0}).
\end{equation}
The coefficients $\beta$ are estimated by maximizing the Cox partial likelihood function \citep{cox1972regression}. Let $\hat{\beta}$ denote the fitted coefficients. The risk score for an individual $i$ is defined as the inner product between the covariates and the coefficients:
\begin{equation*}
\hat r (X_{i,0}) = X_{i,0}^\top \hat \beta.
\end{equation*}
The corresponding survival function is 
\begin{equation*}
\hat S(t\mid X_{i,0}) = \hat S_0(t)^{\exp(X_{i,0}^\top \hat \beta)},
\end{equation*}
where \(\hat S_0(t)\) is the estimated baseline survival function, i.e., the survival function for the reference covariate value \(X_{i,0}=\mathbf{0}\).
Formally, the baseline survival function is
\begin{equation}
\label{eqn:baseline_survival}
S_0(t) \triangleq \mathbb{P} (T>t \mid X_{i,0} = \mathbf{0}),
\end{equation}
and \(\hat S_0(t)\) is estimated using the Breslow estimator \citep{breslow1974covariance}.

\subsubsection{Random Survival Forests (RSF)}
We use the \texttt{sksurv} RSF implementation.\footnote{\url{https://scikit-survival.readthedocs.io/en/stable/api/generated/sksurv.ensemble.RandomSurvivalForest.html}} RSF is a nonparametric ensemble model that extends random forests to right-censored survival data by learning an ensemble of tree-based cumulative hazard functions \citep{ishwaran2008random}. 
An RSF consists of an ensemble of $B$ survival trees. Each tree partitions the feature space into terminal nodes, and for each terminal node a cumulative hazard function is estimated from the training samples in that node using the Nelson-Aalen estimator \citep{nelson1972theory, aalen1978nonparametric}. For an individual $i$, let $\ell_b(X_{i,0})$ denote the terminal node in tree $b$ to which $X_{i,0}$ is assigned. The tree-level cumulative hazard is given by $\hat H_b(t \mid X_{i,0}) = \hat H_{b,\ell_b(X_{i,0})}(t)$.
The forest-level cumulative hazard is obtained by averaging over trees: $\hat H(t \mid X_{i,0}) = \frac{1}{B} \sum_{b=1}^B \hat H_b(t \mid X_{i,0})$.
The risk score for an individual $i$ can be computed as follows:
\begin{equation*}
\hat r (X_{i,0}) = \sum_{t' \in \MC{T}} \hat H(t' \mid X_{i,0}),
\end{equation*}
where $\MC{T}$ is a set that contains all the distinct event times observed in the training data. The corresponding survival function is the \texttt{sksurv} forest prediction. 

\subsubsection{DeepSurv}
We use the \texttt{pycox} implementation of DeepSurv.\footnote{\url{https://github.com/havakv/pycox}} DeepSurv is a neural-network extension of the CoxPH model in which the linear risk function is replaced by a nonlinear function learned by a deep neural network \citep{katzman2018deepsurv}.
In DeepSurv, the hazard for an individual with covariates $X_{i,0}$ is $\lambda(t \mid X_{i,0}) = \lambda_0(t)\exp\!\big(f_\theta(X_{i,0})\big)$,
where $\lambda_0(t)$ is the baseline hazard defined in display~\eqref{eqn:baseline_hazard} and $f_\theta(X_{i,0})$ is a neural network with parameters $\theta$ that replaces the linear term $(X_{i,0})^\top \beta$ in CoxPH. The network weights $\theta$ are then learned by maximizing the Cox partial likelihood.
After training, the risk score for an individual $i$ is
\begin{equation*}
\hat r (X_{i,0}) = f_\theta(X_{i,0}),
\end{equation*}
and the corresponding survival function is
\begin{equation*}
\hat S(t \mid X_{i,0}) = \hat S_0(t)^{\exp(f_\theta(X_{i,0}))},
\end{equation*}
where $\hat S_0(t)$ is the estimated baseline survival function defined in display~\eqref{eqn:baseline_survival}; we use the Breslow estimator \citep{breslow1974covariance}.

\subsubsection{DeepHit}
We use the \texttt{pycox} implementation of DeepHit.\footnote{\url{https://github.com/havakv/pycox}} 
DeepHit is a neural-network–based, nonparametric survival model that directly learns the discrete-time distribution of event times without assuming proportional hazards \citep{lee2018deephit}. 
The network is a multi-layer perceptron with batch normalization, trained with Adam (batch size $256$; $100$ epochs during the hyperparameter search and $300$ epochs for the selected configuration) and the DeepHit loss with $\alpha = 0.2$ and $\sigma = 0.1$.

\subsubsection{BCESurv}
\label{app:bcesurv}

We use the \texttt{pycox} implementation of BCESurv \citep{kvamme2023brier}, which fits an MLP model using the loss from Section \ref{sec:bce-surv}; we call this model MLP BCESurv.\footnote{\url{https://github.com/havakv/pycox}}
The network is a multi-layer perceptron with two hidden layers, batch normalization, and dropout, trained with Adam (batch size $256$, at most $200$ epochs) with early stopping on the validation loss (patience $20$). Covariates are standardized using the training set.
We also implement XGBoost BCESurv, which minimizes the same loss with an XGBoost-based model. 


\subsubsection{MTLR}
\label{app:mtlr}
We use the \texttt{pycox} implementation of MTLR \citep{yu2011learning,fotso2018deep}. 
For \textbf{Neural MTLR}, the model is a multi-layer perceptron with two hidden layers, batch normalization, and dropout. We trained with Adam (batch size $256$, $200$ epochs), keeping the weights with the lowest validation loss. 
Note that we use the \texttt{pycox} implementation rather than the original \texttt{R} implementation of MTLR. 

\subsection{Dynamic Models}
\paragraph{Last-Observation-Carried-Forward (LOCF).}
We employ a slightly modified version of the last-observation-carried-forward (LOCF) method in several of the dynamic models to represent the information available up to a particular time \citep{little2019statistical}. Recall from display~\eqref{eqn:survival-dynamic} that dynamic survival prediction focuses on the conditional survival probability for a future horizon $\Delta > 0$, given the information available at time $t$. Since the covariates are usually observed at irregular intervals, their covariate values are generally not available exactly at time $t$. We therefore apply the LOCF method, where the time-dependent covariate at time $t$ is imputed by its most recent observed value prior to $t$.

Recall from Section~\ref{sec:dynamic_problem_statement} that the observed covariate history for individual $i$ up to time $t$ is $\mathcal H_{i,t}
\triangleq
\{(s, X_{i,s}) : s \in \mathcal T_i,\; s \le t\}$,
where $\mathcal T_i$ denotes the set of observation times for individual $i$.
Define the most recent observation time prior to $t$ by
\begin{equation*}
s_i(t)
\triangleq
\arg\max_{s \in \mathcal T_i:\, s \le t} s .
\end{equation*}
The LOCF representation of the covariate history at time $t$ is then
\begin{equation}
\label{eqn:locf}
X_{i,t}^{\mathrm{LOCF}}
\;\triangleq\;
X_{i,s_i(t)} .
\end{equation}

\subsubsection{Classification-Based Models}
\label{app:classification_model_dynamic}

\paragraph{Representation of history $\phi$ during training.}
Recall from Section \ref{sec:dynamicMethod}, the classification-based dynamic models (MITRA, TabPFN, and XGBoost) take the vector-valued function $\phi_k(\HH_{i,t_k}, Y_{i,1:k}, t_k, t_{k+\Delta})$ as input, and predict the binary label $Y_{i,k+\Delta}$ during training. Let $\bar X_{i,t_k}$ be the running means of the covariates computed up to $t_k$ (we use running modes for categorical covariates, which are then one-hot encoded). Missing values are imputed with the training median, and an observation without any covariates $X_{i,t_k}$ before $t_k$ is assigned the mean $X_{i,t_k}$ across the entire training set. We consider the following feature representations:
\begin{itemize}
    \item Base representation: $\phi_k\!\left(\mathcal H_{i,t_k},\, Y_{i,1:k},\, t_k,\, t_{k+\Delta}\right) =
    \Big[
    X^{\mathrm{LOCF}}_{i,t_k},\;
    \bar X_{i,t_k}, \;
    t_k,\;
    t_{k+\Delta}
    \Big]$
    \item Representation incorporating the time elapsed since the most recent observation prior to $t_k$: \\ 
    $\phi_k\!\left(\mathcal H_{i,t_k},\, Y_{i,1:k},\, t_k,\, t_{k+\Delta}\right) =
    \Big[
    X^{\mathrm{LOCF}}_{i,t_k},\;
    \bar X_{i,t_k},\;
    t_k,\;
    t_{k+\Delta}, \;
    t_k - s_i(t_k)
    \Big]$
    \item Representation incorporating the info (i.e. the index) of the prediction horizon: \\
    $\phi_k\!\left(\mathcal H_{i,t_k},\, Y_{i,1:k},\, t_k,\, t_{k+\Delta}\right) =
    \Big[
    X^{\mathrm{LOCF}}_{i,t_k},\;
    \bar X_{i,t_k},\;
    t_k,\;
    t_{k+\Delta},\;
    k+\Delta
    \Big]$
    \item Full Representation: \\
    $\phi_k\!\left(\mathcal H_{i,t_k},\, Y_{i,1:k},\, t_k,\, t_{k+\Delta}\right) =
    \Big[
    X^{\mathrm{LOCF}}_{i,t_k},\;
    \bar X_{i,t_k},\;
    t_k,\;
    t_{k+\Delta},\;
    t_k - s_i(t_k),\;
    k+\Delta
    \Big]$
\end{itemize}
See Appendix \ref{app:hyperparameters} on how we select which representation to use.

\paragraph{Clipping Survival Probabilities.}

Since $\hat S(t_{k+\Delta} \mid T_i > t_k, \mathcal H_{i,t_k})$ must be non-increasing in $\Delta$, we apply a similar monotonicity post-processing described in Section~\ref{sec:staticMethod}. We refine display~\eqref{eqn:dynamic_inference} by the following
\begin{equation*}
\hat S\!\left(t_{k+\Delta} \mid T_i > t_k,\; \mathcal H_{i,t_k}\right)
\triangleq
\min_{1 \le \ell \le \Delta}
1- \hat{p}(\HH_{i,t_k}, Y_{i,1:k}, t_k, t_{k+ \ell}), 
\qquad \Delta = 1, \dots, K-1-k.
\end{equation*}
This ensures that $\hat S(t_{k+1}\mid t_k,\; \mathcal H_{i,t_k})\ge \hat S(t_{k+2}\mid t_k,\; \mathcal H_{i,t_k} )\ge \cdots \ge \hat S(t_{K-1}\mid t_k,\; \mathcal H_{i,t_k} )$.

The time-dependent risk score is defined by averaging the predicted future cumulative risks over the remaining time horizon. For individual $i$ at time $t_k$, the risk score is 
\begin{equation*}
\hat r (\HH_{i,t_k})
=
\frac{1}{K-1-k}
\sum_{\Delta=1}^{K-1-k}
\Bigl(1 -
\hat S\!\left(t_{k+\Delta} \mid T_i > t_k,\; \mathcal H_{i,t_k}\right)
\Bigr).
\end{equation*}

\subsubsection{Landmark CoxPH}
\label{app:landmark_cox}

Since the Cox Proportional Hazards (CoxPH) model requires static covariates, we adopt the landmarking approach using LOCF \citep{little2019statistical}. Again, we use the \texttt{sksurv} package to implement the CoxPH model. For each $t_k$, the covariate history is represented by
\begin{equation}
\label{eqn:locf_landmark}
    \phi_{\TN{landmark}}(\HH_{i,t_k}) \triangleq 
    ( X_{i,t_k}^{\mathrm{LOCF}}, \; 
    \bar X_{i,t_k})
\end{equation} 
where we also included the running means of the covariates computed up to $t_k$, $\bar X_{i,t_k}$. For each $t_k$, $k \in \{1,\dots,K-2\}$, we fit a separate CoxPH model using only individuals who are still at risk (the event did not occur and is not censored) at $t_k$. Each landmark CoxPH model yields regression coefficient estimates $\hat\beta_{t_k}$. Analogous to the static setting of the CoxPH model (see Appendix \ref{app:static_cox}), the time‐dependent risk score for individual $i$ at $t_k$ can be computed as
\begin{equation*}
\hat r (\HH_{i,t_k})
=
\phi_{\TN{landmark}}(\HH_{i,t_k})^\top \hat \beta_{t_k}.
\end{equation*}

Survival predictions over a horizon $\Delta >0$ are obtained conditionally on survival up to $t_k$:
\begin{equation*}
\hat S\!\left(t_{k+\Delta} \mid T_i > t_k,\; \mathcal H_{i,t_k}\right)
=
\hat S_{0}\left(t_{k+\Delta} \mid t_k \right)^{\exp(\hat r (\HH_{i,t_k}))}.
\end{equation*}
where \(\hat S_{0}(t_{k+\Delta} \mid t_k)\) is the estimated baseline survival function, i.e., the survival function for a reference individual with zero covariates. Formally, the baseline survival function conditional on time $t_k$ is
\begin{equation}
\label{eqn:baseline_survival_landmark}
S_{0}(t_{k+\Delta} \mid t_k)
\;\triangleq\;
\! \mathbb{P}\left(
T > t_{k+\Delta}
\;\middle|\;
T > t_k,\;
\phi_{\TN{landmark}}(\HH_{i,t_k}) = 0
\right),
\end{equation}
and \(\hat S_{0}(t_{k+\Delta} \mid t_k)\) is estimated using the Breslow estimator \citep{breslow1974covariance}.

\subsubsection{Landmark Random Survival Forest}
Similar to the landmark CoxPH model, at each landmark time $t_k$, $k \in \{1,\dots,K-2\}$, we fit a separate random survival forest (RSF) using only individuals who are still at risk at $t_k$, $\phi_{\TN{landmark}}(\HH_{i,t_k})$ defined in display~\eqref{eqn:locf_landmark} as a representation of the history. 
Again, we use the \texttt{sksurv} RSF implementation.
Each landmark‐specific forest estimates a cumulative hazard function for $t \geq t_k$, $\hat H_{t_k}\!\left(t \mid T_i > t_k,\; \mathcal H_{i,t_k}\right)$.
The time‐dependent risk score for individual $i$ at landmark time $t_k$ is defined as 
\begin{equation*}
\hat r (\HH_{i,t_k})
=
\sum_{t' \in \MC{T}}
\hat H_{t_k}\!\left(t' \mid T_i > t_k,\; \mathcal H_{i,t_k}\right),
\end{equation*}
where $\MC{T}$ is a set that contains all the distinct event times observed in the training data.
Conditional on survival up to the landmark time, the survival probability over a horizon $\Delta$ is the \texttt{sksurv} forest prediction $\hat S_{t_k}(t_{k+\Delta} \mid T_i > t_k, \HH_{i,t_k})$. 

\subsubsection{Dynamic DeepHit}
Dynamic DeepHit \citep{lee2019dynamic} extends DeepHit to the dynamic setting by modeling the conditional distribution of future event times given an individual’s longitudinal history. We use the following public implementation: \url{https://github.com/Jeanselme/DynamicDeepHit/tree/main}

We use a fixed discretization with $K$ time bins that aligns with the time discretization approach described in Section~\ref{sec:method}. The implementation of Dynamic DeepHit takes as input the longitudinal history $\HH_{i,t_k}$ and outputs the estimates of the cumulative incidence distribution,
\begin{equation*}
\hat p_{t_k,k +\Delta}(\HH_{i,t_k})
\approx
\mathbb{P}\!\left(T_i \le t_{k+\Delta} \,\middle|\, T_i > t_k,\; \mathcal H_{i,t_k}\right),
\qquad \Delta = 1,\dots,K-k-1.
\end{equation*}
We define the time-dependent risk score at time $t_k$ as the predicted probability of experiencing the event in the next interval,
\begin{equation*}
\hat r (\HH_{i,t_k})
=
\hat p_{t_k,k +1}(\HH_{i,t_k}).
\end{equation*}
The conditional survival probability at $t_{k+\Delta}$ given survival up to $t_k$ is
\begin{equation*}
\hat S\!\left(t_{k+\Delta} \mid T_i > t_k,\; \mathcal H_{i,t_k}\right)
=
1 - \; \hat p_{t_k,k +\Delta}(\HH_{i,t_k}).
\end{equation*}

\subsubsection{Joint Model}
The joint model links a longitudinal process and a survival outcome through shared subject-specific random effects \citep{ibrahim2010basic}. We implement the joint model using the \texttt{JMbayes2} package.\footnote{\url{https://cran.r-project.org/web/packages/JMbayes2/index.html}} We model the longitudinal process using a linear mixed-effects model, while the survival outcome is modeled using a CoxPH model.

Given survival up to a time $t_k$ and the observed longitudinal history $\mathcal H_{i,t_k}$ the joint model yields dynamic survival predictions at future times $t_{k+\Delta} > t_k$. At each $t_k$, the predicted event probability at a future time point $t_{k+\Delta}$ is denoted using $\hat p_{t_k,k+\Delta}\!\left(\mathcal H_{i,t_k}\right)$,
which can be obtained by calling the \texttt{predict()} function in the \texttt{JMbayes2} package.
We define the time-dependent risk score at $t_k$ as the predicted probability the event occurs at the next interval
\begin{equation*}
\hat r\!\left(\mathcal H_{i,t_k}\right)
=
\hat p_{t_k,k+1}\!\left(\mathcal H_{i,t_k}\right),
\end{equation*}
and the conditional survival probability at $t_{k+\Delta}$ can be computed using
\begin{equation*}
\hat S\!\left(t_{k+\Delta} \mid T_i > t_k,\; \mathcal H_{i,t_k}\right)
=
1 - \hat p_{t_k,k+\Delta}\!\left(\mathcal H_{i,t_k}\right).
\end{equation*}

\subsubsection{Temporally-Consistent Survival Analysis (TCSA)}
\label{app:tcsa}
TCSA \citep{maystre2022temporally} models dynamic survival analyses as a Markovian process. 
We use the reference implementation (\texttt{tdsurv}) with its proportional-hazards model, which has a nonparametric baseline over the horizon.
The original experiments of \citet{maystre2022temporally} place one step at each clinic visit, but we form the time points with a pre-specified grid (as for all our dynamic models). We select $\lambda \in \{0, 0.2, 0.5, 1\}$ and the $\ell_2$ penalty in $\{0.01, 0.1, 1, 10\}$ by the validation IPCW C-Index, and fit with $30$ fixed point iterations.

\subsection{Hyperparameters}
\label{app:hyperparameters}

\paragraph{Model selection approach.}
For each dataset, we select
hyperparameters using the best performing model in terms of \emph{C-Index} (IPCW estimator) on the validation split.

\paragraph{Static search spaces}
For the baselines, we perform a small grid search:
\begin{itemize}
  \item \textbf{CoxPH:} $\ell_2$ regularization strength
  $\alpha \in \{10^{-6},\,10^{-3},\,10^{-1}\}$.
  \item \textbf{RSF:} number of trees $n_{\text{estimators}} \in \{50,100,200\}$,
  min split $\in \{2,5,10,20\}$.
  \item \textbf{DeepSurv:} dropout $\in \{0,0.1\}$,
  hidden layers $\in\{(128,32),(256,32),(512,32)\}$,
  learning rate $\in \{10^{-3},10^{-4},10^{-5}\}$.
  \item \textbf{DeepHit:} dropout $\in \{0,0.1\}$,
  hidden layers $\in\{(128,32),(256,32),(512,32)\}$,
  learning rate $\in \{10^{-3},10^{-4},10^{-5}\}$.
  \item \textbf{MLP BCESurv:} dropout $\in \{0.1,0.3,0.5\}$,
  hidden layers $\in\{32,64,128\}\times\{16,32,64\}$,
  learning rate $\in \{10^{-2},10^{-3},10^{-4}\}$.
  \item \textbf{MTLR:} learning rate $\in \{10^{-2},10^{-3},10^{-4}\}$,
  weight decay $\in \{10^{-5},10^{-4},10^{-3}\}$.
  \item \textbf{Neural MTLR:} dropout $\in \{0.1,0.3,0.5\}$,
  hidden layers $\in\{32,64,128\}\times\{16,32,64\}$,
  learning rate $\in \{10^{-2},10^{-3},10^{-4}\}$, weight decay $10^{-3}$.
  \item \textbf{XGBoost Models:} XGBoost-CF, XGBoost-DH, XGBoost-MC, and XGBoost BCESurv use $100$ trees, learning rate $0.1$, and maximum depth $3$.
\end{itemize}


\paragraph{Dynamic search spaces}
For the baselines, we perform a small grid search:
\begin{itemize}
  \item \textbf{Landmark CoxPH:} $\ell_2$ regularization strength
  $\alpha \in \{10^{-6},\,10^{-3},\,10^{-1}\}$.
  \item \textbf{Landmark RSF:} number of trees $n_{\text{estimators}}\in\{100,200,500,1000\}$, max depth $\in\{3,10,50,\texttt{None}\}$, min split $\in\{10,50\}$.
  \item \textbf{Dynamic DeepHit:} RNN hidden units $\in\{50,100\}$, fully connected hidden layers $\in\{(100,100),(100,100,100)\}$, loss parameter $\sigma \in\{1,3\}$, learning rate $\in\{10^{-4}, 10^{-3},5\times10^{-3},10^{-2}\}$.
  \item \textbf{Joint Model:} the default priors of \texttt{JMbayes2}, without tuning.
  \item \textbf{TCSA:} proportional-hazards model class of the reference implementation (nonparametric baseline hazard shared across steps), $\lambda \in \{0, 0.2, 0.5, 1\}$ ($\lambda=1$: landmarking), $\ell_2$ regularization strength $\in \{0.01, 0.1, 1, 10\}$; $30$ outer iterations of the TD backups for $\lambda<1$.
\end{itemize}

Recall from Appendix~\ref{app:classification_model_dynamic} the representations of the features $\phi$. For all classifier-based dynamic models (MITRA, TabPFN, and XGBoost), we tune how
temporal information is encoded in the input features using two binary
hyperparameters:
\begin{itemize}
  \item \texttt{quantile\_info}: whether the index of the prediction horizon is included as an additional covariate;
  \item \texttt{time\_since}: whether the time elapsed since the most recent
  observation prior to the timepoint $t_k$ is included as a covariate.
\end{itemize}
Each binary flag takes values in $\{0,1\}$, so there are four combinations of feature representations. For XGBoost, we additionally tune the tree ensemble parameters: \texttt{n\_estimators} $\in\{100, 200\}$, \texttt{learning\_rate} $\in\{0.1, 0.3\}$, and \texttt{max\_depth} $\in\{3, 6\}$.

\clearpage
\section{Additional Experimental Details and Results}

\subsection{Datasets and Preprocessing}
\label{app:datasets}

\begin{table}[h!]
\centering
\caption{Summary statistics for each static survival dataset.}
\label{tab:static_datasets}
\footnotesize
\begin{tabular}{rlrrr}
\toprule
No. & Dataset & Covariates & Samples & Censor Rate \\
\midrule
1 & actg & 11 & 1151 & 91.7 \\
2 & Aids2 & 4 & 2839 & 38.0 \\
3 & Bergamaschi & 10 & 82 & 65.9 \\
4 & breast & 4 & 100 & 74.0 \\
5 & burn & 11 & 154 & 68.8 \\
6 & cancer & 8 & 228 & 27.6 \\
7 & cgd & 10 & 128 & 65.6 \\
8 & colon & 9 & 929 & 51.3 \\
9 & cost & 13 & 518 & 22.0 \\
10 & d.oropha.rec & 12 & 192 & 27.6 \\
11 & dataDIVAT1 & 5 & 5943 & 83.6 \\
12 & dataDIVAT2 & 4 & 1837 & 68.3 \\
13 & dataDIVAT3 & 7 & 4267 & 94.4 \\
14 & dataOvarian1 & 160 & 912 & 40.4 \\
15 & diabetes & 4 & 394 & 60.7 \\
16 & Dialysis & 4 & 6805 & 76.4 \\
17 & divorce & 3 & 3371 & 69.4 \\
18 & e1684 & 3 & 284 & 31.0 \\
19 & flchain & 10 & 7874 & 72.5 \\
20 & follic & 5 & 541 & 21.6 \\
21 & Framingham & 7 & 4699 & 68.7 \\
22 & GBSG2 & 8 & 686 & 56.4 \\
23 & glioma & 4 & 37 & 37.8 \\
24 & grace & 5 & 1000 & 67.6 \\
25 & hepatoCellular & 43 & 227 & 57.3 \\
26 & Melanoma & 5 & 205 & 72.2 \\
27 & mgus & 9 & 241 & 6.6 \\
28 & nki70 & 75 & 144 & 66.7 \\
29 & nwtco & 7 & 4028 & 85.8 \\
30 & ova & 5 & 358 & 25.7 \\
31 & ovarian & 4 & 26 & 53.8 \\
32 & pbc & 6 & 312 & 59.9 \\
33 & Pbc3 & 12 & 349 & 82.5 \\
34 & pharmacoSmoking & 11 & 125 & 28.8 \\
35 & phpl04K8a & 21 & 442 & 46.6 \\
36 & prostate & 15 & 502 & 29.5 \\
37 & rdata & 4 & 1040 & 47.4 \\
38 & retinopathy & 7 & 394 & 60.7 \\
39 & rhc & 53 & 5735 & 35.1 \\
40 & rott2 & 11 & 2982 & 57.3 \\
41 & scania & 5 & 1931 & 43.8 \\
42 & smarto & 26 & 3873 & 88.1 \\
43 & stagec & 7 & 146 & 63.0 \\
44 & support2 & 35 & 9105 & 31.9 \\
45 & TRACE & 6 & 1878 & 49.0 \\
46 & uis & 8 & 628 & 19.1 \\
47 & UnempDur & 6 & 3241 & 38.7 \\
48 & Unemployment & 5 & 452 & 43.4 \\
49 & veteran & 6 & 137 & 6.6 \\
50 & vlbw & 22 & 617 & 82.7 \\
51 & whas500 & 16 & 461 & 61.8 \\
52 & wpbc & 32 & 198 & 76.3 \\
53 & Z243 & 9 & 100 & 0.0 \\
54 & zinc & 13 & 431 & 81.2 \\
\bottomrule
\end{tabular}
\end{table}

\begin{table}[h!]
\centering
\caption{\bo{Excluded static datasets.} (i): at least as many covariates (after one-hot encoding) as training samples; (ii): more than $10{,}000$ samples; (iii): a validation event at a time where the Kaplan--Meier estimate of the training censoring distribution is $0$ (numbered as in the main text). The criteria are checked on the training and validation sets only.}
\label{tab:static_excluded}
\small
\begin{tabular}{lrrrccc}
\toprule
Dataset & Total Samples & Training Samples & Covariates & (i) & (ii) & (iii) \\
\midrule
acath & 3504 & 2452 & 3 &  &  & \checkmark \\
AML\_Bull & 116 & 80 & 6283 & \checkmark &  & \\
chop & 414 & 289 & 3833 & \checkmark &  & \\
DBCD & 295 & 205 & 4919 & \checkmark &  & \\
DLBCL & 240 & 167 & 7399 & \checkmark &  & \\
gse1992 & 124 & 86 & 15541 & \checkmark &  & \\
gse3143 & 158 & 110 & 8663 & \checkmark &  & \checkmark \\
gse4335 & 115 & 79 & 12815 & \checkmark &  & \\
hdfail & 52422 & 36693 & 36 &  & \checkmark & \\
LeukSurv & 1043 & 729 & 29 &  &  & \checkmark \\
MCLcleaned & 92 & 64 & 574 & \checkmark &  & \\
micro.censure & 117 & 81 & 81 & \checkmark &  & \\
NSBCD & 115 & 79 & 549 & \checkmark &  & \\
prostateSurvival & 14294 & 10004 & 6 &  & \checkmark & \\
vdv & 78 & 54 & 4705 & \checkmark &  & \\
\bottomrule
\end{tabular}
\end{table}

\paragraph{Dynamic.}



Across the $6$ dynamic datasets \citep{drysdale2022survset}, the median number of samples is $512$ (range $103$--$4603$). The number of covariates ranges from $4$ to $7$. The censoring rate ranges from $27.2\%$ to $89.7\%$. Of the other two dynamic datasets of \emph{SurvSet}, \texttt{heartvalve} has no $K \geq 4$ with at least $5$ training events per interval, and \texttt{Rossi} meets criterion (iii) (follow-up ends at $52$ weeks for every sample. The criteria are applied to the training and validation observations, not to the rows of the longitudinal data.


\begin{table}[h!]
\centering
\caption{\bo{Summary statistics for each dynamic survival dataset.} We also report the discretization granularity $K$ that we use for each dataset.}
\label{tab:dynamic_datasets}
\begin{tabular}{|c|c|c|c|c|c|}
\hline
\textbf{No.} & \textbf{Dataset} & \textbf{No. of covariates} & \textbf{No. of Samples} & \textbf{Censoring Rate} & $K$ \\
\hline
1 & aids & 5 & 467 & 59.7 & 6 \\
\hline
2 & csl & 6 & 446 & 39.5 & 8 \\
\hline
3 & epileptic & 5 & 556 & 66.2 & 7 \\
\hline
4 & heart & 4 & 103 & 27.2 & 5 \\
\hline
5 & oldmort & 7 & 4603 & 57.2 & 10 \\
\hline
6 & FRTCS & 5 & 697 & 89.7 & 8 \\
\hline
\end{tabular}
\end{table}


\bo{Data splitting and imputation procedure.} We split each dataset into training, validation, and test sets (70:15:15); in the static setting, we split using a stratified
procedure on the event indicator to preserve the event-censoring ratio
across splits. In the dynamic setting, we split randomly without stratification.
Categorical covariates were then one-hot encoded. 
Missing covariate values were imputed with the training mean (static) or median (dynamic) using a \texttt{SimpleImputer}\footnote{\url{https://scikit-learn.org/stable/modules/generated/sklearn.impute.SimpleImputer.html}} fit on the training data only and subsequently applied to the validation and test sets. 

\clearpage
\subsection{Static results across different discretization granularities \texorpdfstring{$K$}{K}}
\label{app:K_sensitivity}

We report performance across different numbers of discretization granularities $K \in \{5, 10, 15, 20\}$ in Table~\ref{tab:static_by_k} at each $K$.
We find that for MITRA, CF has the better mean C-Index and IBS at every $K$. For TabPFN, DH has the better mean Integrated AUC and IBS at every $K$, and the better C-Index for $K \leq 15$. 
In Figure~\ref{fig:K_sensitivity_group}, we additionally show results for the best performing model in each class as a function of $K$.

\begin{table}[h!]
\centering
\caption{\textbf{Static setting for each number of intervals $K$} (KM grid). Mean test value over the datasets in which every model has the metric ($K=5$: $54$, $K=10$: $54$, $K=15$: $53$, $K=20$: $53$ datasets; Integrated AUC: $53$, $52$, $50$, $49$). MITRA-MC and TabPFN-MC support at most $10$ classes. Best and second best (ties share the colour) in \colorbox{blue!25}{dark} and \colorbox{blue!10}{light} blue.} 
\label{tab:static_by_k}
\resizebox{\textwidth}{!}{
\begin{tabular}{l cccc cccc cccc}
\toprule
 & \multicolumn{4}{c}{\textbf{C-Index} $\uparrow$} & \multicolumn{4}{c}{\textbf{Integrated AUC} $\uparrow$} & \multicolumn{4}{c}{\textbf{IBS} $\downarrow$} \\
\cmidrule(r){2-5} \cmidrule(r){6-9} \cmidrule(r){10-13}
\textbf{Model} & $K=5$ & $K=10$ & $K=15$ & $K=20$ & $K=5$ & $K=10$ & $K=15$ & $K=20$ & $K=5$ & $K=10$ & $K=15$ & $K=20$ \\
\midrule
\multicolumn{13}{l}{\textit{Cumulative-Failure Models}} \\
MITRA-CF & \cellcolor{blue!10}0.694 & \cellcolor{blue!25}\textbf{0.699} & \cellcolor{blue!25}\textbf{0.702} & \cellcolor{blue!25}\textbf{0.702} & \cellcolor{blue!10}0.758 & \cellcolor{blue!10}0.751 & \cellcolor{blue!25}\textbf{0.752} & \cellcolor{blue!25}\textbf{0.748} & \cellcolor{blue!25}\textbf{0.155} & \cellcolor{blue!25}\textbf{0.152} & \cellcolor{blue!25}\textbf{0.150} & \cellcolor{blue!25}\textbf{0.150} \\
TabPFN-CF & 0.672 & 0.660 & 0.663 & 0.665 & 0.739 & 0.721 & 0.726 & 0.720 & \cellcolor{blue!10}0.157 & 0.160 & 0.158 & 0.157 \\
XGBoost-CF & 0.682 & 0.687 & 0.695 & 0.688 & 0.738 & 0.732 & 0.740 & 0.737 & 0.159 & 0.158 & \cellcolor{blue!10}0.154 & \cellcolor{blue!10}0.153 \\
\midrule
\multicolumn{13}{l}{\textit{Discrete-time Hazard Models}} \\
MITRA-DH & 0.686 & 0.680 & 0.684 & 0.685 & 0.754 & 0.744 & \cellcolor{blue!10}0.748 & \cellcolor{blue!10}0.745 & 0.160 & 0.159 & 0.157 & 0.157 \\
TabPFN-DH & \cellcolor{blue!25}\textbf{0.695} & 0.676 & 0.667 & 0.663 & \cellcolor{blue!25}\textbf{0.759} & \cellcolor{blue!25}\textbf{0.752} & 0.742 & 0.733 & \cellcolor{blue!25}\textbf{0.155} & \cellcolor{blue!10}0.156 & 0.156 & 0.156 \\
XGBoost-DH & 0.691 & 0.689 & \cellcolor{blue!10}0.700 & 0.683 & 0.743 & 0.733 & 0.740 & 0.734 & 0.162 & 0.163 & 0.158 & 0.159 \\
\midrule
\multicolumn{13}{l}{\textit{Multi-class Models}} \\
MITRA-MC & 0.620 & 0.636 & -- & -- & 0.705 & 0.695 & -- & -- & 0.289 & 0.299 & -- & -- \\
TabPFN-MC & 0.622 & 0.652 & -- & -- & 0.707 & 0.703 & -- & -- & 0.291 & 0.295 & -- & -- \\
XGBoost-MC & 0.616 & 0.633 & 0.632 & 0.620 & 0.688 & 0.691 & 0.691 & 0.686 & 0.284 & 0.291 & 0.288 & 0.295 \\
\midrule
\multicolumn{13}{l}{\textit{Other Discrete-time Classification Models}} \\
MTLR & 0.666 & 0.688 & 0.665 & 0.662 & 0.691 & 0.727 & 0.723 & 0.709 & 0.188 & 0.173 & 0.162 & 0.166 \\
Neural MTLR & 0.668 & 0.682 & 0.663 & 0.670 & 0.712 & 0.716 & 0.706 & 0.694 & 0.191 & 0.176 & 0.173 & 0.173 \\
MLP BCESurv & 0.621 & 0.647 & 0.647 & 0.678 & 0.655 & 0.698 & 0.694 & 0.711 & 0.222 & 0.192 & 0.185 & 0.181 \\
XGBoost BCESurv & 0.684 & 0.691 & 0.694 & 0.690 & 0.740 & 0.732 & 0.736 & 0.737 & 0.162 & 0.161 & 0.158 & 0.159 \\
\midrule
\multicolumn{13}{l}{\textit{Classical \& Neural Survival Baselines}} \\
CoxPH & 0.693 & \cellcolor{blue!10}0.693 & 0.696 & \cellcolor{blue!10}0.696 & 0.741 & 0.739 & 0.742 & 0.736 & 0.162 & 0.159 & 0.156 & 0.156 \\
RSF & 0.691 & 0.691 & 0.693 & 0.693 & 0.733 & 0.731 & 0.734 & 0.730 & 0.165 & 0.163 & 0.161 & 0.160 \\
DeepHit & 0.557 & 0.642 & 0.645 & 0.642 & 0.574 & 0.672 & 0.674 & 0.665 & 0.210 & 0.197 & 0.197 & 0.195 \\
DeepSurv & 0.660 & 0.672 & 0.686 & 0.642 & 0.710 & 0.711 & 0.711 & 0.681 & 0.174 & 0.172 & 0.164 & 0.177 \\
\bottomrule
\end{tabular}
}
\end{table}

\begin{figure}[h!]
    \centering
    \includegraphics[width=\linewidth]{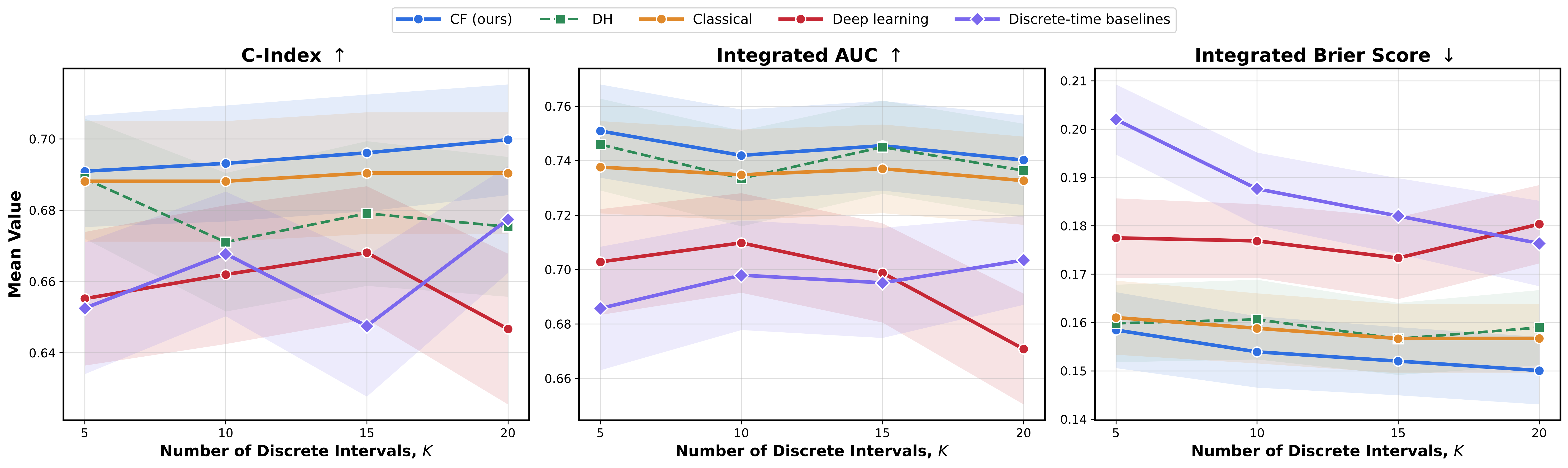} 
    \caption{\bo{Static model performance over discretization granularities $K \in \{5,10,15,20\}$ by model group.} We group methods into classical (CoxPH, RSF), deep learning (DeepSurv, DeepHit), discrete-time baselines (MTLR, Neural MTLR, MLP BCESurv), our cumulative failure formulation (MITRA-CF, TabPFN-CF, XGBoost-CF), and the discrete-time hazard variants (MITRA-DH, TabPFN-DH, XGBoost-DH). Within each group we pick the best-performing model by validation C-Index. 
}
    \label{fig:K_sensitivity_group}
\end{figure}

\bo{Effect of post-hoc calibration.} Figure~\ref{fig:dh_k_calibration} shows the IBS of the CF and DH models of the two TFMs at each $K$, before and after the post-hoc calibration of Section~\ref{sec:models}. Without calibration, the IBS of the DH models worsens as $K$ grows: from $0.155$ at $K=5$ to $0.213$ at $K=20$ for TabPFN-DH, and from $0.160$ to $0.168$ for MITRA-DH, while the CF models change less (TabPFN-CF: $0.161$ to $0.171$; MITRA-CF: $0.158$ to $0.156$). After calibration, the IBS of the DH models no longer depends on $K$ ($0.155$ to $0.157$ for TabPFN-DH). The sensitivity of the uncalibrated DH models to $K$ is thus removed by a monotone recalibration of the predicted hazards. The C-Index and Integrated AUC do not depend on the calibration.

\begin{figure}[h!]
    \centering
    \includegraphics[width=0.85\linewidth]{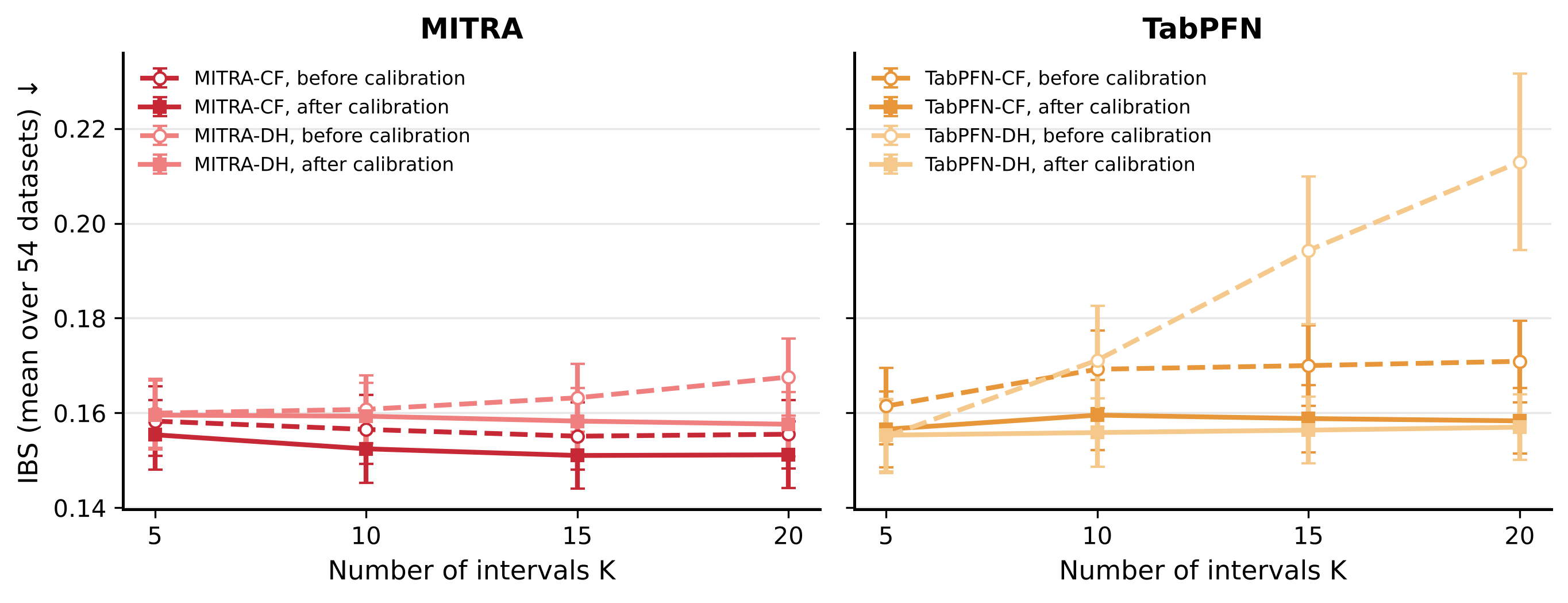} 
    \caption{\bo{IBS before and after post-hoc calibration for each $K$.} Mean IBS over the $54$ static datasets for the CF and DH models of MITRA (left) and TabPFN (right), before (dashed) and after (solid) post-hoc calibration. Error bars: $\pm$ one standard error over the datasets.}
    \label{fig:dh_k_calibration}
\end{figure}
\subsection{Examining the monotonicity of CF predictions and ways to ensure monotonicity}
\label{sec:isotonic}

The cumulative-failure (CF) classifier predicts $\hat p(X_{i,0}, t_k)$ separately for each grid point, so the curve $1 - \hat p(X_{i,0}, t_k)$ is not guaranteed to be monotonically decreasing in $t_k$. In our experiments, we clip the predicted survival probabilities to ensure monotonicity (Section~\ref{sec:staticMethod}). Here we report how often and by how much the curves violate monotonicity, and we compare clipping with isotonic regression to ensure monotonicity.

\bo{Two transformations to ensure monotonicity.} For a test observation $i$, let $\tilde S_{i,k} \triangleq 1 - \hat p(X_{i,0}, t_k)$, $k = 1, \dots, K-1$, be the curve before the transformation, where $\hat p$ is the predicted probability from a CF model.
\begin{itemize}[leftmargin=*]
    \item \bo{Clipping}: $\hat S^{\TN{clip}}_{i,k} \triangleq \min_{j \leq k} \tilde S_{i,j}$.
    \item \bo{Isotonic regression}: the non-increasing sequence closest to the curve in squared error,
    $\big( \hat S^{\TN{iso}}_{i,1}, \dots, \hat S^{\TN{iso}}_{i,K-1} \big) \triangleq \arg\min_{s_1 \geq \dots \geq s_{K-1}} \sum_{k=1}^{K-1} \big( s_k - \tilde S_{i,k} \big)^2$,
    computed with the pool-adjacent-violators algorithm (\texttt{IsotonicRegression} of \texttt{scikit-learn}). It is applied to each sample separately, over its $K-1$ grid points.
\end{itemize}

\bo{Reported quantities.} For each model, dataset, and $K$ we compute the following on the test set.
\begin{itemize}[leftmargin=*]
    \item \bo{Violations}: the share of test observations $i$ with $\tilde S_{i,k+1} > \tilde S_{i,k}$ for at least one $k \in \{1, \dots, K-2\}$.
    \item \bo{Size}: among the test observations with a violation, the mean of $\max_k \big(\tilde S_{i,k+1} - \tilde S_{i,k}\big)$, i.e., the mean of largest increase of the predicted survival probabilities between adjacent grid points when a violation occurs.
\end{itemize}

\begin{table}[h!]
\centering
\caption{\bo{Clipping vs.\ isotonic regression for the CF models of Table~\ref{tab:performance_comparison_km}.} ``Violations'' measures the $\%$ of test observations whose curve before the transformation violates monotonicity; 
``Size'' is the mean size of a monotonicity violation among the test observations that have one. We report the mean across all $54$ static datasets for each value of $K$; for ``Size'', across the datasets with at least one violation.} 
\label{tab:monotonicity}
\resizebox{\textwidth}{!}{
\begin{tabular}{ll cc cc cc cc}
\toprule
 & & \multicolumn{2}{c}{\textbf{Before Transformation}} & \multicolumn{2}{c}{\textbf{C-Index} $\uparrow$} & \multicolumn{2}{c}{\textbf{Integrated AUC} $\uparrow$} & \multicolumn{2}{c}{\textbf{IBS} $\downarrow$} \\
\cmidrule(r){3-4} \cmidrule(r){5-6} \cmidrule(r){7-8} \cmidrule(r){9-10}
Model & $K$ & Violations & Size & Clipping & Isotonic & Clipping & Isotonic & Clipping & Isotonic \\
\midrule
\textbf{MITRA-CF} & 5 & 1.5\% & 0.003 & 0.694 & 0.694 & 0.758 & 0.758 & 0.155 & 0.155 \\  
 & 10 & 3.7\% & 0.007 & 0.699 & 0.699 & 0.751 & 0.751 & 0.152 & 0.152 \\  
 & 15 & 5.7\% & 0.008 & 0.700 & 0.700 & 0.749 & 0.749 & 0.151 & 0.151 \\  
 & 20 & 9.1\% & 0.005 & 0.700 & 0.700 & 0.745 & 0.745 & 0.151 & 0.151 \\  
 & All & 5.0\% & 0.006 & 0.699 & 0.699 & 0.752 & 0.752 & 0.152 & 0.152 \\  
\midrule
\textbf{TabPFN-CF} & 5 & 2.8\% & 0.010 & 0.672 & 0.672 & 0.739 & 0.739 & 0.157 & 0.157 \\  
 & 10 & 5.4\% & 0.010 & 0.660 & 0.660 & 0.721 & 0.721 & 0.160 & 0.160 \\  
 & 15 & 8.7\% & 0.010 & 0.661 & 0.661 & 0.722 & 0.722 & 0.159 & 0.159 \\  
 & 20 & 9.5\% & 0.009 & 0.663 & 0.663 & 0.717 & 0.717 & 0.158 & 0.158 \\  
 & All & 6.7\% & 0.009 & 0.665 & 0.665 & 0.727 & 0.727 & 0.158 & 0.158 \\  
\midrule
\textbf{XGBoost-CF} & 5 & 4.7\% & 0.020 & 0.682 & 0.683 & 0.738 & 0.739 & 0.159 & 0.159 \\  
 & 10 & 20.8\% & 0.013 & 0.687 & 0.686 & 0.732 & 0.733 & 0.158 & 0.157 \\  
 & 15 & 27.4\% & 0.011 & 0.693 & 0.693 & 0.736 & 0.737 & 0.155 & 0.155 \\  
 & 20 & 38.6\% & 0.010 & 0.686 & 0.686 & 0.733 & 0.733 & 0.154 & 0.154 \\  
 & All & 22.9\% & 0.013 & 0.688 & 0.688 & 0.737 & 0.737 & 0.156 & 0.156 \\  
\bottomrule
\end{tabular}}
\end{table}


\bo{Results (Table~\ref{tab:monotonicity}).} Violations become more frequent as $K$ grows, since adjacent grid points get closer, but they remain small: when a violation occurs, the predicted survival probability increases by about $0.01$ on average ($0.006$ for MITRA-CF, $0.009$ for TabPFN-CF, and $0.013$ for XGBoost-CF), and by at most $0.02$ for every model and $K$. 
Violations are rarer for the two TFMs than for XGBoost-CF (at $K = 20$, $9.1\%$ of the test observations for MITRA-CF vs.\ $38.6\%$ for XGBoost-CF). The two transformations give almost the same performance on C-Index, Integrated AUC, and IBS for all values of $K$. 
\subsection{Context sampling approaches with MITRA}
\label{app:mitra_context}

TFM classifiers have limited context windows, e.g., MITRA has a maximum context window of $8,192$ examples $\big( (X_{i,0}, t_k), Y_{i,k} \big)$. In the main paper, if a time-to-event dataset, after transforming into a supervised learning problem, has more than $8,192$ examples, we uniformly sampled $8,192$ from all available examples. In the static setting, we examine $54$ datasets and $4$ values of $K$ (total of $216$ combinations; affecting $24$ datasets)---across these, $68$ configurations lead to more than $8,192$ examples.


In this Appendix, we examine the sensitivity to alternative sampling approaches:
\begin{itemize}[leftmargin=*]
    \item \bo{Uniform sampling:} Samples the 8,192 $\big( (X_{i,0}, t_k), Y_{i,k} \big)$ examples uniformly at random from all examples in the training set.
    \item \bo{Sample-stratified sampling:} We sample 8,192 $\big( (X_{i,0}, t_k), Y_{i,k} \big)$ examples at random, stratified  by $i \in [1 \colon n]$. This ensures that each original survival sample $i$ is equally represented.
    \item \bo{Horizon-stratified sampling:} We sample 8,192 $\big( (X_{i,0}, t_k), Y_{i,k} \big)$ examples at random, stratified  by $k \in \{ 1, \dots, K-1\}$. This ensures that each time point $t_k$ is equally represented.
\end{itemize}

\begin{table}[h!]
\centering
\caption{\textbf{Comparing different context sampling approaches for MITRA-CF.} Mean over the $68$ dataset--$K$ combinations ($24$ datasets) whose training set exceeds MITRA's $8{,}192$-samples.} 
\label{tab:mitra_balanced}
\begin{tabular}{l ccc ccc}
\toprule
 & \multicolumn{3}{c}{\textbf{Mean}} \\
\cmidrule(r){2-4} \cmidrule(r){5-7}
\textbf{Sampling Strategy} & C-Index $\uparrow$ & AUC $\uparrow$ & IBS $\downarrow$ \\
\midrule
Uniform & 0.688 & 0.743 & 0.134 \\
Sample-stratified & 0.689 & 0.740 & 0.134 \\
Horizon-stratified & 0.687 & 0.742 & 0.133 \\
\bottomrule
\end{tabular}
\end{table}
\subsection{Direct comparison of TFMs with XGBoost}
\label{app:tfm_vs_xgboost}

The gains of Section~\ref{sec:static-results} are computed relative to the classical and neural survival baselines, which do not include XGBoost with our classification formulations. Table~\ref{tab:tfm_vs_xgboost} compares each TFM classification approach with XGBoost. 
For each of the $54$ static datasets, we average each metric over $K \in \{5, 10, 15, 20\}$ and report (i) the percentage of datasets on which the TFM model has the better test value, and (ii) the relative gain of the TFM model over the XGBoost model.

\begin{table}[h!]
\centering
\caption{\textbf{TFMs vs.\ XGBoost with the same formulation} (static setting, $54$ datasets; models as in Table~\ref{tab:performance_comparison_km}). For each dataset, the metrics are averaged over $K \in \{5, 10, 15, 20\}$. ``Datasets where the TFM is better'': percentage of datasets on which the TFM model has the better test value. ``Relative gain'': $(m^{\mathrm{TFM}}_d - m^{\mathrm{XGB}}_d) / m^{\mathrm{XGB}}_d$ in percent, with the sign chosen so that positive values favor the TFM, averaged over the datasets. $\pm$ denotes one standard error. } 
\label{tab:tfm_vs_xgboost}
\resizebox{\linewidth}{!}{%
\begin{tabular}{l ccc ccc}
\toprule
 & \multicolumn{3}{c}{\textbf{Datasets where the TFM is better (\%)}} & \multicolumn{3}{c}{\textbf{Relative gain (\%)}} \\
\cmidrule(r){2-4} \cmidrule(r){5-7}
\textbf{Comparison} & C-Index & Integrated AUC & IBS & C-Index & Integrated AUC & IBS \\
\midrule
MITRA-CF vs.\ XGBoost-CF & $69 \pm 6$ & $62 \pm 7$ & $54 \pm 7$ & $+1.7 \pm 0.7$ & $+2.1 \pm 0.8$ & $+1.9 \pm 0.9$ \\
TabPFN-CF vs.\ XGBoost-CF & $41 \pm 7$ & $30 \pm 6$ & $37 \pm 7$ & $-2.3 \pm 2.1$ & $-1.4 \pm 1.1$ & $-2.1 \pm 1.5$ \\
\midrule
MITRA-DH vs.\ XGBoost-DH & $52 \pm 7$ & $58 \pm 7$ & $50 \pm 7$ & $-0.5 \pm 1.6$ & $+1.6 \pm 1.1$ & $-0.3 \pm 1.8$ \\
TabPFN-DH vs.\ XGBoost-DH & $46 \pm 7$ & $64 \pm 7$ & $61 \pm 7$ & $-1.4 \pm 1.7$ & $+1.7 \pm 1.4$ & $+1.7 \pm 1.6$ \\
\bottomrule
\end{tabular}}
\end{table}

MITRA-CF is better than XGBoost-CF on most datasets for the C-Index and the Integrated AUC, while the two are close in IBS. TabPFN-CF is better than XGBoost-CF on fewer than half of the datasets for all three metrics. With the DH formulation, both TFMs are better than XGBoost-DH on more than half of the datasets for the Integrated AUC, and are close to it on the C-Index. The benefit of a TFM over XGBoost thus depends on the TFM and on the formulation.

\end{document}